\documentclass[letterpaper,11pt]{article}
\usepackage{tabularx}
\usepackage{amsthm}
\usepackage{mathtools}
\usepackage{amsmath}
\usepackage{commath}
\usepackage{amssymb}
\usepackage{bm}
\usepackage{enumerate}
\usepackage[margin=1in,letterpaper]{geometry} % decreases margins
\usepackage[final]{hyperref}
\hypersetup{
	colorlinks=true,       % false: boxed links; true: colored links
	linkcolor=blue,        % color of internal links
	citecolor=blue,        % color of links to bibliography
	filecolor=magenta,     % color of file links
	urlcolor=blue         
}

\usepackage[toc,page]{appendix}
\usepackage{url}
\usepackage{booktabs}
\usepackage{amsfonts}
\usepackage{nicefrac}
\usepackage{microtype}
\usepackage{lipsum}
\usepackage{fancyhdr}
\usepackage{graphicx}
\usepackage{algorithm}
\usepackage{algorithmic}
\usepackage{microtype}
\usepackage{multicol}
\usepackage{multirow}
\usepackage{booktabs}
\usepackage{subcaption}
\usepackage{float}
\usepackage{xcolor}
\definecolor{navyblue}{rgb}{0.0, 0.0, 0.5}

\allowdisplaybreaks

\newcommand*{\thead}[1]{%
	\multicolumn{1}{c}{\bfseries\begin{tabular}{@{}c@{}}#1\end{tabular}}}

\newtheorem{theorem}{Theorem}[section]
\newtheorem{corollary}{Corollary}[theorem]
\newtheorem{definition}{Definition}[section]
\newtheorem{lemma}[theorem]{Lemma}
\newtheorem{remark}{Remark}
\newtheorem{assumption}{Assumption}

\newcommand{\rr}{{\mathbb R}}

\newcommand{\Secref}[1]{Section~\ref{#1}}

\begin{document}

\title{SCORE: Approximating Curvature Information under Self-Concordant Regularization}
\author{Adeyemi D. Adeoye\thanks{IMT School for Advanced Studies Lucca, Italy. E-mail: adeyemi.adeoye@imtlucca.it} \and Alberto Bemporad\thanks{IMT School for Advanced Studies Lucca, Italy. E-mail: alberto.bemporad@imtlucca.it}}
\date{\today}
\maketitle

\begin{abstract}
Optimization problems that include regularization functions in their objectives are regularly solved in many applications. When one seeks second-order methods for such problems, it may be desirable to exploit specific properties of some of these regularization functions when accounting for curvature information in the solution steps to speed up convergence. In this paper, we propose the SCORE (self-concordant regularization) framework for unconstrained minimization problems which incorporates second-order information in the \emph{Newton-decrement} framework for convex optimization. We propose the generalized Gauss-Newton with Self-Concordant Regularization (GGN-SCORE) algorithm that updates the minimization variables each time it receives a new input batch. The proposed algorithm exploits the structure of the second-order information in the Hessian matrix, thereby reducing computational overhead. GGN-SCORE demonstrates how to speed up convergence while also improving model generalization for problems that involve regularized minimization under the proposed SCORE framework. Numerical experiments show the efficiency of our method and its fast convergence, which compare favorably against baseline first-order and quasi-Newton methods. Additional experiments involving non-convex (overparameterized) neural network training problems show that the proposed method is promising for non-convex optimization.
\end{abstract}

\section{Introduction}\label{sec1}

The results presented in this paper apply to a pseudo-online optimization algorithm based on solving a regularized unconstrained minimization problem under the assumption of strong convexity and Lipschitz continuity. Unlike first-order methods such as stochastic gradient descent (SGD) \cite{robbins1951stochastic, bottou2010large} and its variants \cite{duchi2011adaptive, zeiler2012adadelta, kingma2014adam, johnson2013accelerating} that only make use of first-order information through the function gradients, second-order methods \cite{becker1988improving, liu1989limited, hagan1994training, amari1998natural, martens2010deep, pascanu2013revisiting, martens2015optimizing} attempt to incorporate, in some way, second-order information in their approach, through the Hessian matrix or the Fisher information matrix (FIM). It is well known that this generally provides second-order methods with better (quadratic) convergence than a typical first-order method which only converges linearly in the neighbourhood of the solution \cite{nesterov2018lectures}.

Despite their convergence advantage over first-order methods, second-order methods result into highly prohibitive computations, namely, inverting an $n_w\times n_w$ matrix at each iteration, where $n_w$ is the number of optimization variables. In most commonly used second-order methods, the natural gradient, Gauss-Newton, and the sub-sampled Newton \cite{byrd2011use} -- or its regularized version \cite{erdogdu2015convergence} -- used for incorporating second-order information while maintaining desirable convergence properties, compute the Hessian matrix (or FIM) $\bm{H}$ by using the approximation $\bm{H}\approx \bm{J}^T\bm{J}$, where $\bm{J}$ is the Jacobian matrix. This approximation is still an $n_w\times n_w$ matrix and remains computationally demanding for large problems. In recent works \cite{cai2019gram, zhang2019fast, bernacchia2019exact, karakida2020understanding}, second-order methods for overparameterized neural network models are made to bypass this difficulty by applying a matrix identity, and instead only compute the matrix $\bm{J}\bm{J}^T$ which is a $d\cdot N\times d\cdot N$ matrix, where $d$ is the model output dimension and $N$ is the number of data points. This approach significantly reduces computational overhead in the case $d\cdot N$ is much smaller than $n_w$ (overparameterized models) and helps to accelerate convergence \cite{cai2019gram}. Nevertheless, for an objective function with a differentiable (up to two times) convex regularizer, this simplification requires a closer attention and special modifications for a general problem with a large number of variables.

The idea of exploiting desirable regularization properties for improving the convergence of the (Gauss-)Newton scheme has been around for decades and most of the published works on the topic combine in different ways the idea of Levenberg-Marquardt (LM) regularization, line-search, and trust-region methods \cite{nesterov2006cubic}. For example, the recent work \cite{mishchenko2021regularized} combines the idea of cubic regularization (originally proposed by \cite{nesterov2006cubic}) and a particular variant of the adaptive LM penalty that uses the \emph{Euclidean} gradient norm of the output-fit loss function (see \cite{marumo2020constrained} for a comprehensive list and complexity bounds). Their proposed scheme achieves a global $\mathcal{O}(k^{-2})$ rate, where $k$ is the iteration number. A similar idea is considered in \cite{doikov2021gradient} using the Bregman distances, extending the idea to develop an accelerated variant of the scheme that achieves a $\mathcal{O}(k^{-3})$ convergence rate.

In this paper, we propose a new self-concordant regularization (SCORE) scheme for efficiently choosing optimal variables of the model involving smooth, strongly convex optimization objectives, where one of the objective functions regularizes the model's variable vector and hence avoids overfitting, ultimately improving the model's ability to generalize well. By an \emph{extra} assumption that the regularization function is self-concordant, we propose the GGN-SCORE algorithm (see Algorithm \ref{alg:GGN-SCORE} below) that updates the minimization variables in the framework of the \emph{local norm} $\|\cdot\|_x$ (a.k.a., Newton-decrement) of a self-concordant function $f(x)$ such as seen in \cite{nesterov1994interior}. Our proposed scheme does not require that the output-fit loss function is self-concordant, which in many applications does not hold \cite{mishchenko2021regularized}. Instead, we exploit the \emph{greedy} descent provision of self-concordant functions, via regularization, to achieve a fast convergence rate while maintaining feasible assumptions on the combined objective function (from an application point of view). Although this paper assumes a convex optimization problem, we also provide experiments that show promising results for non-convex problems that arise when training neural networks. Our experimental results provide an interesting opportunity for future investigation and scaling of the proposed method for large-scale machine learning problems, as one of the non-convex problems considered in the experiments involves training an overparameterized neural network. We remark that overparameterization is an interesting and desirable property, and a topic of interest in the machine learning community \cite{bubeck2021universal, muthukumar2020harmless, belkin2019reconciling, allen2019learning}.

This paper is organized as follows: First, we introduce some notations, formulate the optimization problem with basic assumptions, and present an initial motivation for the optimization method in \Secref{ss:background}. In \Secref{ss:overparam}, we derive a new generalized method for reducing the computational overhead associated with the mini-batch Newton-type updates. The idea of SCORE is introduced in \Secref{ss:SCORE} and our GGN-SCORE algorithm is presented thereafter. Experimental results that show the efficiency and fast convergence of the proposed method are presented in \Secref{ss:experiments}.
\section{Preliminaries}\label{ss:background}
\subsection{Notation and Basic Assumptions}
Let $\{(\bm{x}_n,\bm{y}_n)\}_{n=1}^N$ be a sequence of $N$ input and output sample pairs, $\bm{x}_n\in\mathbb{R}^{n_p}, \bm{y}_n\in\mathbb{R}^{d}$, where $n_p$ is the number of features and $d$ is the number of targets. We assume a model $f(\bm{\theta};\bm{x}_n)$, defined by $f\colon \mathbb{R}^{n_w} \times \mathbb{R}^{n_p} \to \mathbb{Y}$ and parameterized by the vector of variables $\bm{\theta}\in\mathbb{R}^{n_w}$. We denote by $\partial_{\bm{a}} \bm{b} \equiv \partial\bm{b}$ the gradient (or first derivative) of $\bm{b}$ (with respect to $\bm{a}$) and $\partial_{\bm{a}\bm{a}}^2 \bm{b} \equiv \partial^2\bm{b}$ the second derivative of $\bm{b}$ with respect to $\bm{a}$. We write $\|\cdot\|$ to denote the $2$-norm. The set $\{\mathrm{diag}(\bm{v})\colon \bm{v}\in\rr^n\}$, where $\mathrm{diag}\colon \rr^n \to \rr^{n\times n}$, denotes the set of all diagonal matrices in $\rr^{n\times n}$. Throughout the paper, bold-face letters denote vectors and matrices.

Suppose that $f(\bm{\theta};\bm{x}_n)$ outputs the value $\hat{\bm{y}}_n \in \mathbb{R}^{d}$. The regularized minimization problem we want to solve is
\begin{align}
	\min_{\bm{\theta}}\mathcal{L}(\bm{\theta}) &\coloneqq \underbrace{\sum_{n=1}^N\ell(\bm{y}_n,\hat{\bm{y}}_n)}_{g(\bm{\theta})} + \lambda \underbrace{\sum_{j=1}^{n_w} r_j(\bm{\theta}_j)}_{h(\bm{\theta})}, \label{eq:emp}
\end{align}
where $\ell:\rr^d\times\rr^d\to\rr$ is a (strongly) convex twice-differentiable output-fit loss function, $r_j:\rr\to\rr$, $j=1,\ldots,n_w$, define a separable regularization term on $\bm{\theta}$, $g(\bm{\theta})\colon\rr^{n_w}\to\rr$, $h(\bm{\theta})\colon\rr^{n_w}\to\rr$. We assume that the regularization function $h(\bm{\theta})$, scaled by the parameter $\lambda>0$, is twice differentiable and strongly convex. The following preliminary conditions define the regularity of the Hessian of $\mathcal{L}(\bm{\theta})$ and are assumed to hold only locally in this work:
\begin{assumption}\label{thm:ass1}
	The functions $g$ and $h$ are twice-differentiable with respect to $\bm{\theta}$ and are respectively $\gamma_l$- and $\gamma_a$-strongly convex.
\end{assumption}
\begin{assumption}\label{thm:ass1-1} $\exists\gamma_u$, $\gamma_b$ with $0\le\gamma_l\leq \gamma_u < \infty$, $0<\gamma_a\leq \gamma_b < \infty$, such that the gradient of $\mathcal{L}(\bm{\theta})$ is $(\gamma_u+\lambda\gamma_b)$-Lipschitz continuous $\forall \bm{x} \in \rr^{n_p}, \bm{y} \in \rr^d$. That is, $\forall \bm{x} \in \rr^{n_p}, \forall \bm{y} \in \rr^d$, the gradient $\partial_{\bm{\theta}}\mathcal{L}(\bm{\theta})=\partial_{\bm{\theta}}g(\bm{\theta}) + \lambda\partial_{\bm{\theta}}h(\bm{\theta})$ satisfies
	\begin{align}
		\left\|\partial_{\bm{\theta}} \mathcal{L}(\bm{y}, f(\bm{\theta_1};\bm{x})) - \partial_{\bm{\theta}} \mathcal{L}(\bm{y}, f(\bm{\theta_2};\bm{x}))\right\| \le (\gamma_u+\lambda\gamma_b) \left\|\bm{\theta_1} - \bm{\theta_2}\right\|, \label{eq:lipell1}
	\end{align}
	for any $\bm{\theta_1},\bm{\theta_2} \in \rr^{n_w}$.
\end{assumption}
\begin{assumption}\label{thm:ass2}
	$\exists \gamma_g, \gamma_h$ with $0<\gamma_g,\gamma_h<\infty$ such that $\forall \bm{x} \in \rr^{n_p}, \forall \bm{y} \in \rr^d$, the second derivatives of $g(\bm{\theta})$ and $h(\bm{\theta})$ respectively satisfy
	\begin{subequations}
		\begin{eqnarray}
			\left\|\partial^2 g(\bm{y}, f(\bm{\theta_1};\bm{x})) - \partial^2 g(\bm{y}, f(\bm{\theta_2};\bm{x}))\right\| \le \gamma_g\left\|\bm{\theta_1} - \bm{\theta_2}\right\|,\\
			\left\|\partial^2 h(\bm{y}, f(\bm{\theta_1};\bm{x})) - \partial^2 h(\bm{y}, f(\bm{\theta_2};\bm{x}))\right\| \le \gamma_h\left\|\bm{\theta_1} - \bm{\theta_2}\right\|,
		\end{eqnarray}
	\end{subequations}
	for any $\bm{\theta_1},\bm{\theta_2} \in \rr^{n_w}$.
\end{assumption}
Commonly used loss functions such as the squared loss, and the sigmoidal cross-entropy loss are twice differentiable and (strongly) convex in the model variables. Certain smoothed sparsity-inducing penalties such as the (pseudo-)Huber function -- presented later in this paper -- constitute the class of functions that may be well-suited for $h(\bm{\theta})$ defined above.

The assumptions of strong convexity and smoothness about the objective $\mathcal{L}(\bm{\theta})$ are standard conventions in many optimization problems as they help to characterize the convergence properties of the underlying solution method \cite{meng2020fast, ye2020nesterov}. However, the smoothness assumption about the objective $\mathcal{L}(\bm{\theta})$ is sometimes not feasible for some multi-objective (or regularized) problems where a non-smooth (penalty-inducing) function $h(\bm{\theta})$ is used (see the recent work \cite{bieker2020treatment}). In such a case, and when the need to incorporate second-order information arise, a well-known approach in the optimization literature is generally either to approximate the non-smooth objectives by a smooth quadratic function (when such an approximation is available) or use a ``proximal splitting" method and replace the $2$-norm in this setting with the $Q$-norm, where $Q$ is the Hessian matrix or its approximation \cite{patrinos2014forward}. In \cite{patrinos2014forward}, the authors propose two techniques that help to avoid the complexity that is often introduced in subproblems when the latter approach is used. While proposing new approaches, \cite{schmidt2007fast} highlights some popular techniques to handle non-differentiability. Each of these works highlight the importance of incorporating second-order information in the solution techniques of optimization problems. By conveniently solving the optimization problem \eqref{eq:emp} where the assumptions made above are satisfied, our method ensures the full curvature information is captured while reducing computational overhead.

\subsection{Approximate Newton Scheme}
Given the current value of $\bm{\theta}$, the (Gauss-)Newton method computes an update to $\bm{\theta}$ via
\begin{align}
	\bm{\theta} \leftarrow \bm{\theta} - \rho\bm{G},
\end{align}
where $\rho$ is the step size, $\bm{G}=\bm{H}^{-1}\partial\mathcal{L}(\bm{\theta})$ and $\bm{H}$ is the Hessian of $\mathcal{L}$ or its approximation. In this work, we consider the generalized Gauss-Newton (GGN) approximation of $\bm{H}$ which we now define in terms of the function $g(\bm{\theta})$. This approximation and its detailed expression motivates the modified version introduced in the next section to include the regularization function $h(\bm{\theta})$.
\begin{definition}[Generalized Gauss-Newton Hessian]
	Let $\mathrm{diag}(\bm{q}_n)\in \mathbb{R}^{d\times d}$ be the second derivative of the loss function $\ell(\bm{y}_n,\hat{\bm{y}}_n)$ with respect to the predictor $\hat{\bm{y}}_n$, $\bm{q}_n = \partial_{\hat{\bm{y}}_n\hat{\bm{y}}_n}^2 \ell(\bm{y}_n,\hat{\bm{y}}_n)$ for $n=1,2,\ldots,N$, and let $\bm{Q}_g\in \mathbb{R}^{dN\times dN}$ be a block diagonal matrix with $\bm{q}_n$ being the $n$-th diagonal block. Let $\bm{J}_n\in \mathbb{R}^{d\times n_w}$ denote the Jacobian of $\hat{\bm{y}}_n$ with respect to $\bm{\theta}$ for $n=1,2,\ldots,N$, and let $\bm{J}_g\in \mathbb{R}^{dN\times n_w}$ be the vertical concatenation of all $\bm{J}_n$'s. Then, the generalized Gauss-Newton (GGN) approximation of the Hessian matrix $\bm{H}_g\in \mathbb{R}^{n_w\times n_w}$ associated with the fit loss $\ell(\bm{y}_n,\hat{\bm{y}}_n)$ with respect to $\bm{\theta}$ is defined by
	\begin{align}
		\bm{H}_g\approx \bm{J}_g^T\bm{Q}_g
		\bm{J}_g = \sum_{n=1}^{N}\bm{J}_n^T\mathrm{diag}(\bm{q}_n)\bm{J}_n. \label{eq:ggn}
	\end{align}
\end{definition}
Let $\bm{e}_n\in \mathbb{R}^{d}$ be the Jacobian of the fit loss defined by $\bm{e}_n=\partial_{\hat{\bm{y}}_n}\ell(\bm{y}_n,\hat{\bm{y}}_n)$ for $n=1,2,\ldots,N$. For example, in case of squared loss $\ell(\bm{y}_n,\hat{\bm{y}}_n)=\frac{1}{2}(\bm{y}_n-\hat{\bm{y}}_n)^2$ we get that $\bm{e}_n$ is the residual $\bm{e}_n=\hat{\bm{y}}_n-\bm{y}_n$. Let $\bm{e}_g\in \mathbb{R}^{dN}$ be the vertical concatenation of all $\bm{e}_n$'s. Then, using the chain rule, we write
\begin{align}
	\bm{J}_n^T\bm{e}_n^T &= \begin{bmatrix}
		\partial_{\theta_1}\hat{y}^{(1)} & \partial_{\theta_1}\hat{y}^{(2))} & \cdots & \partial_{\theta_1}\hat{y}^{(d)}\\
		\partial_{\theta_2}\hat{y}^{(1)} & \partial_{\theta_2}\hat{y}^{(2)} & \cdots & \partial_{\theta_2}\hat{y}^{(d)}\\
		\vdots&\vdots&&\vdots\\
		\partial_{\theta_{n_w}}\hat{y}^{(1)} & \partial_{\theta_{n_w}}\hat{y}^{(2)} & \cdots & \partial_{\theta_{n_w}}\hat{y}^{(d)}
	\end{bmatrix}
	\begin{bmatrix}
		\partial_{\hat{y}}\ell^{(1)}\\
		\partial_{\hat{y}}\ell^{(2)}\\
		\vdots\\
		\partial_{\hat{y}}\ell^{(d)}
	\end{bmatrix} \notag \\
	&= \begin{bmatrix}
		\partial_{\theta_1}\ell^{(1)}+\partial_{\theta_1}\ell^{(2)}+\cdots+\partial_{\theta_1}\ell^{(d)}\\
		\partial_{\theta_2}\ell^{(1)}+\partial_{\theta_2}\ell^{(2)}+\cdots+\partial_{\theta_2}\ell^{(d)}\\
		\vdots\\
		\partial_{\theta_{n_w}}\ell^{(1)}+\partial_{\theta_{n_w}}\ell^{(2)}+\cdots+\partial_{\theta_{n_w}}\ell^{(d)}
	\end{bmatrix} \notag \\
	&= \left[
	\sum_{i=1}^{d}\partial_{\theta_1}\ell^{(i)} \quad
	\sum_{i=1}^{d}\partial_{\theta_2}\ell^{(i)} \quad
	\cdots \quad
	\sum_{i=1}^{d}\partial_{\theta_{n_w}}\ell^{(i)}
	\right]^T, \label{eq:gje}
\end{align}
and
\begin{align}
	\bm{g}_g(\bm{\theta})=\bm{J}_g^T\bm{e}_g=\sum_{n=1}^{N} \bm{J}_n^T\bm{e}_n. \label{eq:gjeVect}
\end{align}
As noted in \cite{schraudolph2002fast}, the GGN approximation has the advantage of capturing the curvature information of $\ell$ in $g(\bm{\theta})$ through the term $\bm{Q}_g$ as opposed to the FIM, for example, which ignores the full second-order interactions. While it may become obvious, say when training a deep neural network with certain loss functions, how the GGN approximation can be exploited to simplify expressions for $\bm{H}_g$ (see e.g. in \cite{bottou2018optimization}), a modification is required to take account of a twice-differentiable convex regularization function to achieve some degree of simplicity and elegance. We derive a modification to the above for the mini-batch scheme presented in the next section that includes the derivatives of $h(\bm{\theta})$ in the GGN approximation of $\bm{H}_g$. This modification leads to our GGN-SCORE algorithm in \Secref{ss:SCORE}.

\section{Second-order (Pseudo-online) Optimization}\label{ss:overparam}
Suppose that at each mini-batch step $k$ we uniformly select a random index set $\mathcal{I}_k\subseteq \{1,2,\ldots,N\},\vert\mathcal{I}_k\vert=m \le N$ (usually $m\ll N$) to access a mini-batch of $m$ samples from the training set. The loss derivatives used for the recursive update of the variables $\bm{\theta}$ in this way is computed at each step $k$, and are estimated as running averages over the batch-wise computations. This leads to a stochastic approximation of the true derivatives at each iteration for which we assume unbiased estimations.

The problem of finding the optimal adjustment $\delta\bm{\theta}_{mk} \coloneqq \bm{\theta}_{k+1}-\bm{\theta}_k$ that solves \eqref{eq:emp} results in solving either an \textit{overdetermined} or an \textit{underdetermined} linear system depending on whether $dm\ge n_w$ or $dm < n_w$, respectively. Consider, for example, the squared fit loss and the penalty-inducing square norm as the scalar-valued functions $g(\bm{\theta})$ and $h(\bm{\theta})$, respectively in \eqref{eq:emp}. Then, $\bm{Q}_g$ will be the identity matrix, and the LM solution $\delta\bm{\theta}$ \cite{levenberg1944method, marquardt1963algorithm} is estimated at each iteration $k$ according to the rule\footnote{For simplicity of notation, and unless where the full notations are explicitly required, we shall subsequently drop the subscripts $m$ and $k$, and assume that each expression represents stochastic approximations performed at step $k$ using randomly selected data batches each of size $m$.}:
\begin{align}
	\delta\bm{\theta} = - (\bm{H}_g+\lambda \bm{I})^{-1}\bm{g}_g = -(\bm{J}_g^T\bm{J}_g+\lambda \bm{I})^{-1}\bm{J}_g^T\bm{e}_g. \label{eq:lmupdate}
\end{align}
If $dm < n_w$ (possibly $dm\ll n_w$), then by using the Searle identity $(\bm{A}\bm{B} + \lambda \bm{I})\bm{A} = \bm{A}(\bm{B}\bm{A} + \lambda \bm{I})$ \cite{searle1982matrix}, we can conveniently update the adjustment $\delta\bm{\theta}$ by
\begin{align}
	\delta\bm{\theta} = -\bm{J}_g^T(\bm{J}_g\bm{J}_g^T+\lambda \bm{I})^{-1}\bm{e}_g.
\end{align}
Clearly, this provides a more computationally efficient way of solving for $\delta\bm{\theta}$. In what follows, we formulate a generalized solution method for the regularized problem \eqref{eq:emp} which similarly exploits the Hessian matrix structure when solving the given optimization problem, thereby conveniently and efficiently computing the adjustment $\delta\bm{\theta}$.

Taking the second-order approximation of $\mathcal{L}(\bm{\theta})$, we have
\begin{align}
	\mathcal{L}(\bm{\theta} + \delta\bm{\theta}) &\approx \mathcal{L}(\bm{\theta}) + \bm{g}^T\delta\bm{\theta} + \frac{1}{2}\delta\bm{\theta}^T \bm{H}\delta\bm{\theta}, \label{eq:approx}
\end{align}
where $\bm{H}\in\rr^{n_w\times n_w}$ is the Hessian of $\mathcal{L}(\bm{\theta})$ and $\bm{g}\in\rr^{n_w}$ is its gradient. Let $M=dm+1$ and define the Jacobian $\bm{J}\in\rr^{M\times n_w}$:
\begin{align}
	\bm{J}^T &= \begin{bmatrix}
		\partial_{\theta_1}\hat{\bm{y}}_1 & \partial_{\theta_1}\hat{\bm{y}}_2 & \cdots & \partial_{\theta_1}\hat{\bm{y}}_m & \lambda \partial_{\theta_1}r_1 \\
		\partial_{\theta_2}\hat{\bm{y}}_1 & \partial_{\theta_2}\hat{\bm{y}}_2 & \cdots & \partial_{\theta_2}\hat{\bm{y}}_m & \lambda \partial_{\theta_2}r_2\\
		\vdots&\vdots&&\vdots&\vdots\\
		\partial_{\theta_{n_w}}\hat{\bm{y}}_1 & \partial_{\theta_{n_w}}\hat{\bm{y}}_2 & \cdots & \partial_{\theta_{n_w}}\hat{\bm{y}}_m & \lambda \partial_{\theta_{n_w}}r_{n_w}
	\end{bmatrix}. \label{eq:Jaug}
\end{align}
Let $\bm{e}_M=1$ and denote by $\bm{e}\in \rr^M$ the vertical concatenation of all $\bm{e}_n$'s, $\bm{e}_n\in\rr^d$, $n=1,2,\ldots,m$ and $\bm{e}_M$. Then by using the chain rule as in \eqref{eq:gje} and \eqref{eq:gjeVect}, we obtain
\begin{align}
	\bm{g}(\bm{\theta}) = \partial_{\bm{\theta}}\mathcal{L}(\bm{\theta}) = \bm{J}^T\bm{e}. \label{eq:gaug}
\end{align}
Let $\bm{q}_n=\partial_{\hat{\bm{y}}_n\hat{\bm{y}}_n}^2\ell(\bm{y}_n,\hat{\bm{y}}_n)$, $\bm{q}_n\in\rr^d$
for $n=1,2,\ldots,m$ (clearly $\bm{q}_n=\bm{1}$ in case of squared fit loss terms) and let $\bm{q}_M=0$. Define $\bm{Q}\in \rr^{M\times M}$ as the diagonal matrix with diagonal elements $\bm{q}_n$, $n=1,2,\ldots,m$ and $\bm{q}_M$, where $\bm{Q}=\bm{Q}^T\succeq \bm{0}$ by convexity of $\ell$. Consider the following slightly modified GGN approximation of the Hessian $\bm{H}\in \rr^{n_w\times n_w}$ associated with $\mathcal{L}(\bm{\theta})$:
\begin{align}
	\bm{H}\approx \bm{J}^T\bm{Q}\bm{J}+\lambda \bm{H}_h, \label{eq:mggn}
\end{align}
where $\bm{H}_h$ is the Hessian of the regularization term $h(\bm{\theta})$, $\bm{H}_h\in\rr^{n_w \times n_w}$, and is a diagonal matrix whose diagonal terms are
\begin{align*}
	{\bm{H}_h}_{jj}=\frac{d^2r_j(\theta_j)}{d\theta_j^2},\ j=1,\ldots,n_w.
\end{align*}
We hold on to the notation $\bm{H}$ to represent the modified GGN approximation of the full Hessian matrix $\partial^2\mathcal{L}$. By differentiating \eqref{eq:approx} with respect to $\delta\bm{\theta}$ and equating to zero, we obtain the optimal adjustment
\begin{align}
	\delta\bm{\theta} = -(\bm{J}^T\bm{Q}\bm{J} + \lambda \bm{H}_h )^{-1}\bm{J}^T\bm{e}.
	\label{eq:delta_theta}
\end{align}
\begin{remark}\label{thm:rem1}
	The inverse matrix in \eqref{eq:delta_theta} exists due to the strong convexity assumption on the regularization function $h$ which makes $\bm{H_r}\succeq (\min_{j}{\bm{H}_h}_{jj})\bm{I}$
	and therefore matrix $\bm{J}^T\bm{Q}\bm{J} + \lambda \bm{H}_h$ is symmetric positive definite and hence invertible.
\end{remark}
Let $\bm{U}=\bm{J}^T\bm{Q}$. Using the identity \cite{duncan1944lxxviii,guttman1946enlargement}
\begin{align}
(\bm{D}-\bm{V}\bm{A}^{-1}\bm{B})^{-1}\bm{V}\bm{A}^{-1} = \bm{D}^{-1}\bm{V}(\bm{A}-\bm{B}\bm{D}^{-1}\bm{V})^{-1}\label{eq:mat-identity}
\end{align}
with $\bm{D}=\lambda\bm{H_r}$, $\bm{V}=-\bm{J}^T$, $\bm{A}=\bm{I}_M$, and $\bm{B}=\bm{U}^T$, and recalling that $\bm{Q}$ is symmetric, from~\eqref{eq:delta_theta} we get
\begin{align}
	\delta\bm{\theta}&=
	\left(\lambda \bm{H}_h-(-\bm{J}^T)\bm{I}_M \bm{U}^T\right)^{-1}(-\bm{J}^T)\bm{I}^{-1}_M\bm{e}\nonumber \\ &=-\frac{1}{\lambda}\bm{H}_h^{-1}\bm{J}^T\left(\bm{I}_M+\bm{U}^T\frac{1}{\lambda}\bm{H}_h^{-1}
	\bm{J}^T\right)^{-1}\bm{e}\nonumber\\
	&=-\bm{H}_h^{-1}\bm{J}^T\left(\lambda\bm{I}_M+\bm{Q}\bm{J}
	\bm{H}_h^{-1}
	\bm{J}^T\right)^{-1}\bm{e}.
	\label{eq:delta_theta_N+1}
\end{align}
\begin{remark}
	When combined with a second identity, namely $\bm{V}\bm{A}^{-1}(\bm{A}-\bm{B}\bm{D}^{-1}\bm{V}) = (\bm{D} -\bm{V}\bm{A}^{-1}\bm{B})\bm{D}^{-1}\bm{V}$, one can directly derive from \eqref{eq:mat-identity} Woodbury identity defined as \cite{higham2002accuracy} $
	(\bm{A}+\bm{U}\bm{B}\bm{V})^{-1} = \bm{A}^{-1} - \bm{A}^{-1}\bm{U}(\bm{B}+\bm{V}\bm{A}^{-1}\bm{U})^{-1}\bm{V}\bm{A}^{-1}$, or in the special case $\bm{B} = -\bm{D}^{-1}$, as $
	(\bm{A}-\bm{U}\bm{D}^{-1}\bm{V})^{-1} = \bm{A}^{-1} + \bm{A}^{-1}\bm{U}(\bm{D}-\bm{V}\bm{A}^{-1}\bm{U})^{-1}\bm{V}\bm{A}^{-1}$. Using Woodbury identity would, in fact, \emph{structurally} not result into the closed-form update step \eqref{eq:delta_theta_N+1} in an exact sense. Our construction involves a more general regularization function than the commonly used square norm, where the Woodbury identity can be equally useful, as its Hessian yields a multiple of the identity matrix.
\end{remark}
Compared to~\eqref{eq:delta_theta}, the clear advantage of the form~\eqref{eq:delta_theta_N+1} is that it requires the factorization of an $M\times M$ matrix rather than an $n_w\times n_w$ matrix, where the term $\bm{H}_h^{-1}$ can be conveniently obtained by exploiting its diagonal structure. Given these modifications, we proceed by making an assumption that defines the residual $\bm{e}$ and the Jacobian $\bm{J}$ in the region of convergence where we assume the starting point $\bm{\theta}_0$ of the process \eqref{eq:delta_theta_N+1} lies.

Let $\bm{\theta}^*$ be a nondegenerate minimizer of $\mathcal{L}$, and define $\mathcal{B}_\epsilon(\bm{\theta}^*)\coloneqq\{\bm{\theta}_k\in\rr^{n_w}\colon\|\bm{\theta}_k - \bm{\theta}^*\| \le \epsilon\}$, a closed ball of a sufficiently small radius $\epsilon\ge0$ about $\bm{\theta}^*$. We denote by $\mathcal{N}_{\epsilon}(\bm{\theta}^*)$ an open neighbourhood of the sublevel set $\Gamma(\mathcal{L})\coloneqq\{\bm{\theta}_k\colon\mathcal{L}(\bm{\theta}_k)\le\mathcal{L}(\bm{\theta}_0)\}$, so that $\mathcal{B}_\epsilon(\bm{\theta}^*)=\mathrm{cl}(\mathcal{N}_{\epsilon}(\bm{\theta}^*))$. We then have $\mathcal{N}_{\epsilon}(\bm{\theta}^*)\coloneqq\{\bm{\theta}_k\in\rr^{n_w}\colon\|\bm{\theta}_k - \bm{\theta}^*\| < \epsilon\}$.
\begin{assumption}\label{thm:ass3}
	\begin{enumerate}[(i)]
		\item Each $\bm{e}_n(\bm{\theta}_k)$ and each $\bm{q}_n(\bm{\theta}_k)$ is Lipschitz smooth, and $\forall \bm{\theta}_k\in\mathcal{N}_{\epsilon}(\bm{\theta}^*)$ there exists $\nu>0$ such that $\|\bm{J}(\bm{\theta}_k)\bm{z}\| \ge \nu\|\bm{z}\|$.
		\item $\lim_{k\to\infty}\left\|\mathbb{E}_m[\bm{H}(\bm{\theta}_k)] - \partial^2\mathcal{L}(\bm{\theta}_k)\right\|=0$ almost surely whenever $\lim_{k\to\infty}\|\bm{g}_{\mathcal{L}}(\bm{\theta}_k)\| = 0$, $\forall \bm{\theta}_k\in\mathcal{N}_{\epsilon}(\bm{\theta}^*)$, where $\mathbb{E}_m[\cdot]$ denotes\footnote{Subsequently, we shall omit the notation for the Hessian and gradient estimates as we assume unbiasedness.} expectation with respect to $m$.
	\end{enumerate}
\end{assumption}
\begin{remark}\label{thm:remJ}
	\assref{thm:ass3}(i) implies that the singular values of $\bm{J}$ are uniformly bounded away from zero and $\exists \beta,\tilde{\beta}>0$ such that $\|\bm{e}\|\le\beta, \|\bm{J}\|=\|\bm{J}^T\| \le\tilde{\beta}$, then as $\bm{Q}\succeq \bm{0}$, we have $\exists K_1$ such that $\bm{Q}\le K_1\bm{I}$, and hence $\|\lambda\bm{I}+\bm{Q}\bm{J}\bm{H}_h^{-1}\bm{J}^T\|\le \lambda+(K/\gamma_a)$, where $K=K_1\tilde{\beta}^2$. Note that although we use limits in \assref{thm:ass3}(ii), the assumption similarly holds in expectation by unbiasedness. Also, a sufficient sample size $m$ may be required for \assref{thm:ass3}(ii) to hold, by law of large numbers, see e.g. \cite[Lemma 1, Lemma 2]{roosta2016sub}.
\end{remark}
\begin{remark}\label{thm:rem2}
	$\bm{H}_g(\bm{\theta})$ and $\bm{H}_h(\bm{\theta})$ satisfy
	\begin{subequations}\label{eq:hessapproxreg}
		\begin{eqnarray*}
			\gamma_l\bm{I}_{n_w}\preceq\bm{H}_g(\bm{\theta}_k)\preceq \gamma_u \bm{I}_{n_w}, \quad \left\|\bm{H}_g(\bm{y}, f(\bm{\theta_1};\bm{x})) - \bm{H}_g(\bm{y}, f(\bm{\theta_2};\bm{x}))\right\| \le \gamma_g\left\|\bm{\theta_1} - \bm{\theta_2}\right\|, \\ \gamma_a\bm{I}_{n_w}\preceq\bm{H}_h(\bm{\theta}_k)\preceq \gamma_b \bm{I}_{n_w}, \quad  \left\|\bm{H}_h(\bm{y}, f(\bm{\theta_1};\bm{x})) - \bm{H}_h(\bm{y}, f(\bm{\theta_2};\bm{x}))\right\| \le \gamma_h\left\|\bm{\theta_1} - \bm{\theta_2}\right\|,
		\end{eqnarray*}
	\end{subequations}
	at any point $\bm{\theta}_k\in\rr^{n_w}$, for any $\bm{x}\in\rr^{n_p},\bm{y}\in\rr^d$, and for any $\bm{\theta_1}, \bm{\theta_2}\in\mathcal{N}_\epsilon(\bm{\theta}^*)$ where \assref{thm:ass3} holds.
\end{remark}
We now state a convergence result for the update type \eqref{eq:delta_theta} (and hence \eqref{eq:delta_theta_N+1}). First, we define the second-order optimality condition and state two useful lemmas.
\begin{definition}[Second-order sufficiency condition (SOSC)]
	Let $\bm{\theta}^*$ be a local minimum of a twice-differentiable function $\mathcal{L}(\cdot)$. The second-order sufficiency condition (SOSC) holds if
	\begin{align}\tag{SOSC}
		\partial\mathcal{L}(\bm{\theta}^*) = 0, \quad \partial^2 \mathcal{L}(\bm{\theta}^*) \succ 0. \label{eq:SOSC}
	\end{align}
\end{definition}
\begin{lemma}[{\cite[Theorem~1.2.3]{nesterov2018lectures}}]\label{thm:minimizer}
	Suppose that \assref{thm:ass2} holds. Let $\mathbb{W}\subseteq \rr^{n_w}$ be a closed and convex set on which $\mathcal{L}(\bm{\theta})$ is twice-continuously differentiable. Let $S\subset\mathbb{W}$ be an open set containing some $\bm{\theta}^*$, and suppose that $\mathcal{L}(\bm{\theta}^*)$ satisfies \eqref{eq:SOSC}. Then, there exists $\mathcal{L}^* = \mathcal{L}(\bm{\theta}^*)$ satisfying
	\begin{align}
		\mathcal{L}(\bm{\theta}_k)>\mathcal{L}^* \quad \forall\bm{\theta}_k \in S.
	\end{align}
\end{lemma}

\begin{lemma}\label{thm:descent}
	The adjustment $\delta\bm{\theta}$ given by \eqref{eq:delta_theta} (and hence, \eqref{eq:delta_theta_N+1}) provides a descent direction for the total loss $\mathcal{L}(\bm{\theta}_k)$ in \eqref{eq:emp} at the $k^{th}$ oracle call.
\end{lemma}
\begin{remark}\label{thm:rem4}
	\lemref{thm:minimizer}, \lemref{thm:descent} and the first part of \eqref{eq:SOSC} ensure that the second part of \assref{thm:ass3}(ii) always holds. In essence, it holds at every point $\bm{\theta}_k$ of the sequence $\{\bm{\theta}_k\}$ generated by the process \eqref{eq:delta_theta_N+1} as long as we choose a starting point $\bm{\theta_0}\in\mathcal{B}_\epsilon(\bm{\theta}^*)$.
\end{remark}
\begin{theorem}\label{thm:basicconv}
	Suppose that Assumptions \ref{thm:ass1}, \ref{thm:ass1-1}, \ref{thm:ass2} and \ref{thm:ass3} hold, and that $\bm{\theta}^*$ is a local minimizer of $\mathcal{L}(\bm{\theta})$ for which the assumptions in \lemref{thm:minimizer} hold. Let $\{\bm{\theta}_k\}$ be the sequence generated by the process \eqref{eq:delta_theta_N+1}. Then starting from a point $\bm{\theta}_0\in\mathcal{N}_{\epsilon}(\bm{\theta}^*)$, $\{\bm{\theta}_k\}$ converges at a $Q$-quadratic rate. Namely:
	\begin{align*}
		\left\|\bm{\theta}_{k+1}-\bm{\theta}^*\right\| \le \xi_k\left\|\bm{\theta}_k-\bm{\theta}^*\right\|^2,
	\end{align*}
	where
	\begin{align*}
		\xi_k = \frac{1}{2}\frac{\gamma_g + b\gamma_h}{\left(\gamma_l+a\gamma_a - (\gamma_g+b\gamma_h)\left\|\bm{\theta}_k-\bm{\theta}^*\right\|\right)}.
	\end{align*}
\end{theorem}
The proofs of these results are reported in \appref{sec:appA} and \appref{sec:appB}.
\section{Self-concordant Regularization (SCORE)}\label{ss:SCORE}
In the case $dm < n_w$, one could deduce by mere visual inspection that the matrix $\bm{H}_h^{-1}$ in \eqref{eq:delta_theta_N+1} plays an important role in the perturbation of $\bm{\theta}$ within the region of convergence due to its ``dominating" structure in the equation. This may be true as it appears. But beyond this ``naive" view, let us derive a more technical intuition about the update step \eqref{eq:delta_theta_N+1} by using a similar analogy as that given in \cite[Chapter 5]{nesterov2018lectures}. We have
\begin{align}
	\bm{\theta}_{k+1}=\bm{\theta}_k-\bm{H}_h^{-1}\bm{J}^T\left(\lambda\bm{I}+\bm{Q}\bm{J}
	\bm{H}_h^{-1}
	\bm{J}^T\right)^{-1}\bm{e}, \label{eq:updatemethod}
\end{align}
where $\bm{I}$ is the identity matrix of suitable dimension. By some simple algebra (see e.g., \cite[Lemma 5.1.1]{nesterov2018lectures}), one can show that indeed the update method \eqref{eq:updatemethod} is affine-invariant. In other words, the region of convergence $\mathcal{N}_{\epsilon}(\bm{\theta}^*)$ does not depend on the problem scaling but on the local topological structure of the minimization functions $g(\bm{\theta})$ and $h(\bm{\theta})$ \cite{nesterov2018lectures}.

Consider the non-augmented form of \eqref{eq:updatemethod}: $\bm{\theta}_{k+1}=\bm{\theta}_k-\bm{Q}_g\bm{J}_g^T\bm{G}_g^{-1}\bm{e}_g$, where $\bm{G}_g = \bm{J}_g\bm{J}_g^T$ is the so-called \textit{Gram matrix}. It has been shown that indeed when learning an overparameterized model (sample size smaller than number of variables), and as long as we start from an initial point $\bm{\theta}_0$ close enough to the minimizer $\bm{\theta}^*$ of $\mathcal{L}(\bm{\theta})$ (assuming that such a minimizer exists), both $\bm{J}_g$ and $\bm{G}_g$ remain stable throughout the learning process (see e.g., \cite{cai2019gram, du2019gradient, du2018gradient}). The term $\bm{Q}_g$ may have little to no effect on this notion, for example, in case of the squared fit loss, $\bm{Q}_g$ is just an identity term. However, the original Equation \eqref{eq:updatemethod} involves the Hessian matrix $\bm{H}_h$ which, together with its bounds (namely, $\gamma_a$, $\gamma_b$ characterizing the region of convergence), changes rapidly when we scale the problem \cite{nesterov2018lectures}. It therefore suffices to impose an additional assumption on the function $h(\bm{\theta})$ that will help to control the rate at which its second-derivative $\bm{H}_h$ changes, thereby enabling us to characterize an affine-invariant structure for the region of convergence. Namely:
\begin{assumption}[SCORE]\label{thm:ass4}
	The scaled regularization function $\lambda h(\bm{\theta})$ has a third derivative and is $M_h$-self-concordant. That is, the inequality
	\begin{align}
		\left\vert\left<\bm{u},\left(\partial^3h(\bm{\theta})[\bm{u}]\right)\bm{u}\right>\right\vert \le 2M_h\left<\bm{u},\partial^2h(\bm{\theta})\bm{u}\right>^{3/2}, \label{eq:sc}
	\end{align}
	where $M_h\ge0$, holds for any $\bm{\theta}$ in the closed convex set $\mathbb{W}\subseteq\rr^{n_w}$ and $\bm{u}\in\rr^{n_w}$.
\end{assumption}
Here, $\partial^3h(\bm{\theta})[\bm{u}]\in \rr^{n_w \times n_w}$ denotes the limit
\begin{align*}
	\partial^3h(\bm{\theta})[\bm{u}] \coloneqq \lim\limits_{t\to 0}\frac{1}{t}\left(\partial^2(\bm{\theta}+t\bm{u})-\partial^2(\bm{\theta})\right), \quad t \in \rr.
\end{align*}
For a detailed analysis of self-concordant functions, we refer the interested reader to \cite{nesterov1994interior, nesterov2018lectures}.

Given this \emph{extra} assumption, and following the idea of Newton-decrement in \cite{nesterov1994interior}, we propose to update $\bm{\theta}$ by
\begin{align}
	\delta\bm{\theta}=-\frac{\alpha}{1+M_h\eta_k}\bm{H}_h^{-1}\bm{J}^T\left(\lambda\bm{I}+\bm{Q}\bm{J}
	\bm{H}_h^{-1}
	\bm{J}^T\right)^{-1}\bm{e}, \label{eq:proposed}
\end{align}
where $\alpha>0$, $\eta_k = \left<\bm{g}_h,\bm{H}_h^{-1}\bm{g}_h\right>^{1/2}$ and $\bm{g}_h$ is the gradient of $h(\bm{\theta})$. The proposed method which we call GGN-SCORE is summarized, for one oracle call, in Algorithm \ref{alg:GGN-SCORE}.
\begin{algorithm}
	\caption{GGN-SCORE}
	\label{alg:GGN-SCORE}
	\begin{algorithmic}[1]
		\STATE {\bfseries Input:} variables vector $\bm{\theta}_k$, data $\{(\bm{x}_n,\bm{y}_n)\}_{n=1}^m$, $\bm{H}_h$ (see \eqref{eq:mggn}), $\bm{Q}$ (see \eqref{eq:mggn}), $\bm{J}$ (see \eqref{eq:mggn}), $\bm{e}$ (see \eqref{eq:gaug}), parameters $\alpha,M_h,\lambda$
		\STATE {\bfseries Output:} variables vector $\bm{\theta}_{k+1}$
		\STATE Compute $\bm{g}_h = \partial_{\bm{\theta}_k} h(\bm{\theta}_k)$
		\STATE Choose $\eta_k = \left<\bm{g}_h,\bm{H}_h^{-1}\bm{g}_h\right>^{1/2}$
		\STATE Set $\rho_k = \frac{\alpha}{1+M_h\eta_k}$ \label{alg:step-size}
		\STATE Solve $\left(\lambda\bm{I}+\bm{Q}\bm{J}
		\bm{H}_h^{-1}
		\bm{J}^T\right)\bm{b}=\bm{e}$ for $\bm{b}$
		\STATE Set $\bm{G}=\bm{H}_h^{-1}\bm{J}^T\bm{b}$
		\STATE Compute $\bm{\theta}_{k+1} = \bm{\theta}_k - \rho_k \bm{G}$
	\end{algorithmic}
\end{algorithm}
There is a wide class of self-concordant functions that meet practical requirements, either in their scaled form \cite[Corollary 5.1.3]{nesterov2018lectures} or in their original form (see e.g., in \cite{sun2019generalized} for a comprehensive list). Two of them are used in our experiments in \Secref{ss:experiments}.

In the following, we state a local convergence result for the process \eqref{eq:proposed}. We introduce the notation $\|\Delta\|_{\bm{\theta}}$ to represent the local norm of a direction $\Delta$ taken with respect to the local Hessian of a self-concordant function $h$, $\|\Delta\|_{\bm{\theta}}\coloneqq\left<\Delta,\partial_{\bm{\theta}\bm{\theta}}^2h(\bm{\theta})\Delta\right>^{1/2} = \left\|\left[\partial_{\bm{\theta}\bm{\theta}}^2h(\bm{\theta})\right]^{1/2}\Delta\right\|$. Hence, without loss of generality, a ball $\mathcal{B}_\epsilon(\bm{\theta}^*)$ of radius $\epsilon$ about $\bm{\theta}^*$ is taken with respect to the local norm for the function $h$ in the result that follows.
\begin{theorem}\label{thm:main}
	Let Assumptions \ref{thm:ass1}, \ref{thm:ass1-1}, \ref{thm:ass2}, \ref{thm:ass3} and \ref{thm:ass4} hold, and let $\bm{\theta}^*$ be a local minimizer of $\mathcal{L}(\bm{\theta})$ for which the assumptions in \lemref{thm:minimizer} hold. Let $\{\bm{\theta}_k\}$ be the sequence generated by Algorithm \ref{alg:GGN-SCORE} with $\alpha = \frac{\sqrt{\gamma_a}}{\beta_1}(K+\lambda\gamma_a)$, $\beta_1=\beta\tilde{\beta}$. Then starting from a point $\bm{\theta}_0\in\mathcal{N}_{M_h^{-1}}(\bm{\theta}^*)$, $\{\bm{\theta}_k\}$ converges to $\bm{\theta}^*$ according to the following descent properties:
	\begin{align*}
		&\mathbb{E}[\mathcal{L}(\bm{\theta}_{k+1})] \le \mathcal{L}(\bm{\theta}_k) - \left(\frac{\lambda}{M_h^2}\omega(\zeta_k)+\frac{\gamma_l}{2\gamma_a}\omega''(\tilde{\zeta}_k)-\xi\right),\\
		&\mathbb{E}\|\bm{\theta}_{k+1}-\bm{\theta}^*\|_{\bm{\theta}_{k+1}} \le \vartheta\|\bm{\theta}_k-\bm{\theta}^*\|_{\bm{\theta}_k} + \frac{\gamma_u}{\beta_1}\|\bm{\theta}_k-\bm{\theta}^*\| + \frac{\gamma_g}{2}\|\bm{\theta}_k-\bm{\theta}^*\|^2,
	\end{align*}
	where $\omega(\cdot)$ is an auxiliary univariate function defined by $\omega(t)\coloneqq t-\ln(1+t)$ and has a second derivative $\omega''(t)=1/(1+t)^2$, and
	\begin{align*}
		&\zeta_k \coloneqq \frac{M_h}{1+M_h\eta_k}, \quad \tilde{\zeta}_k \coloneqq M_h\eta_k, \quad \xi \coloneqq \frac{2(\gamma_u+\lambda\gamma_b)}{\sqrt{\gamma_a}},\\
		&\vartheta \coloneqq 1+\frac{\lambda}{\sqrt{\gamma_a}\beta_1(1-M_h\|\bm{\theta}_k-\bm{\theta}^*\|_{\bm{\theta}_k})}.
	\end{align*}
\end{theorem}
\noindent
The proof is provided in \appref{sec:appB}. The results of \thmref{thm:main} combine strong convexity and smoothness properties of both $g(\bm{\theta})$ and $h(\bm{\theta})$, and requires that only $h(\bm{\theta})$ is self-concordant.
\section{Experiments}\label{ss:experiments}
In this section, we validate the efficiency of GGN-SCORE (Algorithm \ref{alg:GGN-SCORE}) in solving the problem of minimizing a regularized strongly convex quadratic function and in solving binary classification tasks. For the binary classification tasks, we use real datasets: $\verb*|ijcnn1|$ and $\verb*|w2a|$ from the LIBSVM repository \cite{chang2011libsvm} and $\verb*|dis|$, $\verb*|hypothyroid|$ and $\verb*|coil2000|$ from the PMLB repository \cite{romano2021pmlb}, with an $80$:$20$ train:test split each. The datasets are summarized in Table \ref{tab:datasetsummary}. In each classification task, the model with a sigmoidal output $\hat{\bm{y}}_n$ is trained using the cross-entropy fit loss function $\ell(\bm{y}_n,\hat{\bm{y}}_n) = \frac{1}{2}\sum_{n=1}^N \bm{y}_n\log\left(\frac{1}{\hat{\bm{y}}_n}\right) + (\bm{1}-\bm{y}_n)\log\left(\frac{1}{\bm{1}-\hat{\bm{y}}_n}\right)$, and the ``deviance" residual \cite{dunn2018generalized} $\bm{e}_n = (-1)^{\bm{y}_n+1}\sqrt{-2\left[\bm{y}_n\log(\hat{\bm{y}}_n) + (\bm{1}-\bm{y}_n)\log(\bm{1}-\hat{\bm{y}}_n)\right]}$, $\bm{y}_n, \hat{\bm{y}}_n\in\{\bm{0},\bm{1}\}$.
\begin{table*}
	\caption{Real datasets: $N =$ number of data samples, $n_p=$ number of features, $d=$ number of targets.}
	\label{tab:datasetsummary}
	%	\vskip 0.15in
	\begin{center}
		\begin{small}
			%			\begin{sc}
				\begin{tabular}{ccccr}
					\toprule
					\thead{dataset}&\thead{$N$} & \thead{$n_p$\\before \& after\\feature mapping} & \thead{$d$}\\
					\midrule
					\multirow{2}{*}{\texttt{dis}}&\multirow{2}{*}{$3772$}& $29$& \multirow{2}{*}{$1$} \\
					&& $3017$&\\
					\midrule
					\multirow{2}{*}{\texttt{hypothyroid}}&\multirow{2}{*}{$3163$}& $25$& \multirow{2}{*}{$1$} \\
					&& $2530$&\\
					\midrule
					\multirow{2}{*}{\texttt{w2a}}&\multirow{2}{*}{$3470$}& $300$& \multirow{2}{*}{$1$} \\
					&& $2776$&\\
					\midrule
					\multirow{2}{*}{\texttt{ijcnn1}}&\multirow{2}{*}{$35000$}& $22$& \multirow{2}{*}{$1$} \\
					&& $28000$&\\
					\midrule
					\multirow{2}{*}{\texttt{coil2000}}&\multirow{2}{*}{$9822$}& $85$& \multirow{2}{*}{$1$} \\
					&& $7857$&\\
					\bottomrule
				\end{tabular}
				%			\end{sc}
		\end{small}
	\end{center}
	%	\vskip -0.1in
\end{table*}
\begin{figure*}
	\centering
	\vspace{1.4\baselineskip}
	\begin{subfigure}{1.0\textwidth}
		\vspace{-1.7\baselineskip}
		\begin{multicols}{4}
			\centerline{\includegraphics[width=1.2\linewidth]{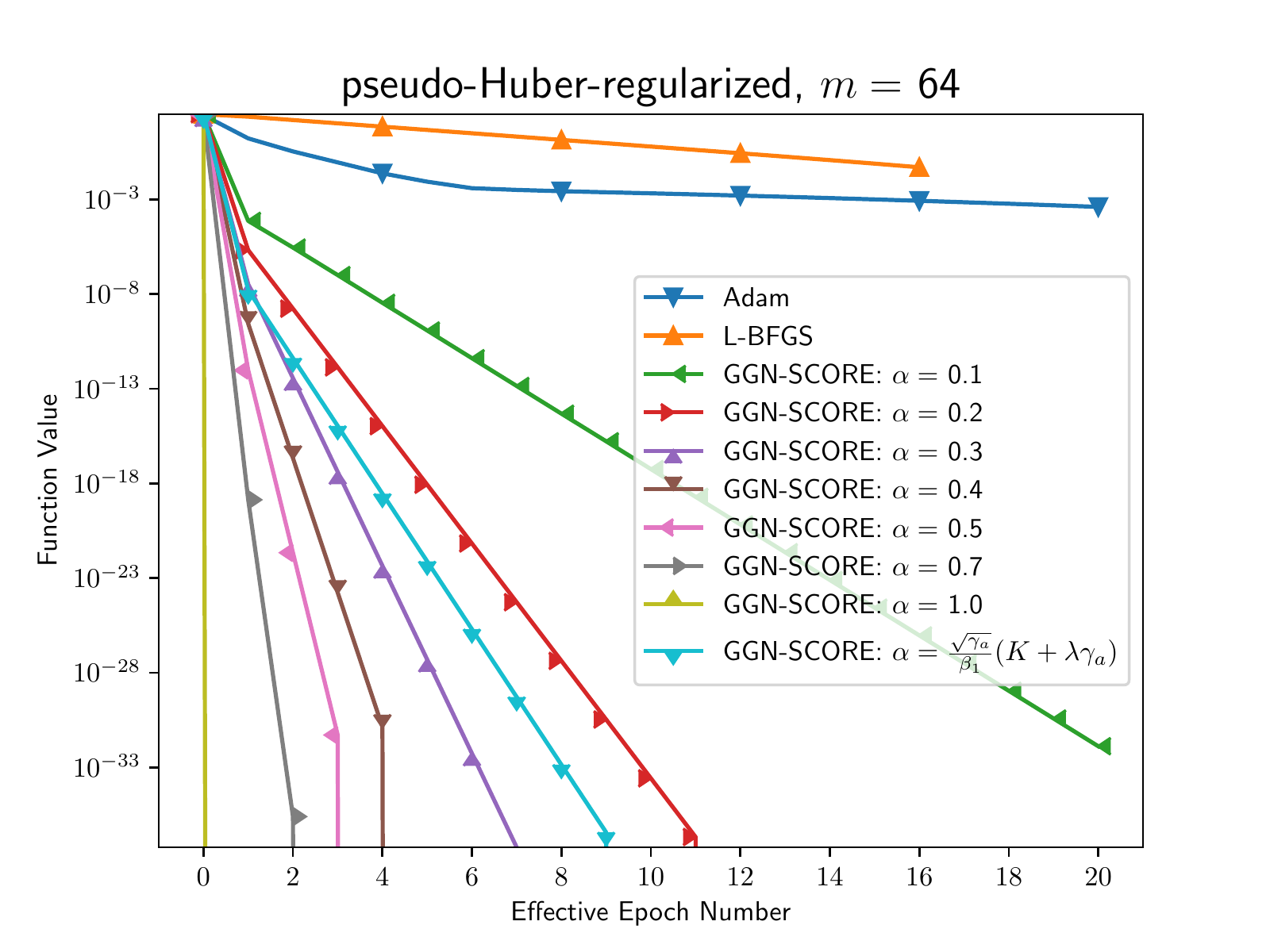}}
			\centerline{\includegraphics[width=1.2\linewidth]{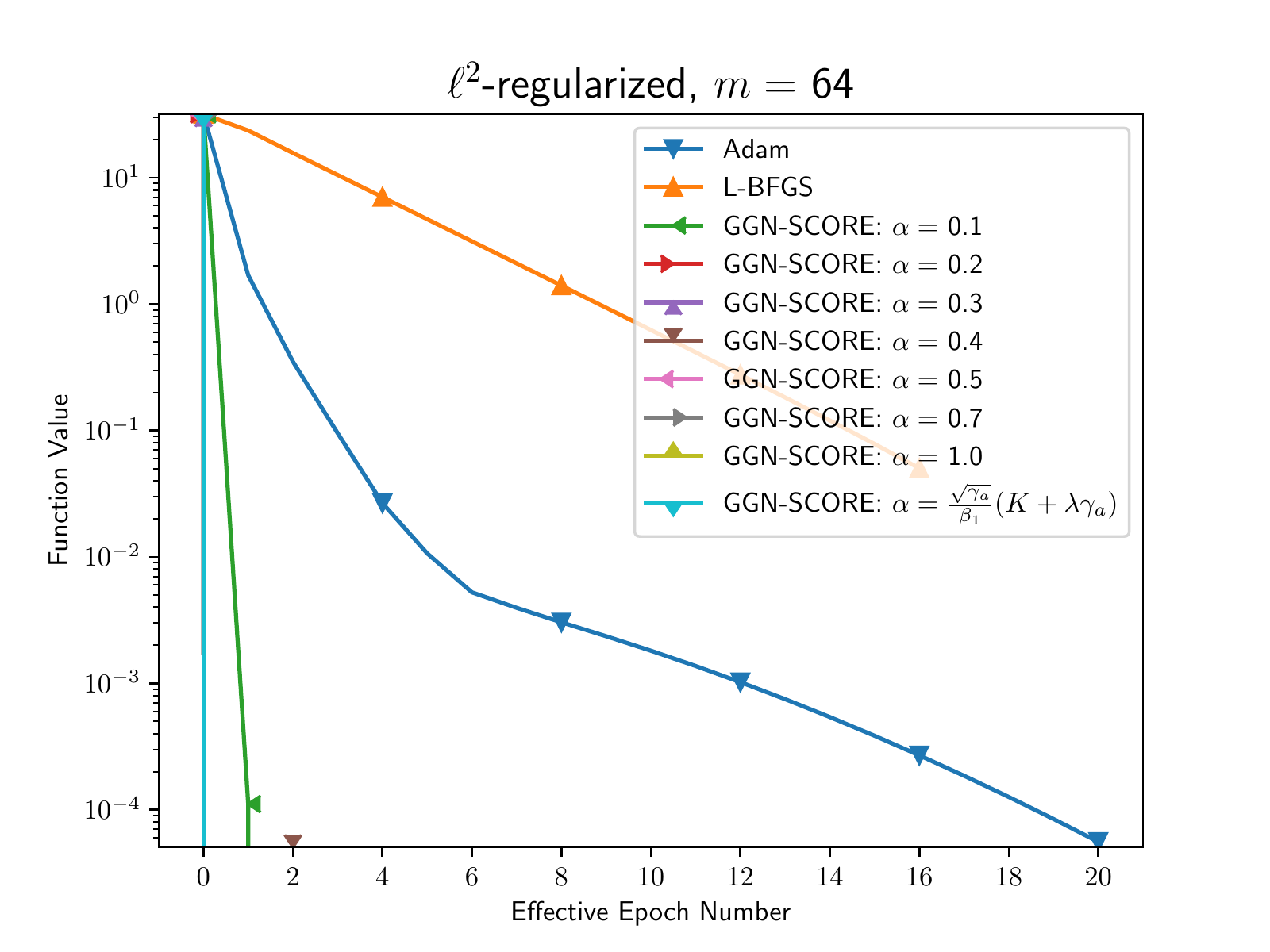}}
			\centerline{\includegraphics[width=1.2\linewidth]{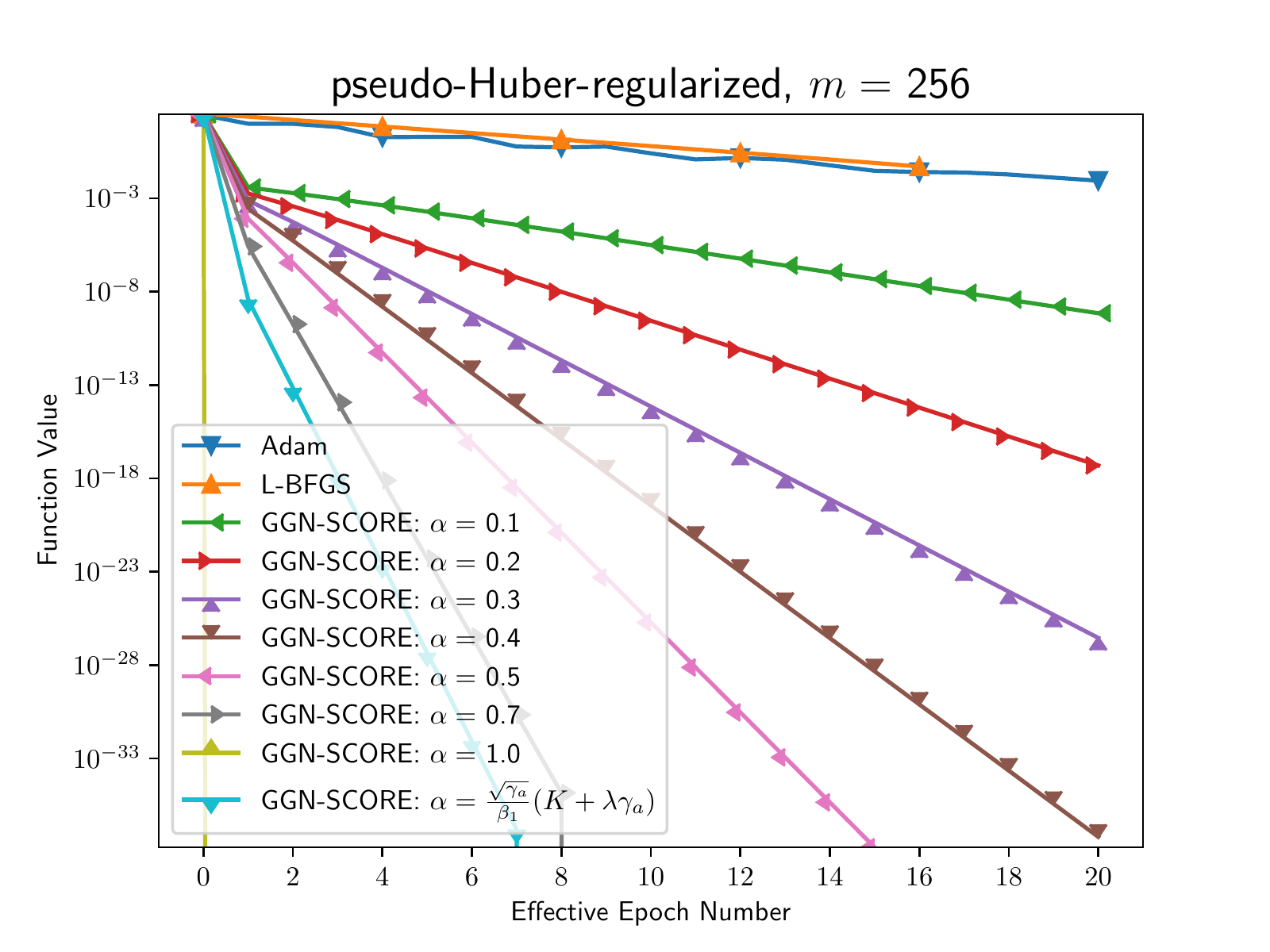}}
			\centerline{\includegraphics[width=1.2\linewidth]{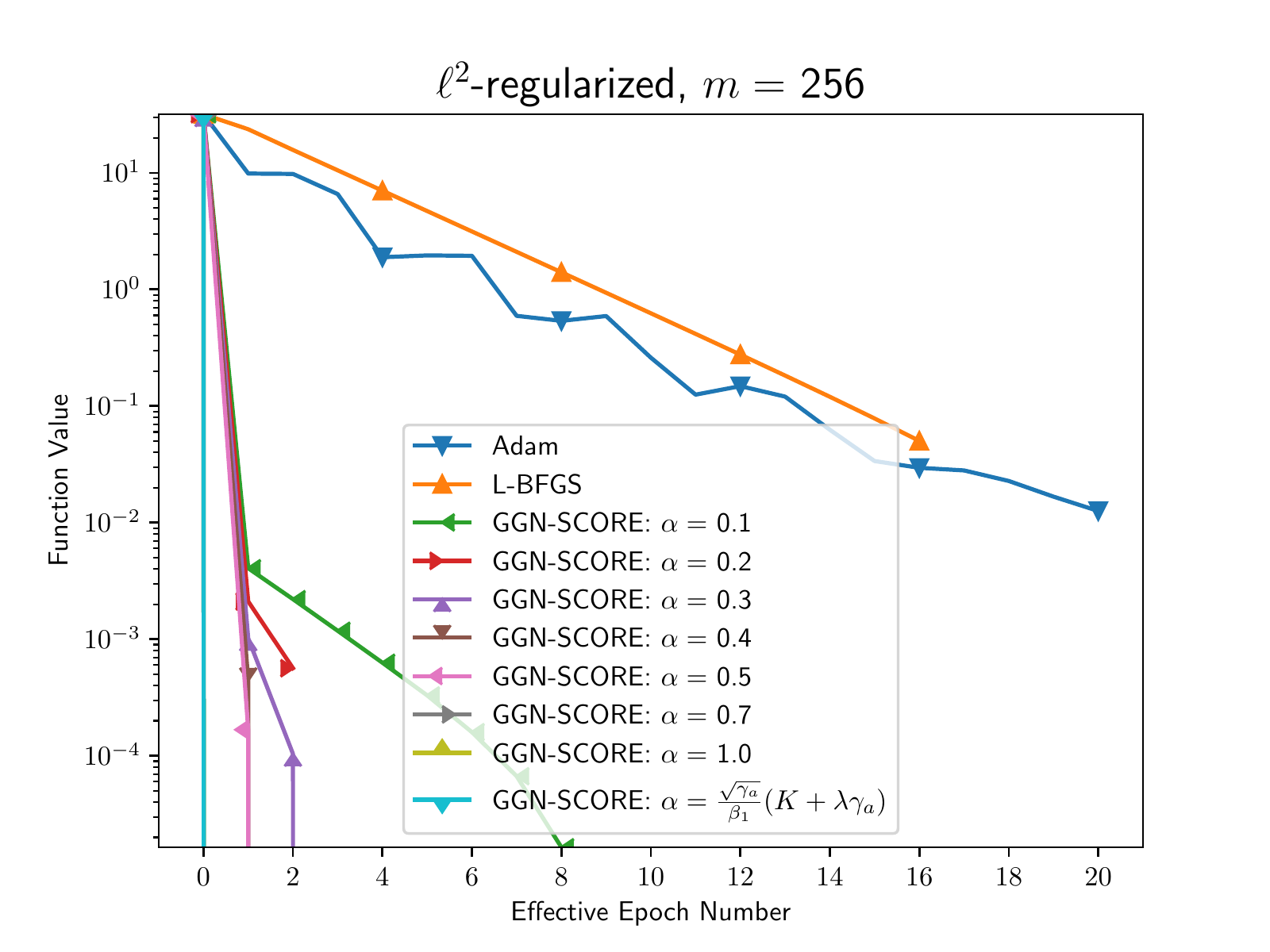}}
			\centerline{\includegraphics[width=1.2\linewidth]{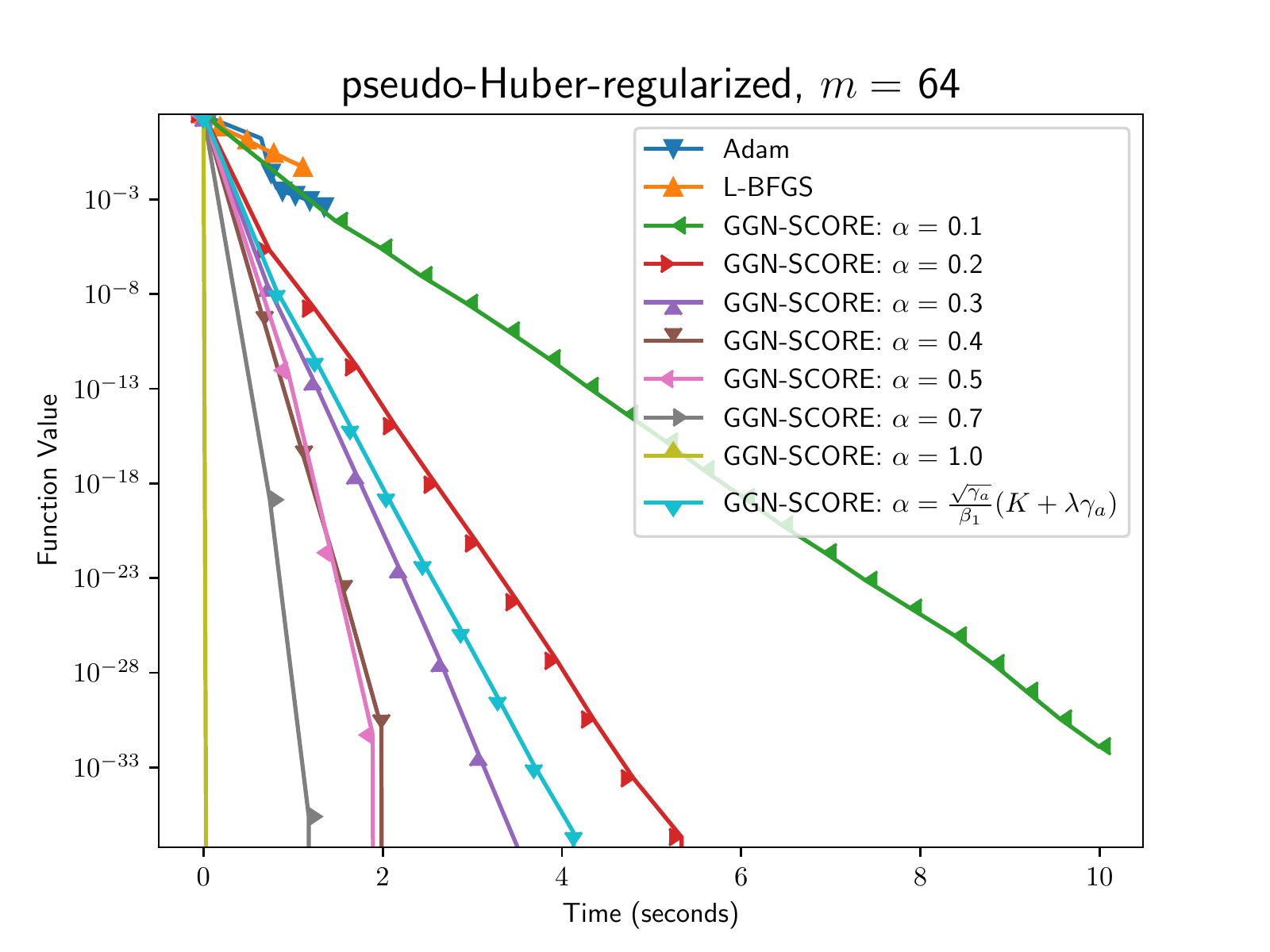}}
			\centerline{\includegraphics[width=1.2\linewidth]{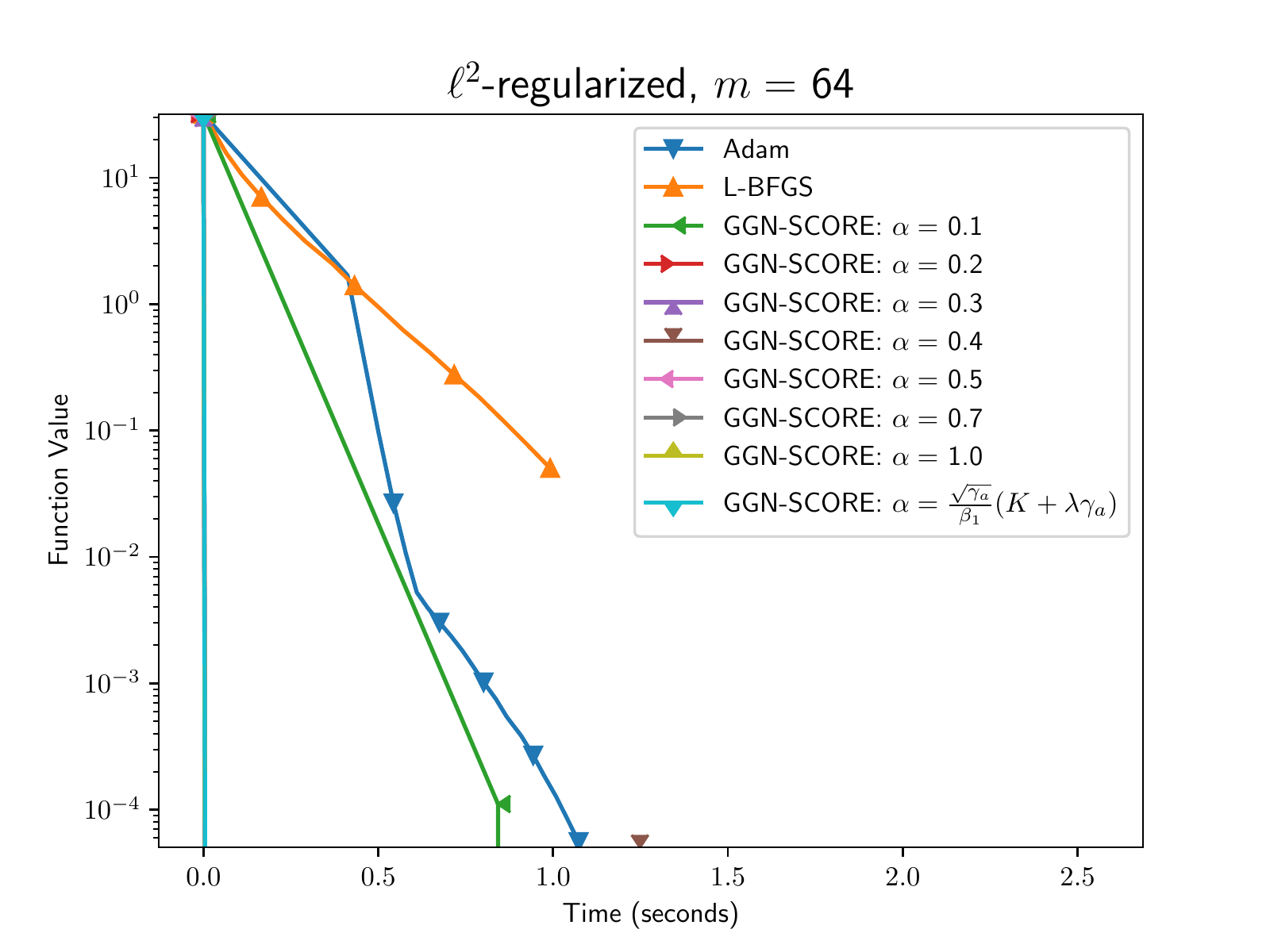}}
			\centerline{\includegraphics[width=1.2\linewidth]{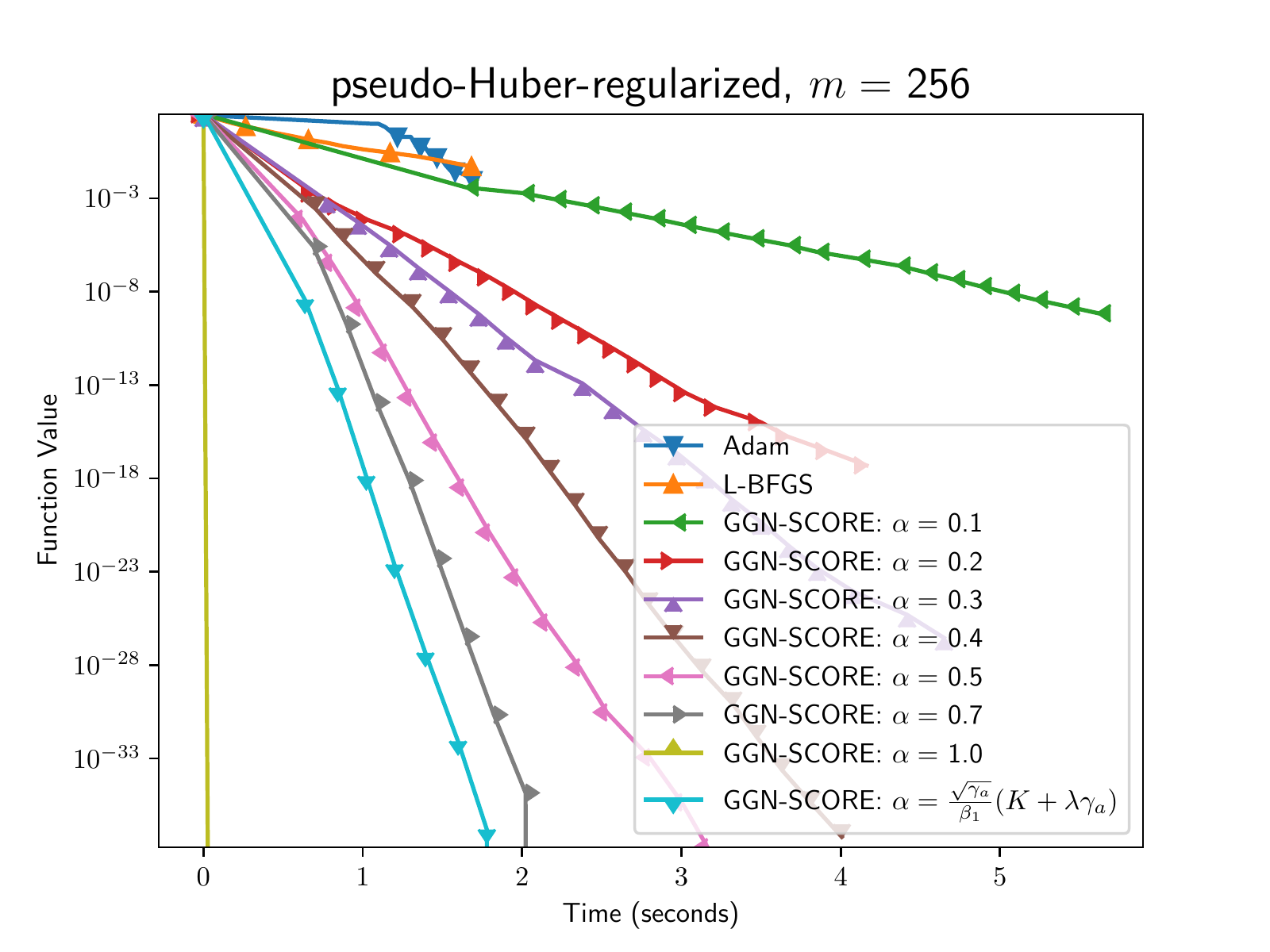}}
			\centerline{\includegraphics[width=1.2\linewidth]{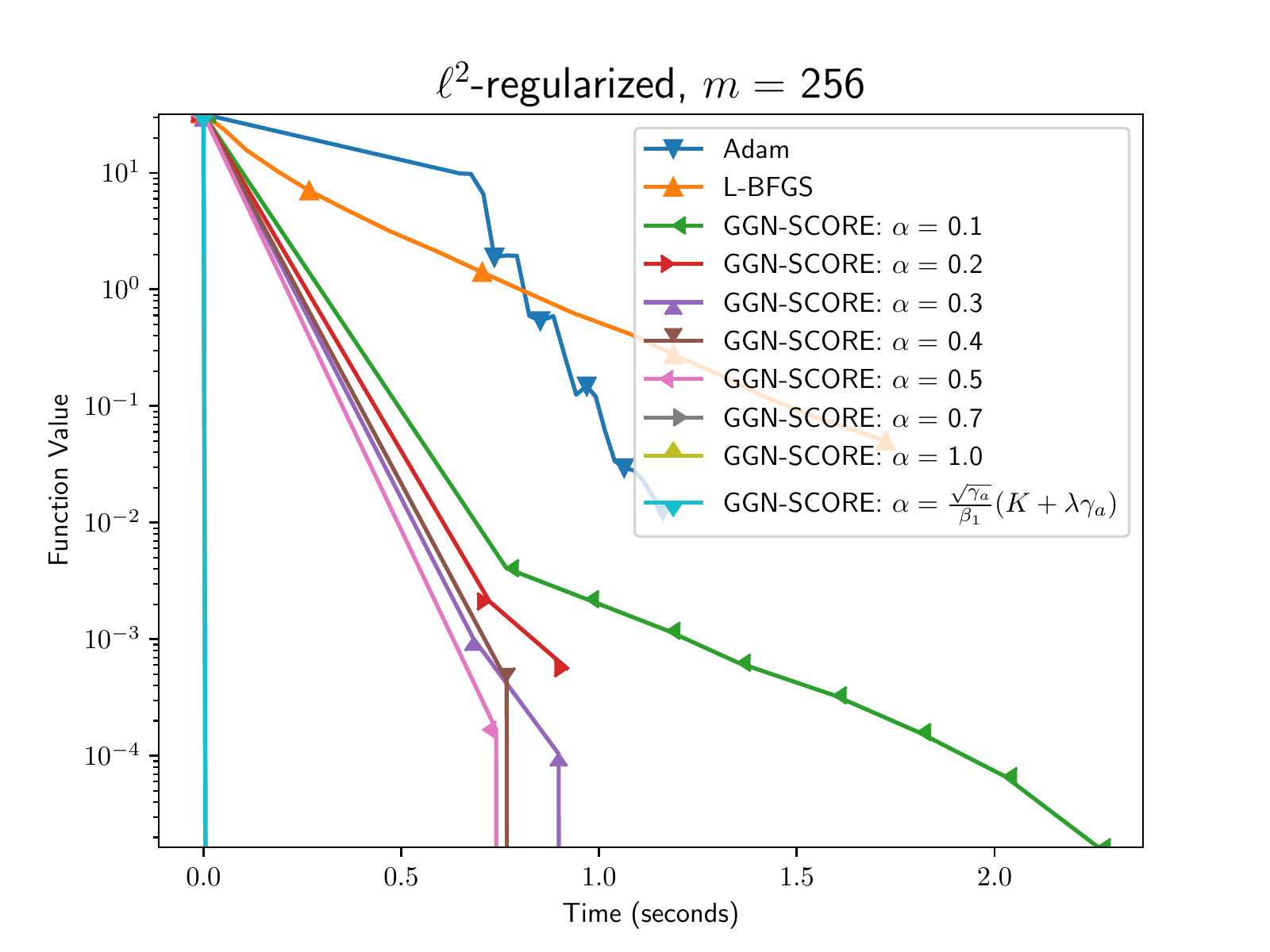}}
			\centerline{\includegraphics[width=1.2\linewidth]{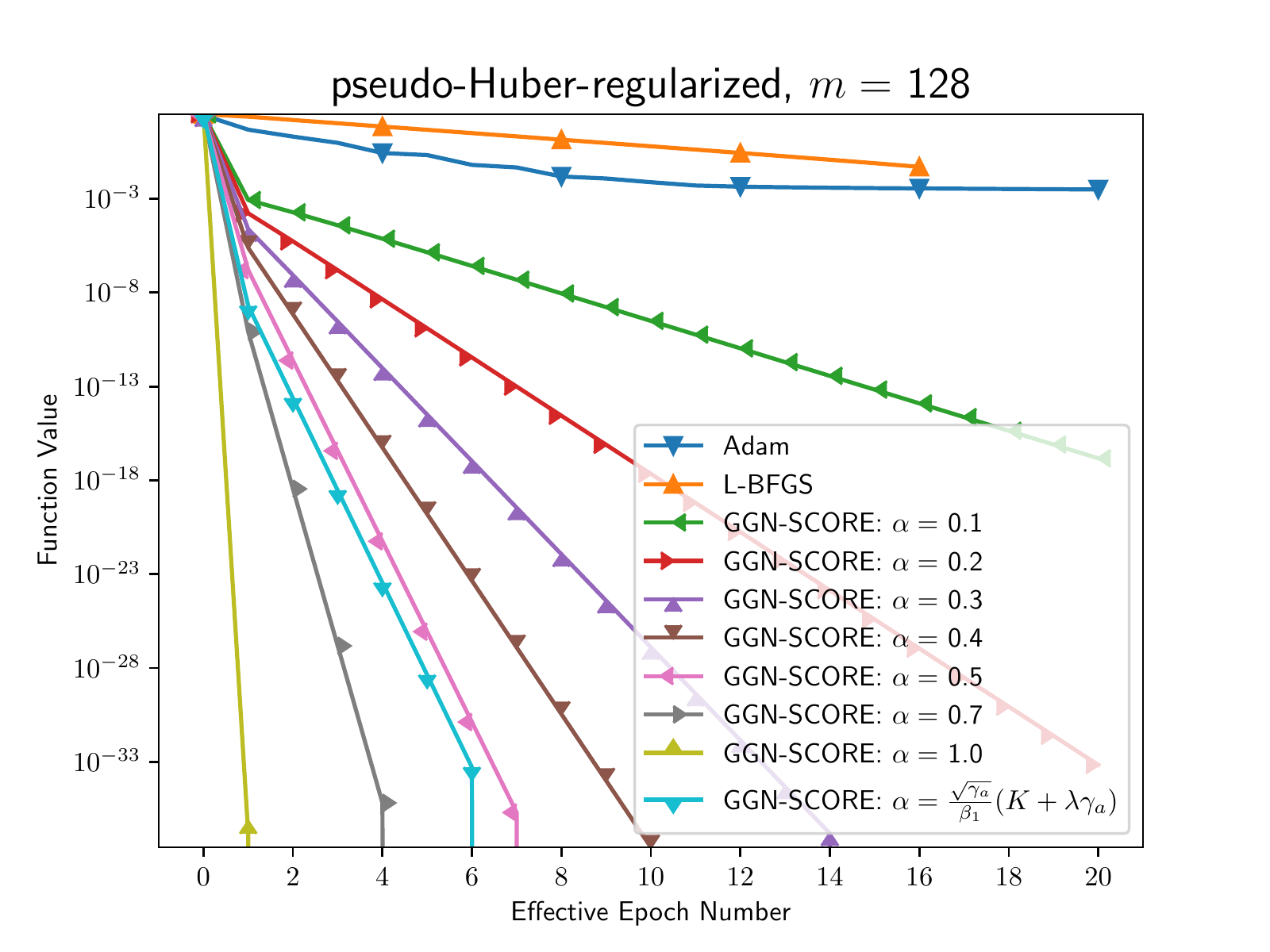}}
			\centerline{\includegraphics[width=1.2\linewidth]{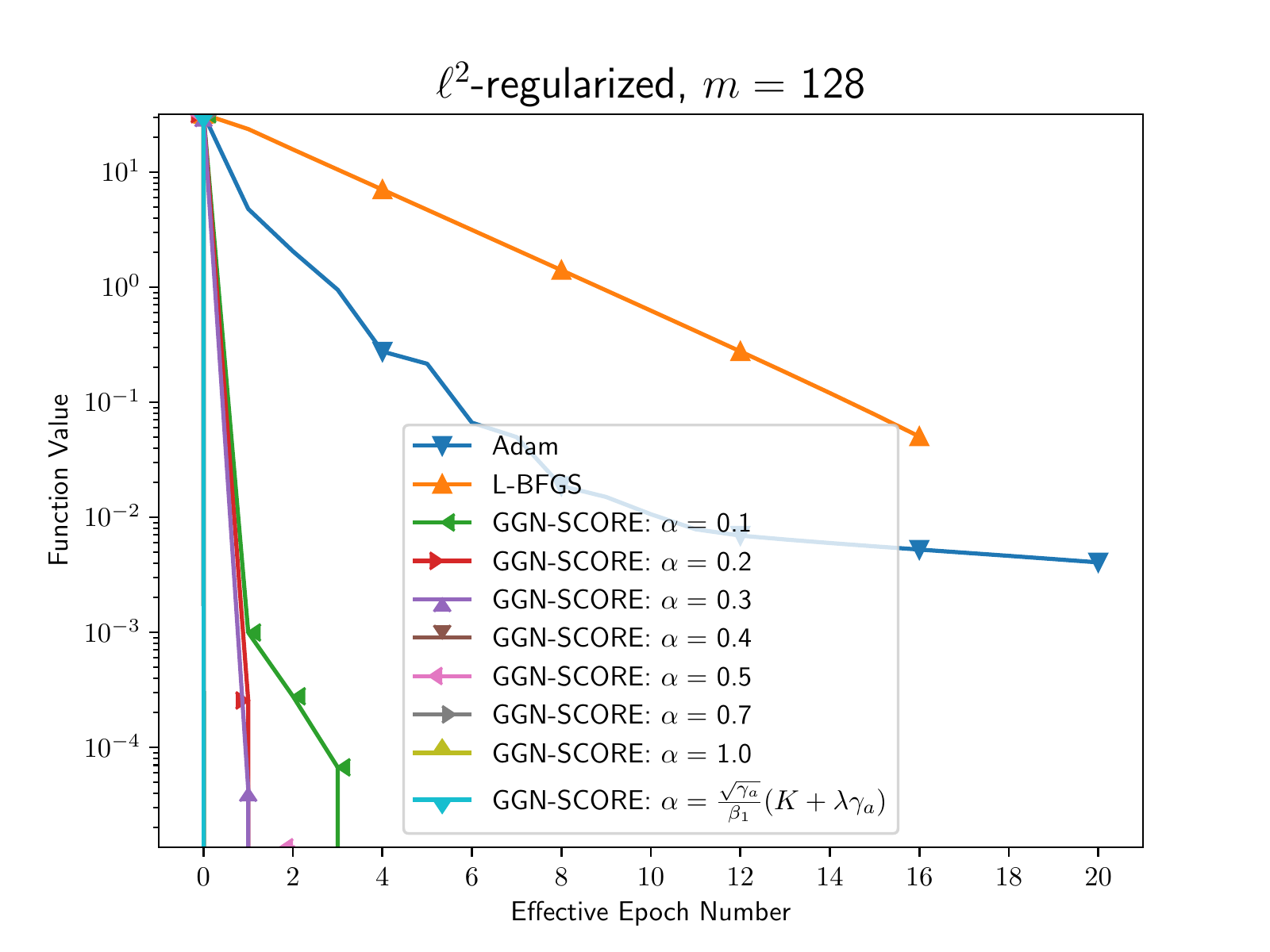}}
			\centerline{\includegraphics[width=1.2\linewidth]{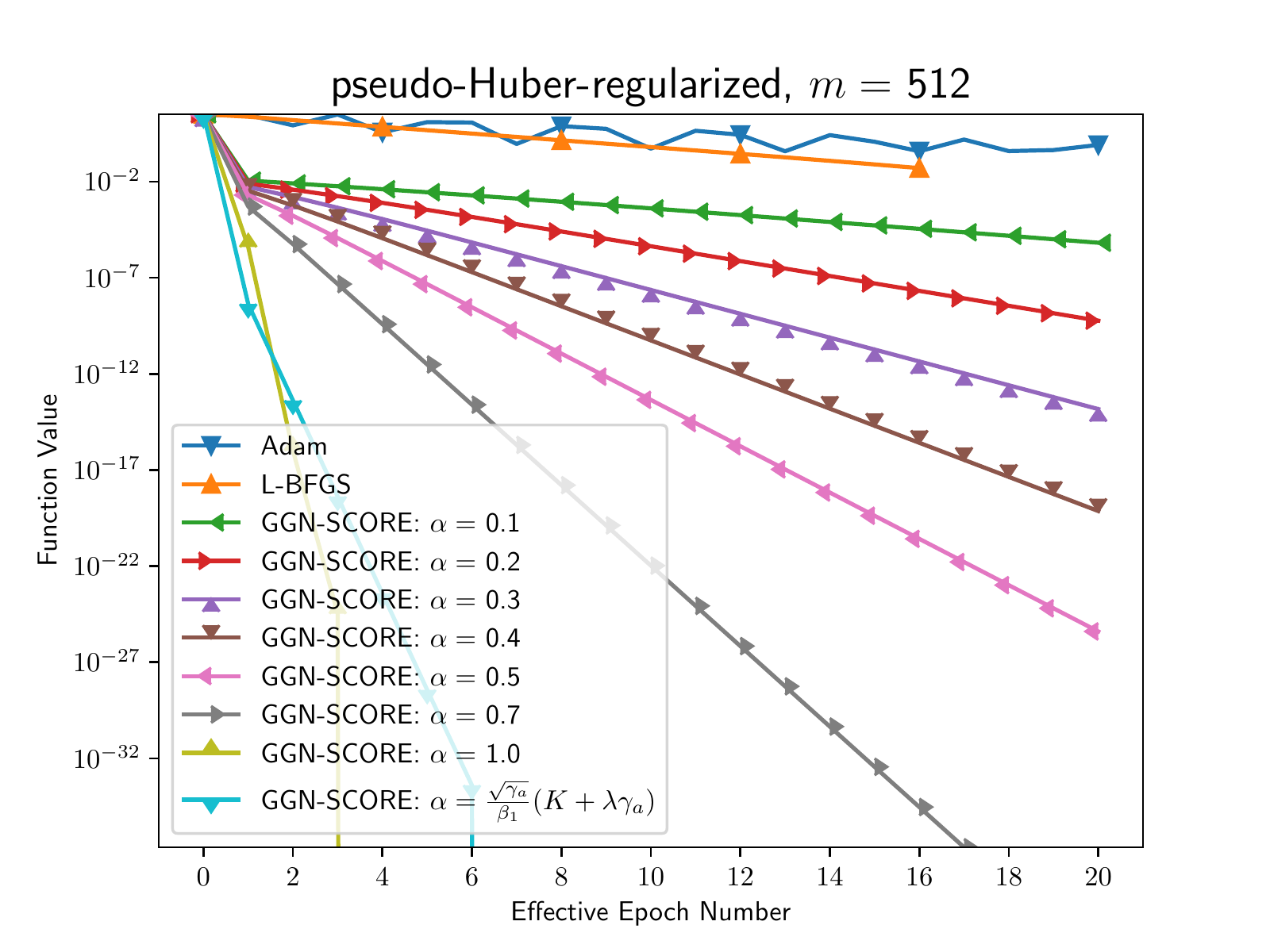}}
			\centerline{\includegraphics[width=1.2\linewidth]{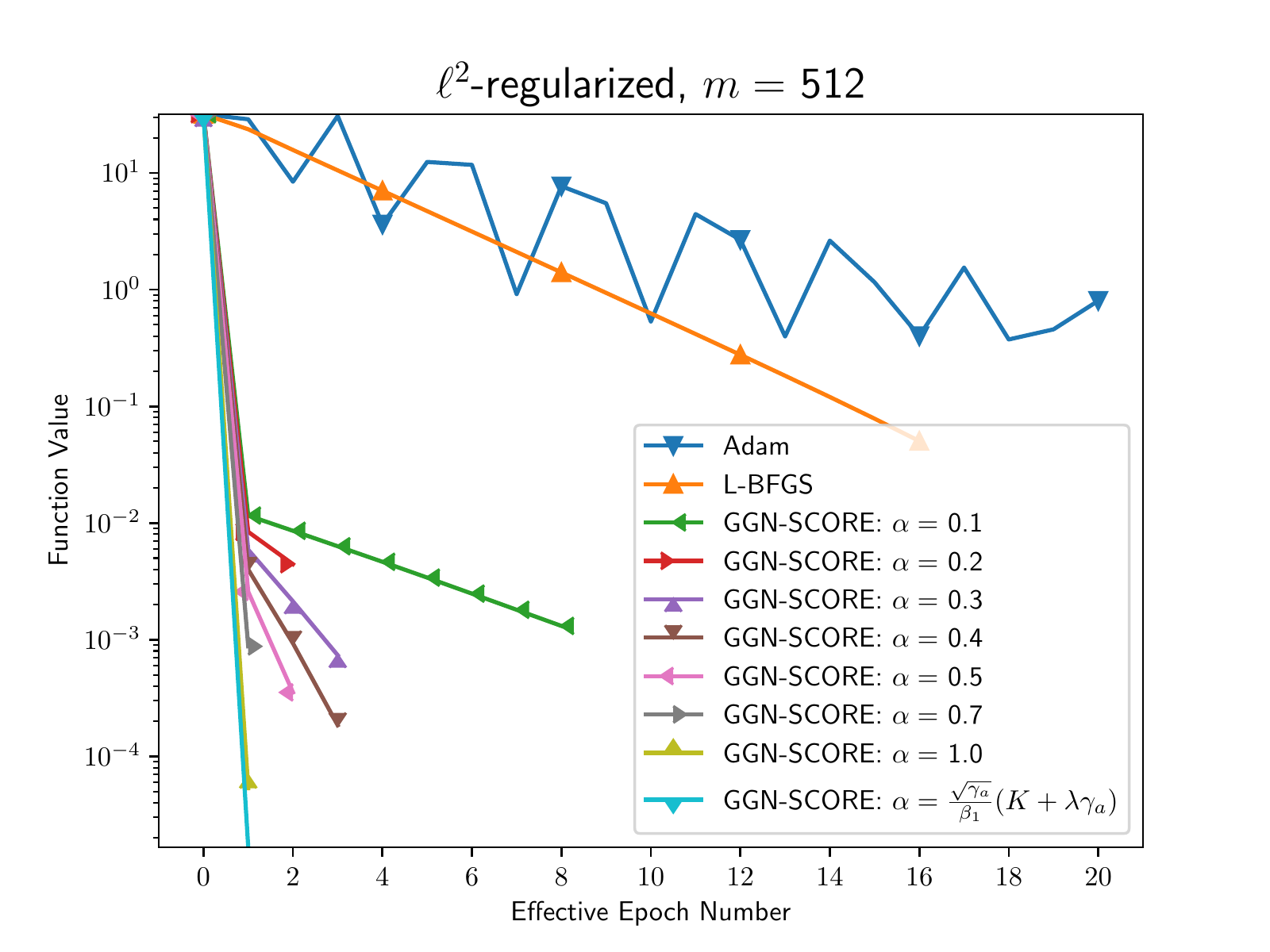}}
			\centerline{\includegraphics[width=1.2\linewidth]{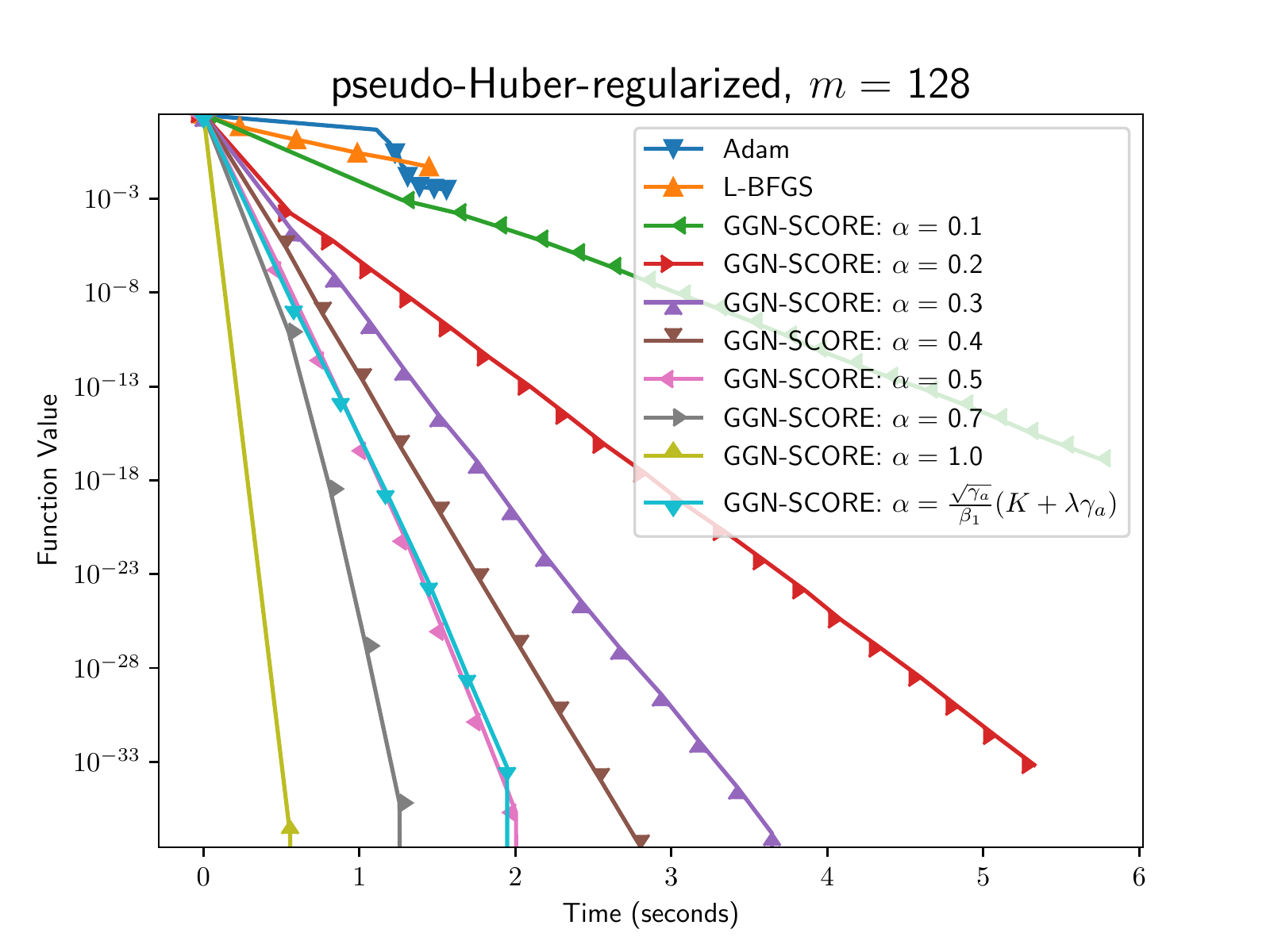}}
			\centerline{\includegraphics[width=1.2\linewidth]{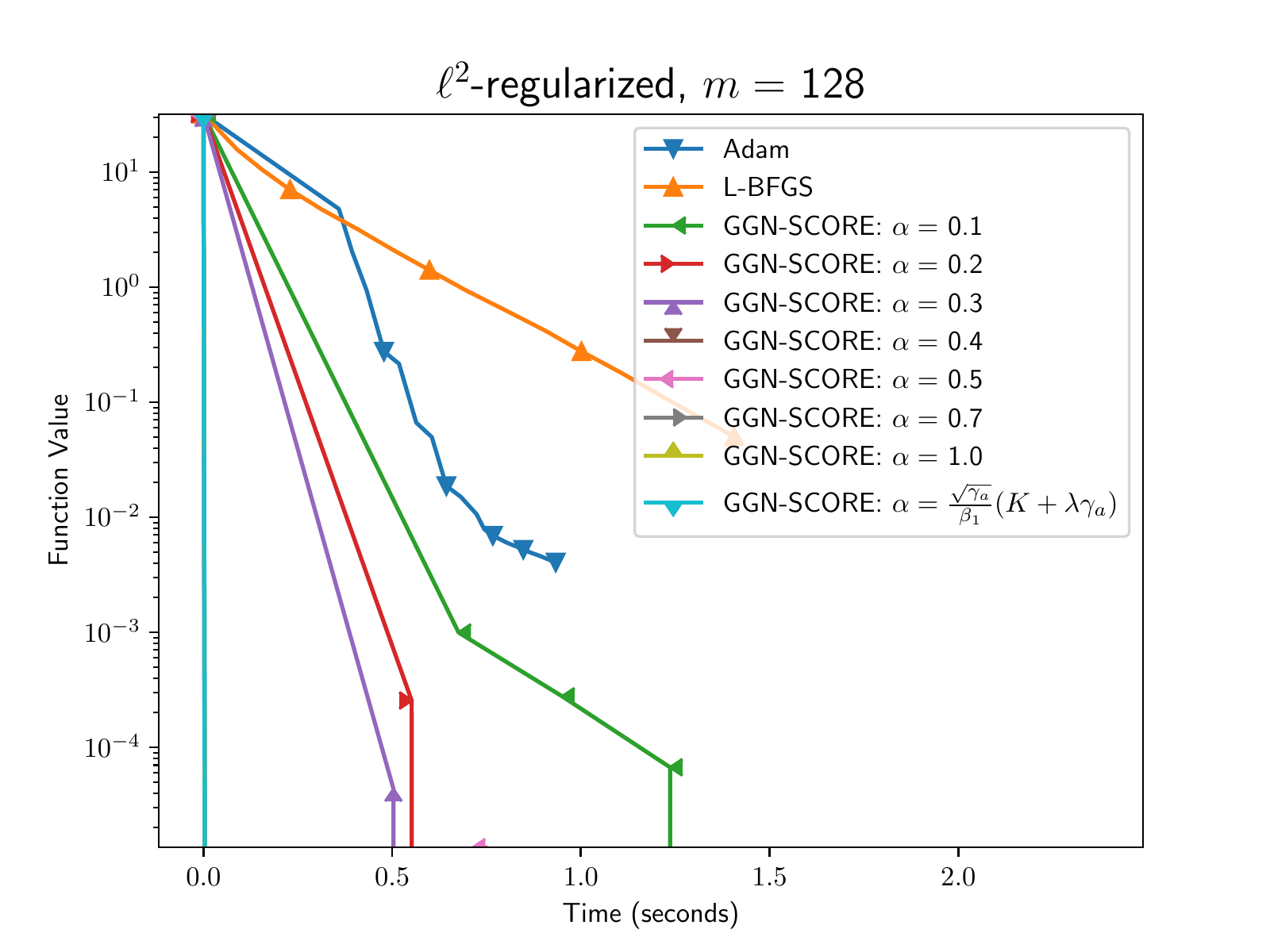}}
			\centerline{\includegraphics[width=1.2\linewidth]{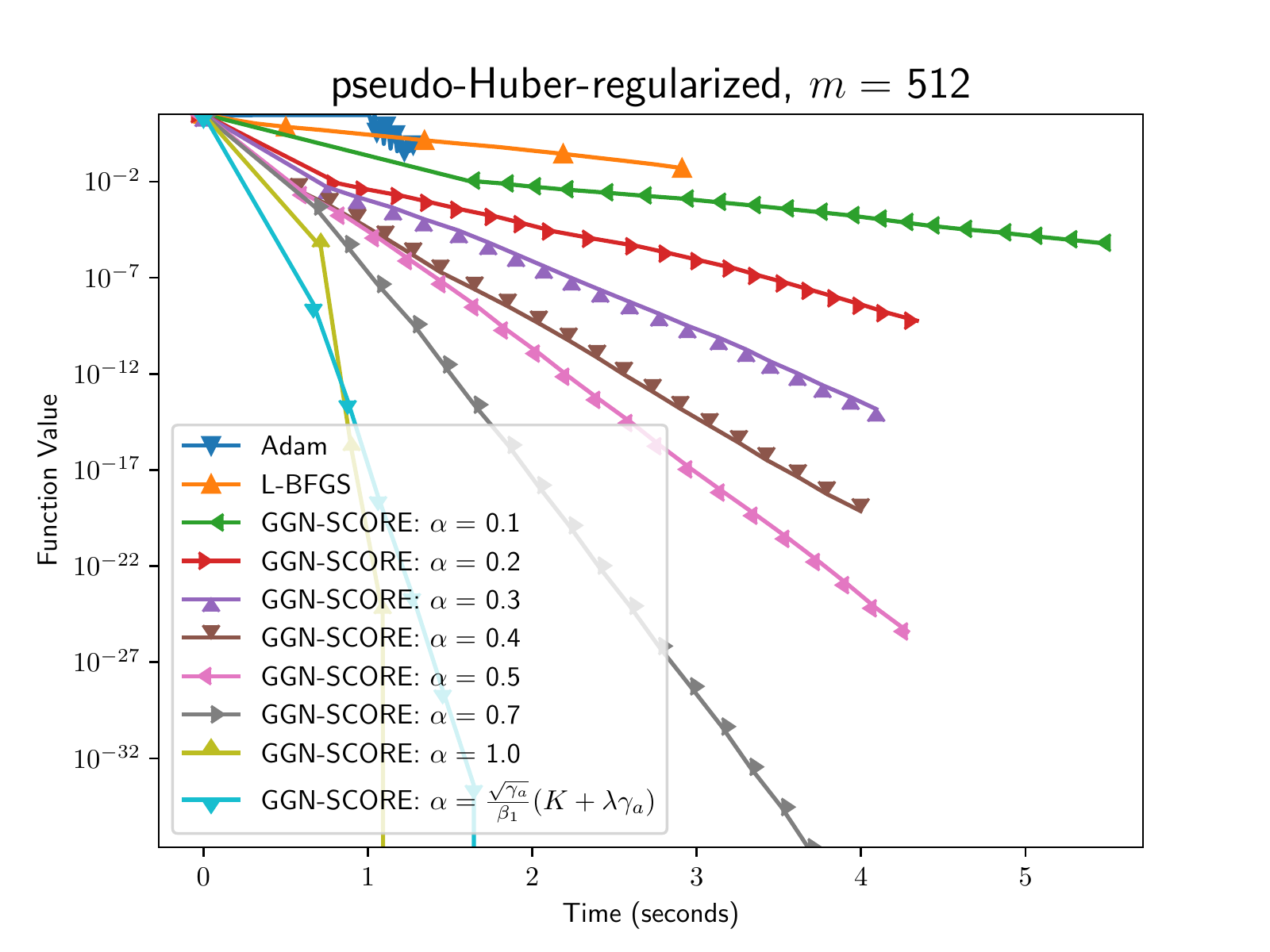}}
			\centerline{\includegraphics[width=1.2\linewidth]{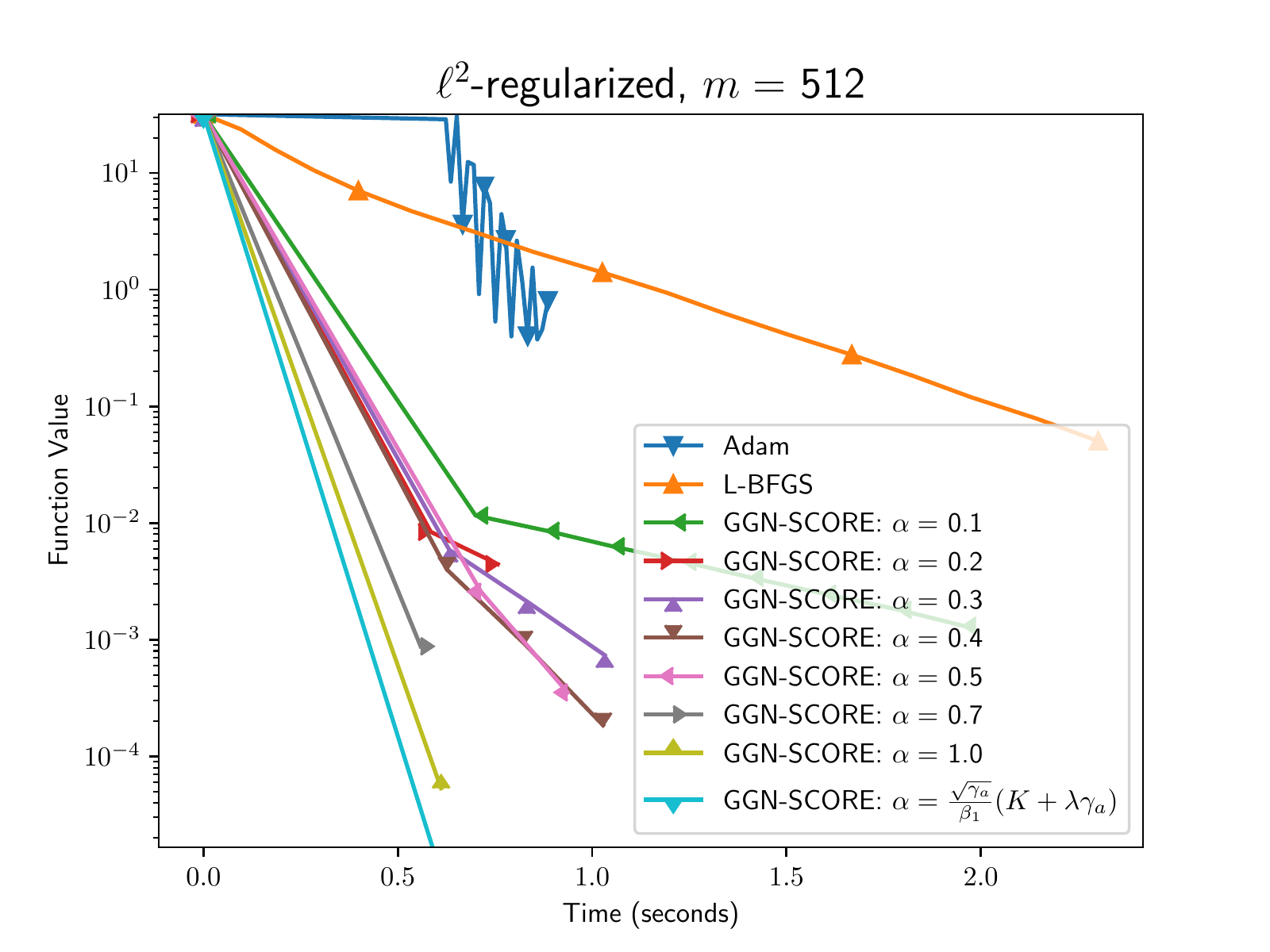}}
		\end{multicols}
	\end{subfigure}
	\caption{Numerical behaviour of GGN-SCORE for different values of $\alpha$ in the strongly convex quadratic test problem \eqref{eq:testproblem}.}
	\label{fig:step-scale-toy}
\end{figure*}
\begin{figure*}
	\vspace{1.4\baselineskip}
	\begin{subfigure}{1.0\textwidth}
		\vspace{-1.7\baselineskip}
		\begin{multicols}{4}
			\centerline{\includegraphics[width=1.2\linewidth]{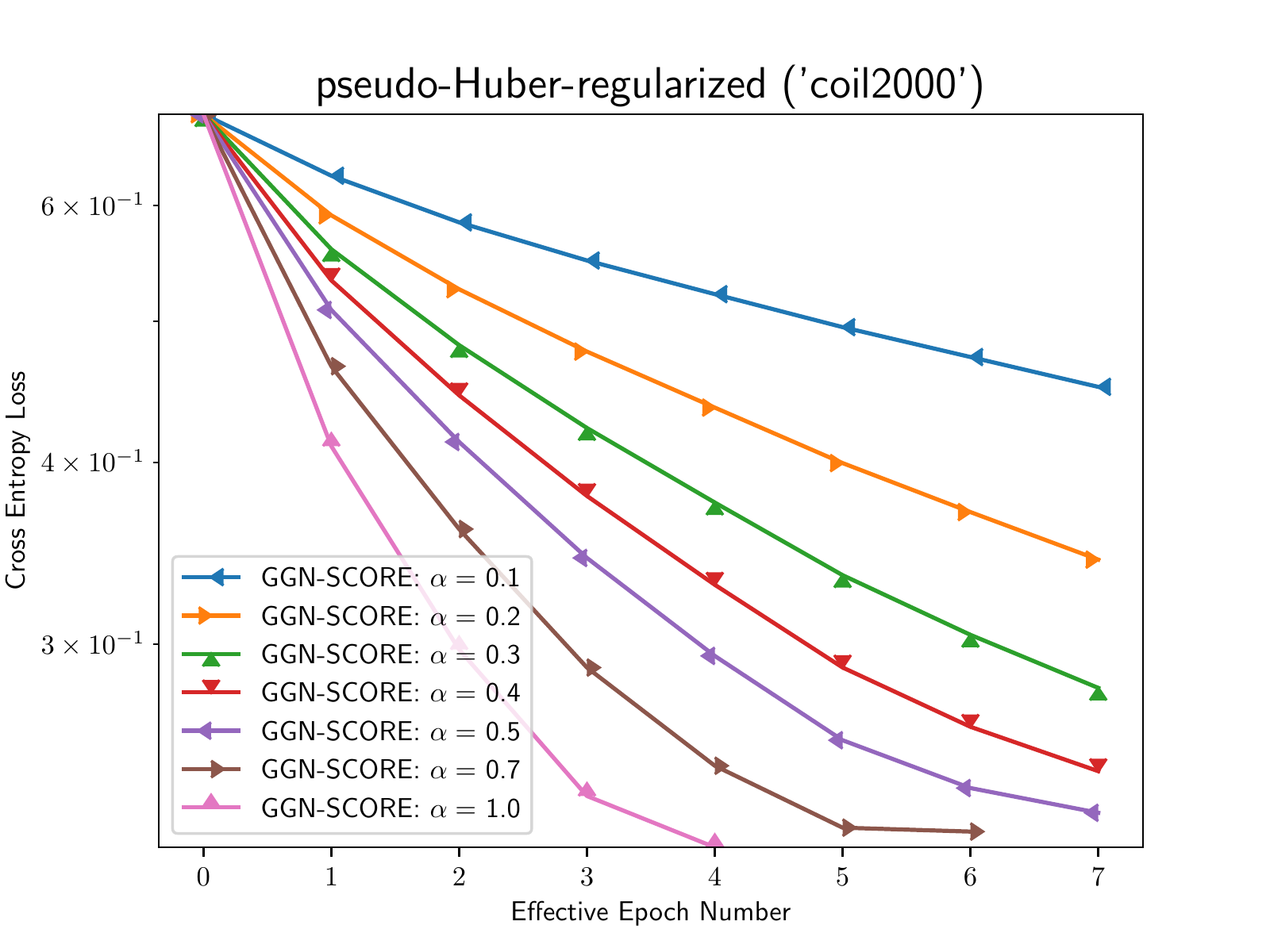}}
			\centerline{\includegraphics[width=1.2\linewidth]{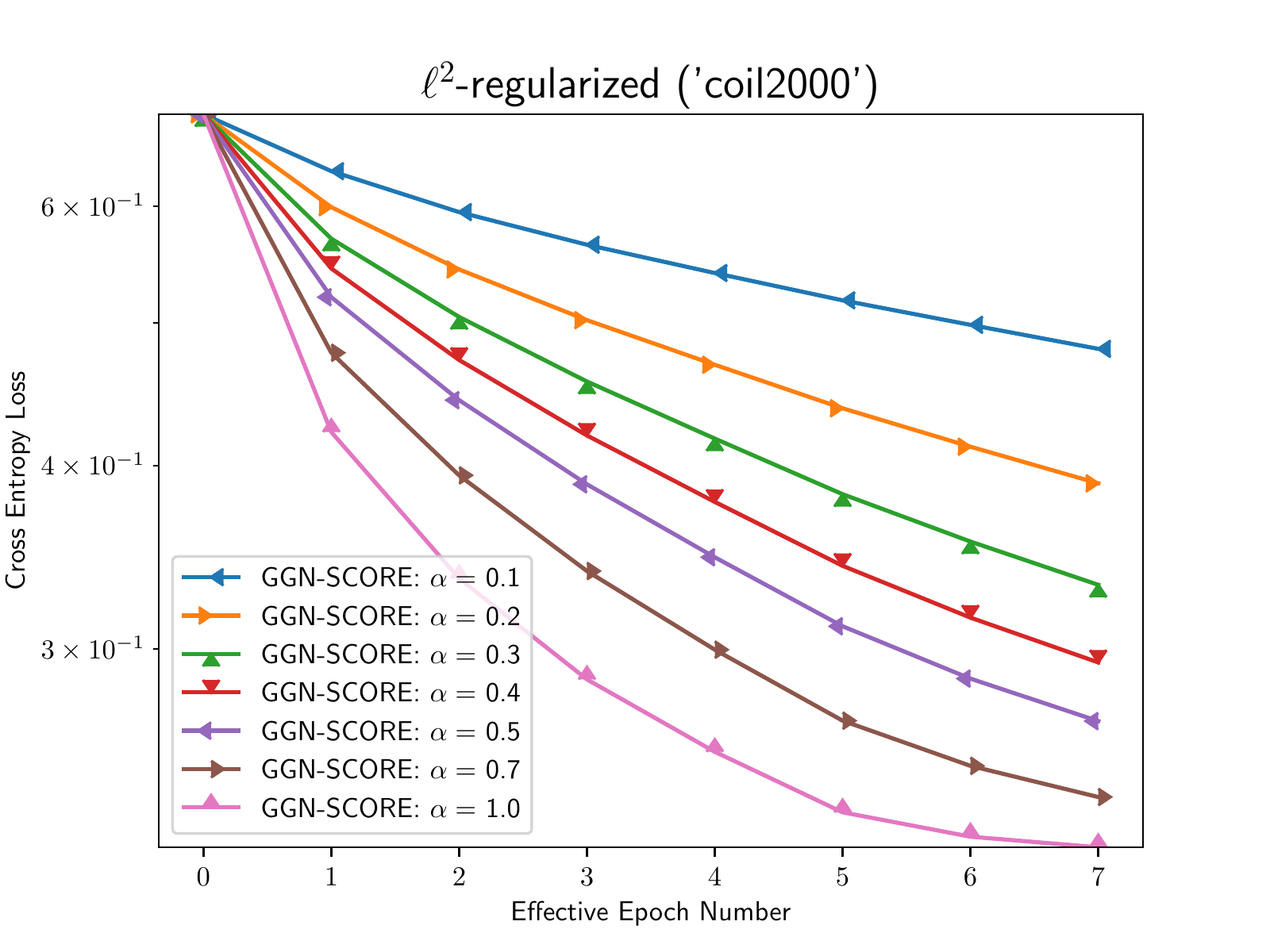}}
			\centerline{\includegraphics[width=1.2\linewidth]{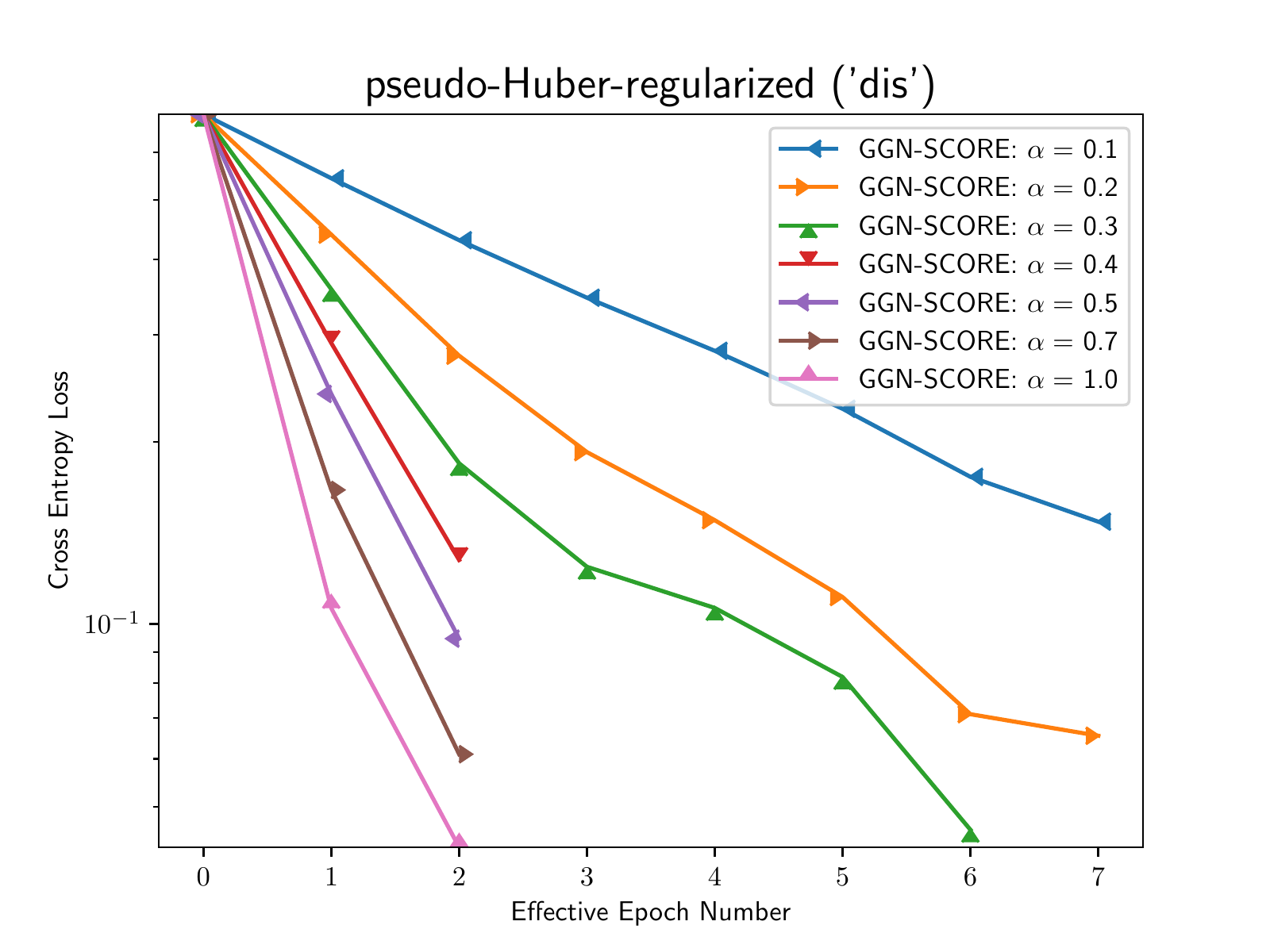}}
			\centerline{\includegraphics[width=1.2\linewidth]{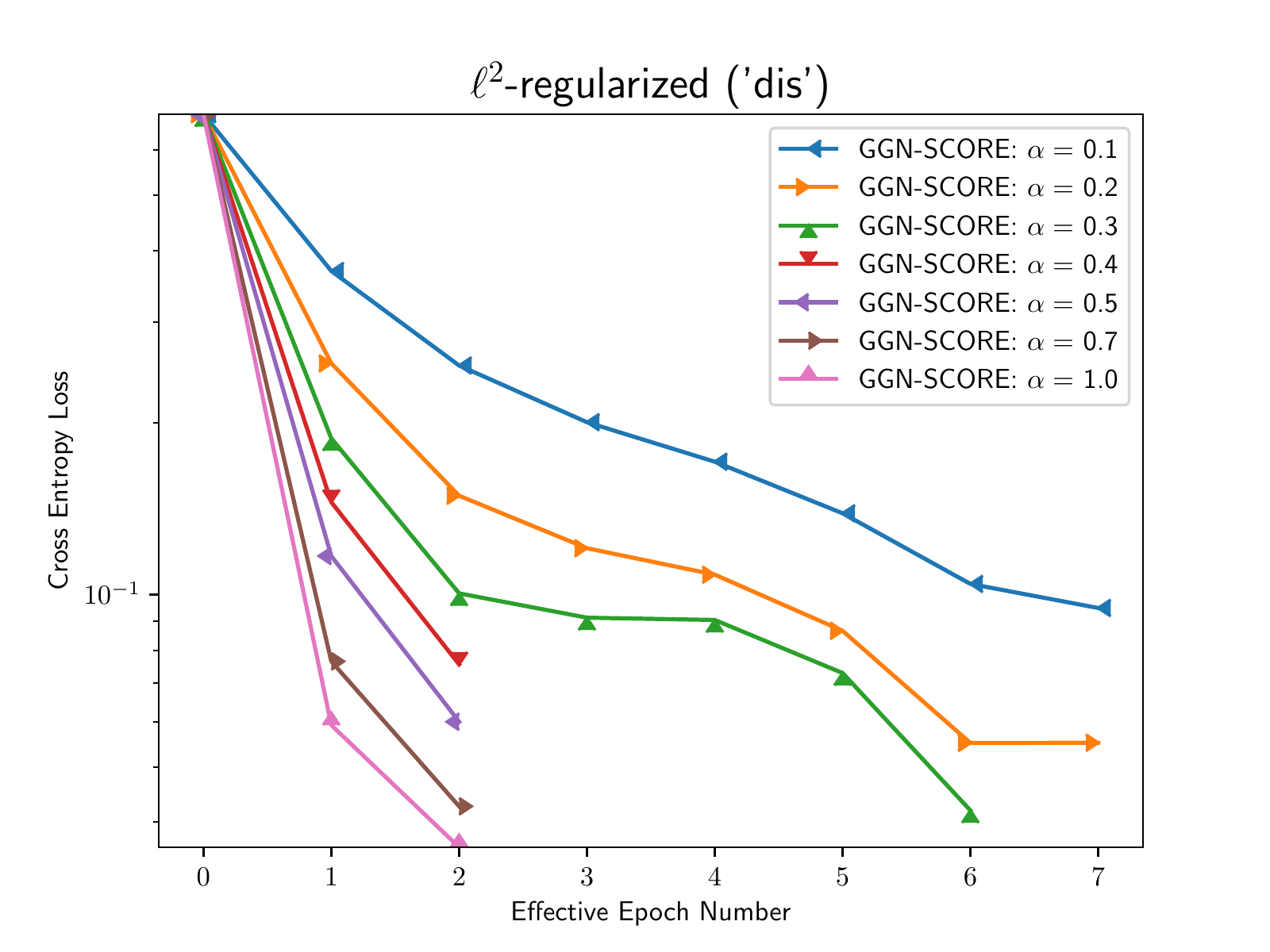}}
			\centerline{\includegraphics[width=1.2\linewidth]{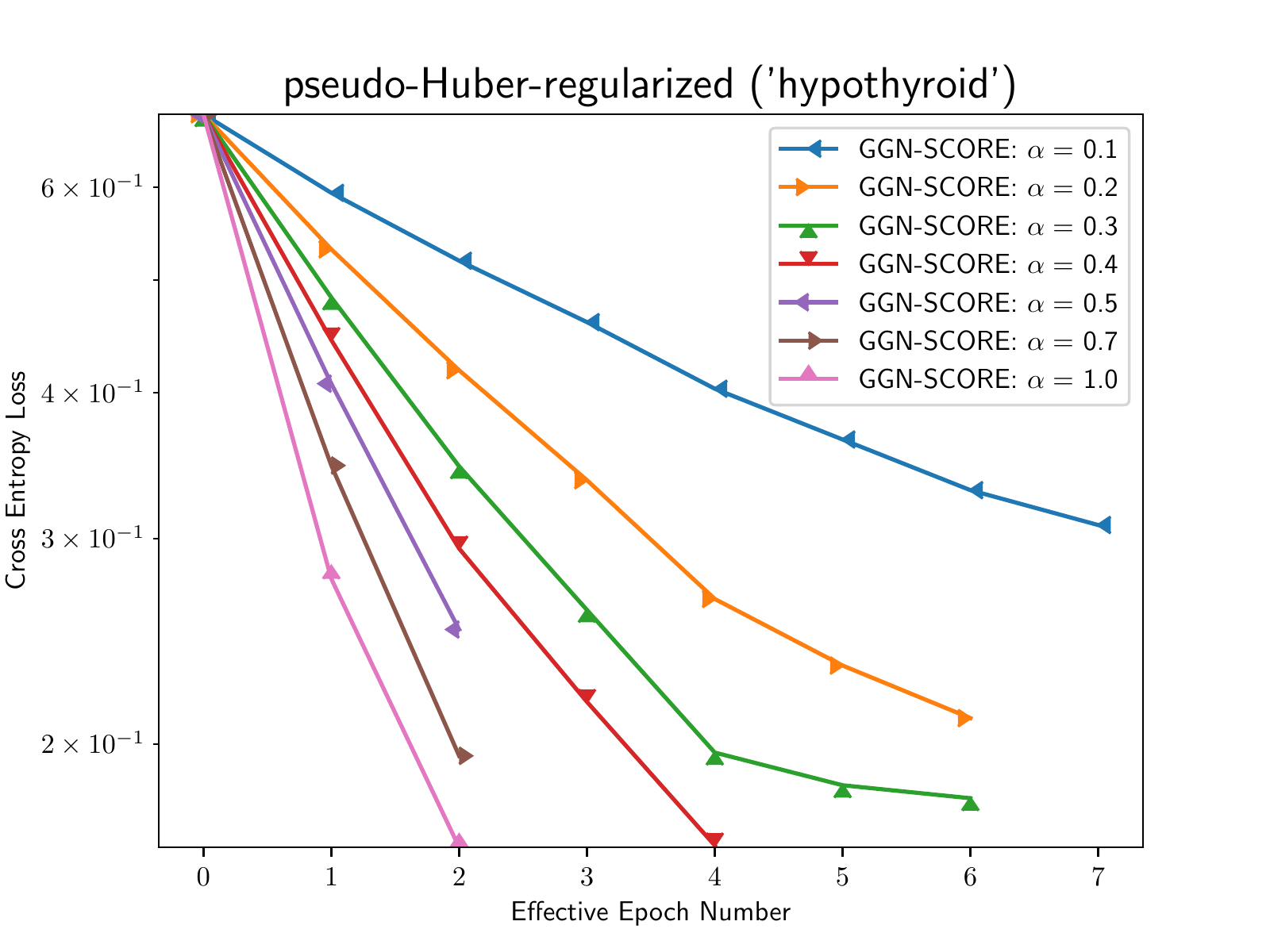}}
			\centerline{\includegraphics[width=1.2\linewidth]{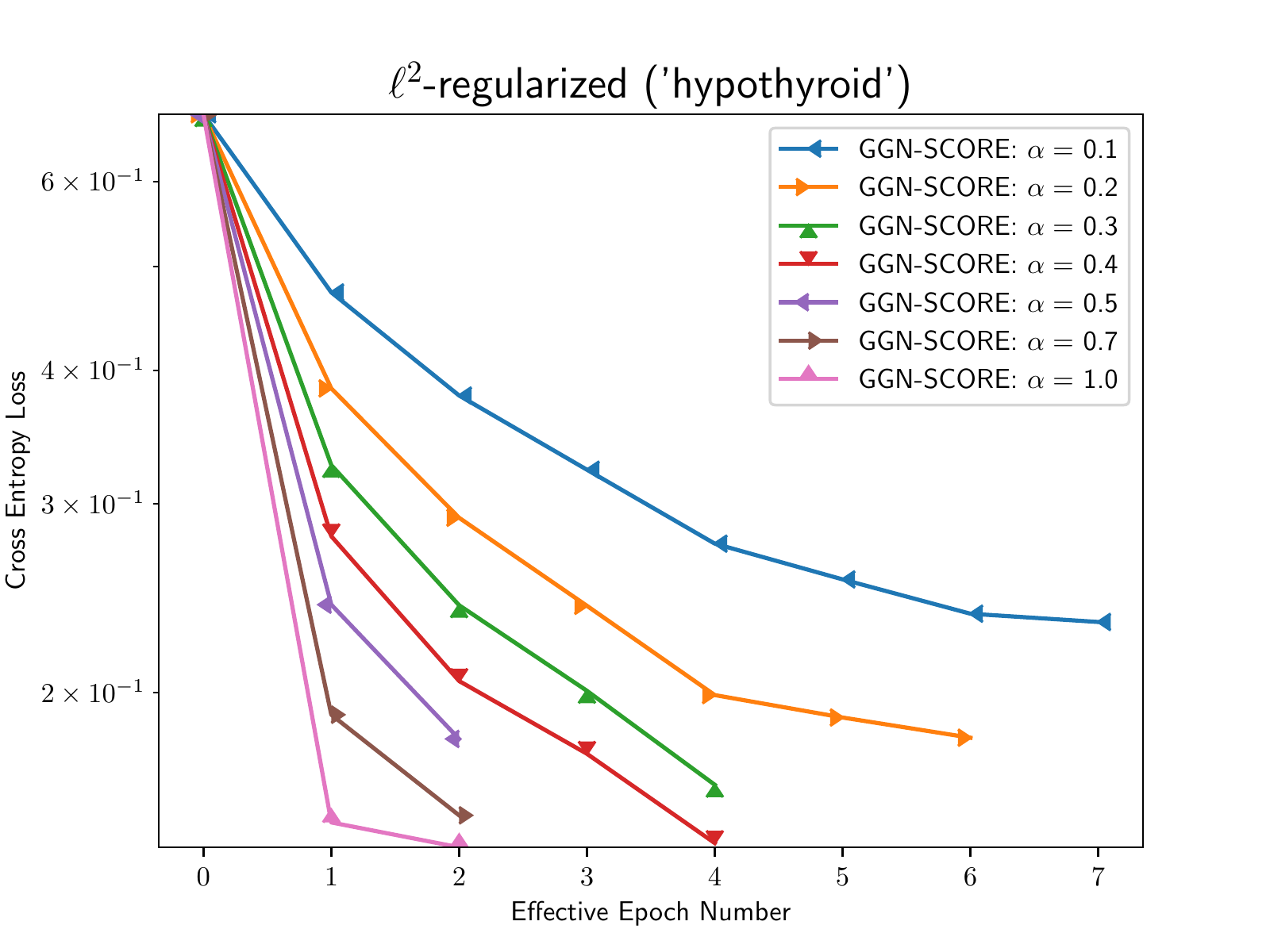}}
			\centerline{\includegraphics[width=1.2\linewidth]{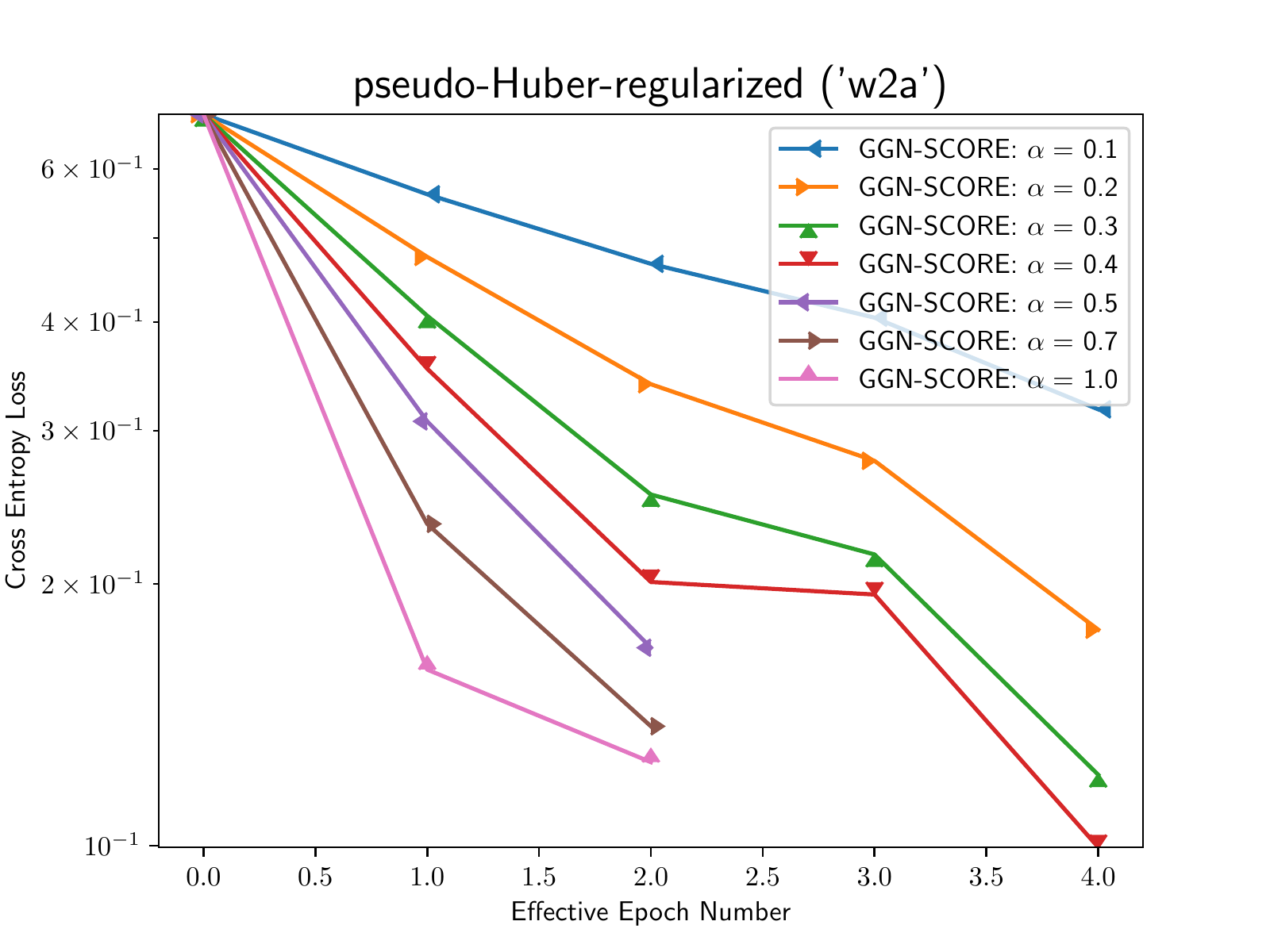}}
			\centerline{\includegraphics[width=1.2\linewidth]{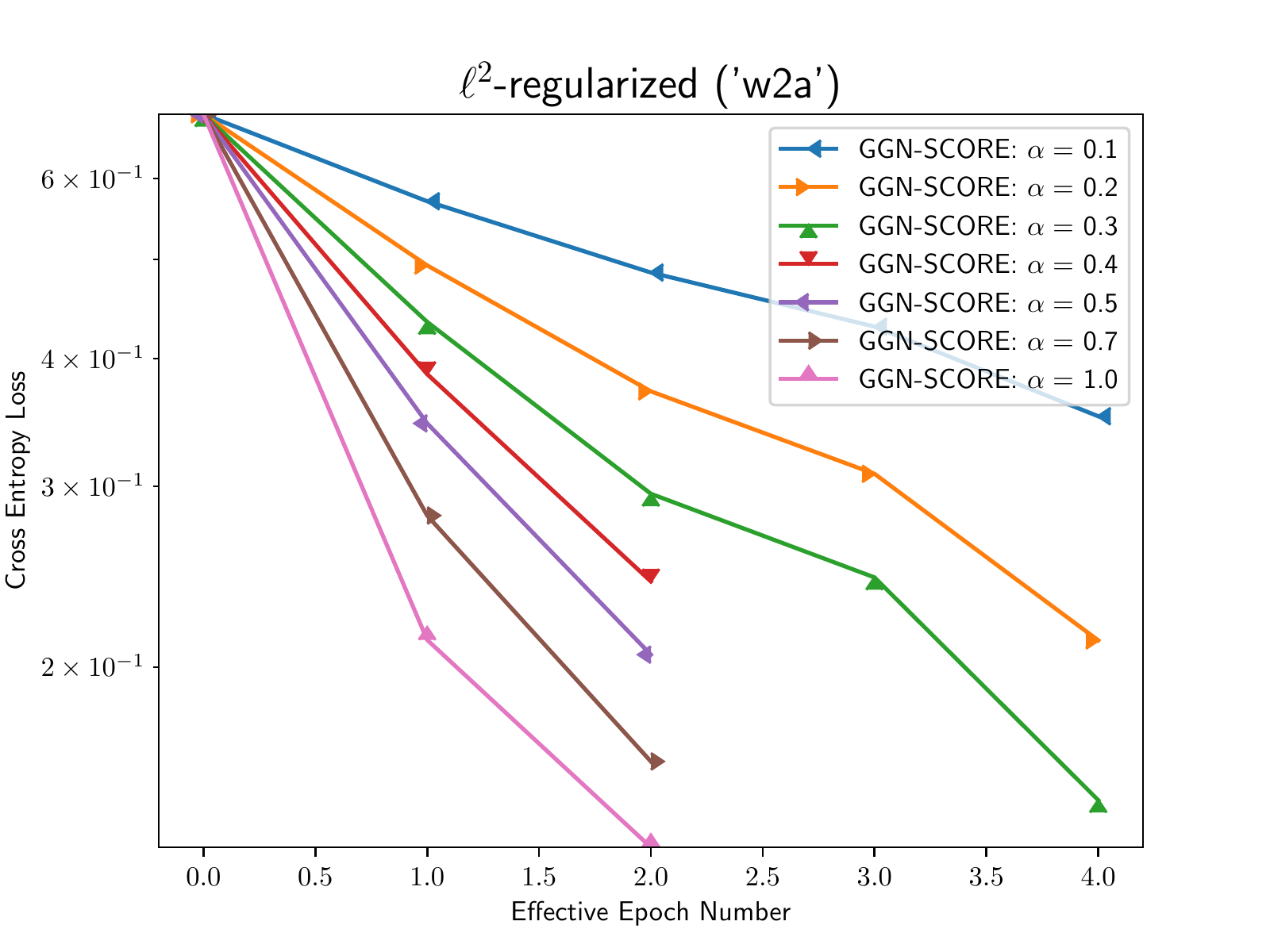}}
		\end{multicols}
	\end{subfigure}
	\caption{Numerical behaviour of GGN-SCORE for different values of $\alpha$ using real datasets. $m = 512$ for \texttt{w2a}, \texttt{dis} and \texttt{hypothyroid}, and $m=2048$ for \texttt{coil2000}.}
	\label{fig:step-scale}
\end{figure*}
We map the input space to higher dimensional feature space using the radial basis function (RBF) kernel $K(\bm{x}_n,\bm{x}_n')=\exp(-\gamma\|\bm{x}_n-\bm{x}_n'\|^2)$ with $\gamma=0.1$. In each experiment, we use the penalty value $\lambda=0.1$ with both pseudo-Huber regularization $h_\mu(\bm{\theta})$ \cite{ostrovskii2021finite} parameterized by $\mu>0$ \cite{charbonnier1997deterministic, hartley_zisserman_2004} and $\ell^2$ regularization $h_2(\bm{\theta})$ defined respectively as
\begin{align*}
	h_\mu(\bm{\theta}) \coloneqq \sqrt{\mu^2 +\bm{\theta}^2}-\mu, \qquad
	h_2(\bm{\theta}) = \|\bm{\theta}\|_2^2 \coloneqq \sum_{i=1}^{n_w}\left\vert\theta_i\right\vert^2,
\end{align*}
with coefficient $\mu=1.0$. Throughout, we choose a batch size, $m$ of $512$ for $\verb*|w2a|$, $\verb*|dis|$ and $\verb*|hypothyroid|$, $2048$ for $\verb*|coil2000|$ and $4096$ for $\verb*|ijcnn1|$, unless otherwise stated. We assume a scaled self-concordant regularization so that $M_h=1$ \cite[Corollary 5.1.3]{nesterov2018lectures}.
\subsection{GGN-SCORE for different values of $\alpha$}
To illustrate the behaviour of GGN-SCORE for different values of $\alpha$ in Algorithm \ref{alg:GGN-SCORE} versus its value indicated in \thmref{thm:main}, we consider the problem of minimizing a regularized strongly convex quadratic function:
\begin{align}
	\min_{\bm{\theta}}\mathcal{L}(\bm{\theta}) \coloneqq \frac{1}{2}\bm{\theta}^\top\hat{\bm{Q}}\bm{\theta} - \bm{p}^\top\bm{\theta} + \lambda h(\bm{\theta}) \equiv g(\bm{\theta}) + \lambda h(\bm{\theta}), \label{eq:testproblem}
\end{align}
where $\hat{\bm{Q}}\in \rr^{n_w\times n_w}$ is symmetric positive definite, $\bm{p}\in\rr^{n_w}$, $g$ is $\gamma_a$-strongly convex and has $\gamma_u$-Lipschitz gradient, with the smallest and largest eigenvalues of $\hat{\bm{Q}}$ corresponding to $\gamma_a$ and  $\gamma_u$, respectively. For this function, suppose the gradient and Hessian of $h(\bm{\theta})$ is known, for example when we choose $h=h_\mu$ or $h=h_2$, we have $\partial \mathcal{L}(\bm{\theta}) = \hat{\bm{Q}}\bm{\theta} - \bm{p} + \lambda\partial h(\bm{\theta})$ and $\partial^2\mathcal{L}(\bm{\theta}) = \hat{\bm{Q}} + \lambda\partial^2h(\bm{\theta})$. The coefficients $\hat{\bm{Q}}$ and $\bm{p}$ form our data, and with $\hat{\bm{Q}}\equiv 0.1\times \bm{M}^\top \bm{M}$, $N=n_w=1000$, we generate the data $\bm{M}\in\rr^{N\times N}$ randomly from a uniform distribution on $[0,1]$ and consider the case $\mathcal{L}^* = \bm{0}$ in which $\bm{p}$ is the zero vector. The optimization variable $\bm{\theta}$ is initialized to a random value generated from a normal distribution with mean $0$ and standard deviation $0.01$. \figref{fig:step-scale-toy} shows the behaviour of GGN-SCORE for this problem with different values of $\alpha$ in $(0,1]$ and $\alpha = \frac{\sqrt{\gamma_a}}{\beta_1}(K+\lambda\gamma_a)$ indicated in \thmref{thm:main}. We experiment with different batch sizes $m\in \{64, 128, 256, 512\}$. One observes from \figref{fig:step-scale-toy} that larger batch size yields better convergence speed when we choose $\alpha = \frac{\sqrt{\gamma_a}}{\beta_1}(K+\lambda\gamma_a)$, validating the recommendation in \remref{thm:remJ}. \figref{fig:step-scale-toy} also shows the comparison of GNN-SCORE with the first-order Adam \cite{kingma2014adam} algorithm, and the quasi-Newton Limited-memory Broyden-Fletcher-Goldfarb-Shanno (L-BFGS) \cite{nocedal1980updating} method using optimally tuned learning rates. While choosing $\alpha = \frac{\sqrt{\gamma_a}}{\beta_1}(K+\lambda\gamma_a)$ yields the kind of convergence shown in \thmref{thm:main}, \figref{fig:step-scale-toy} shows that by choosing $\alpha$ in $(0,1]$, we can similarly achieve a great convergence that scale well with the problem.

Strictly speaking, the value of $\alpha$ indicated in \thmref{thm:main} is not of practical interest, as it contains terms that may not be straightforward to retrieve in practice. In practice, we treat $\alpha$ as a hyperparameter that takes a \emph{fixed} positive value in $(0, 1]$. For an adaptive step-size selection rule, such as that in Line \ref{alg:step-size} of Algorithm \ref{alg:GGN-SCORE}, choosing a suitable scaling constant such as $\alpha$ is often straightforward, as the main step-size selection task is accomplished by the defined rule. We show the behaviour of GGN-SCORE on the real datasets for different values of $\alpha$ in $(0, 1]$ in \figref{fig:step-scale}. In general, a suitable scaling factor $\alpha$ should be selected based on the application demands.
\subsection{Comparison with SGD, Adam, and L-BFGS methods on real datasets}
Using the real datasets, we compare GGN-SCORE for solving \eqref{eq:emp} with results from the SGD, Adam, and the L-BFGS algorithms using optimally tuned learning rates. We also consider the training problem of a neural network with two hidden layers of dimensions $(2, 128)$, respectively for the $\verb*|coil2000|$ dataset, one hidden layer with dimension $1$ for the $\verb*|ijcnn1|$ dataset, and two hidden layers of dimensions $(4, 128)$, respectively for the remaining datasets. We use ReLU activation functions in the hidden layers of the networks, and the network is overparameterized for $\verb*|dis|$, $\verb*|hypothyroid|$, $\verb*|w2a|$ and $\verb*|coil2000|$ with $25425$, $21529$, $23497$ and $16229$ trainable parameters, respectively. We choose $\alpha\in\{0.2, 0.5\}$ for GGN-SCORE. Minimization variables are initialized to the zero vector for all the methods. The neural network training problems are solved under the same settings. The results are respectively displayed in \figref{fig:lossplots} and \figref{fig:nonconvex} for the convex and non-convex cases.
\begin{figure*}
	\vspace{1.4\baselineskip}
	\begin{subfigure}{1.0\textwidth}
		\vspace{-1.7\baselineskip}
		\begin{multicols}{4}
			\centerline{\includegraphics[width=1.2\linewidth]{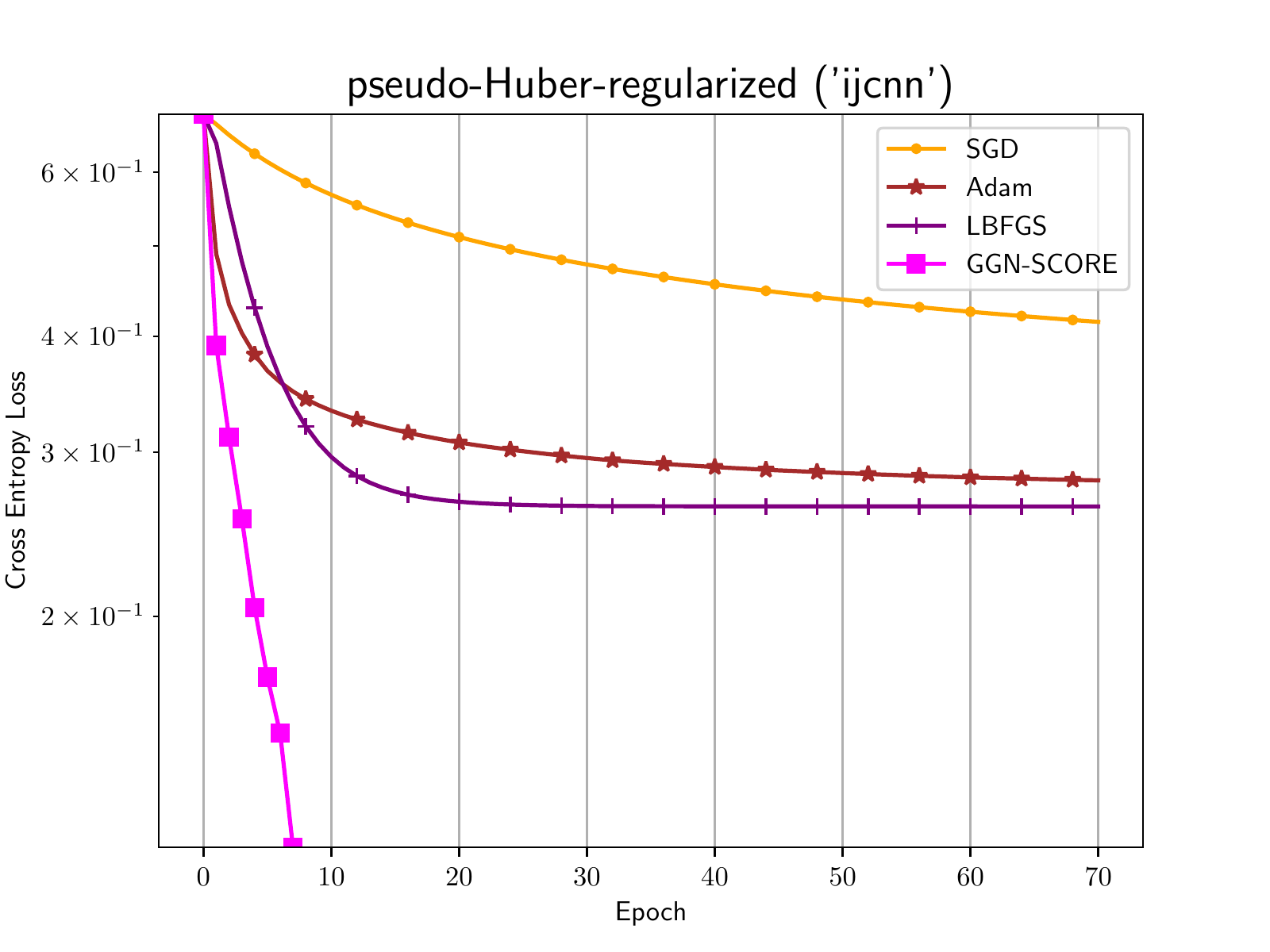}}
			\centerline{\includegraphics[width=1.2\linewidth]{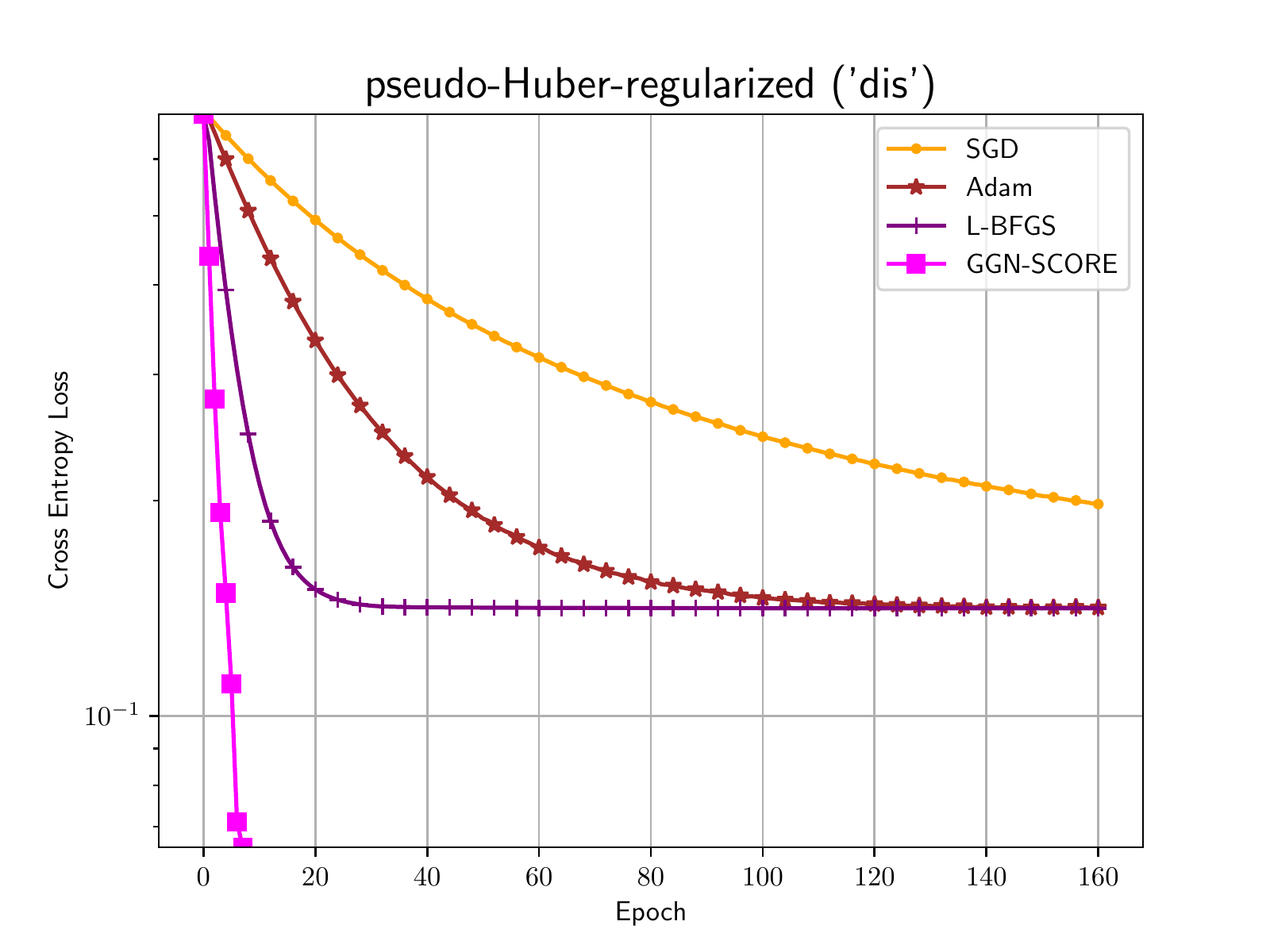}}
			\centerline{\includegraphics[width=1.2\linewidth]{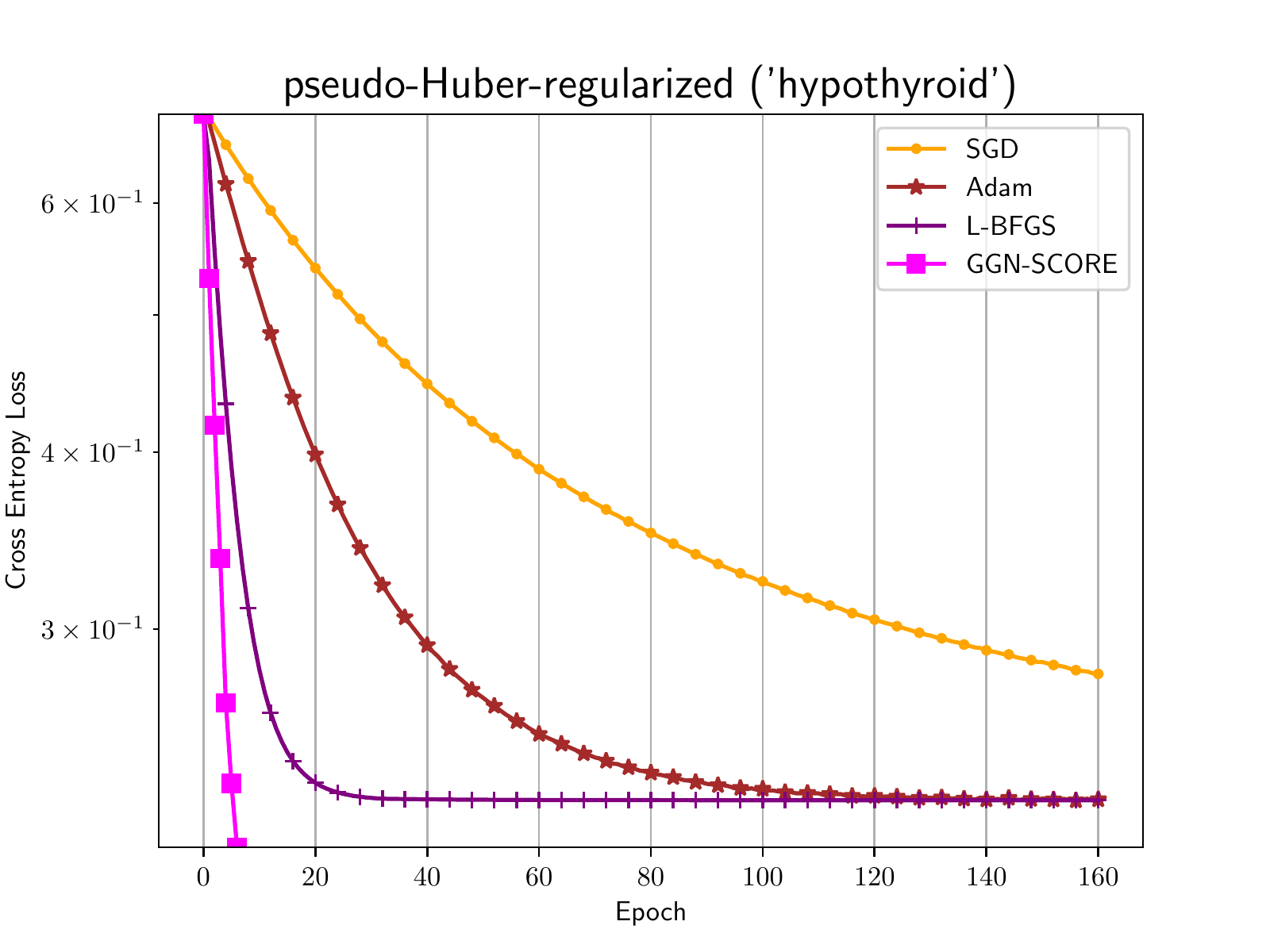}}
			\centerline{\includegraphics[width=1.2\linewidth]{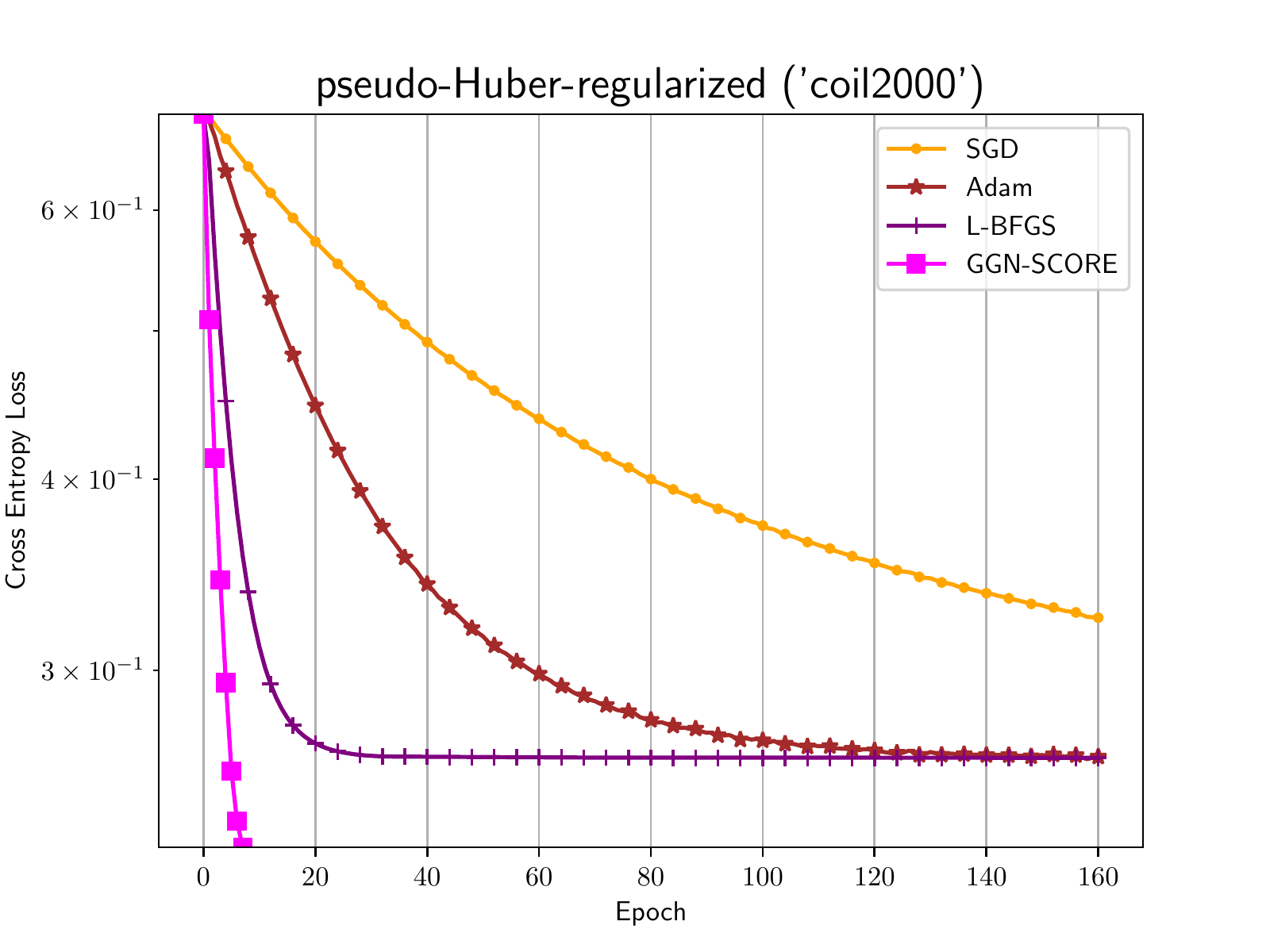}}
			\centerline{\includegraphics[width=1.2\linewidth]{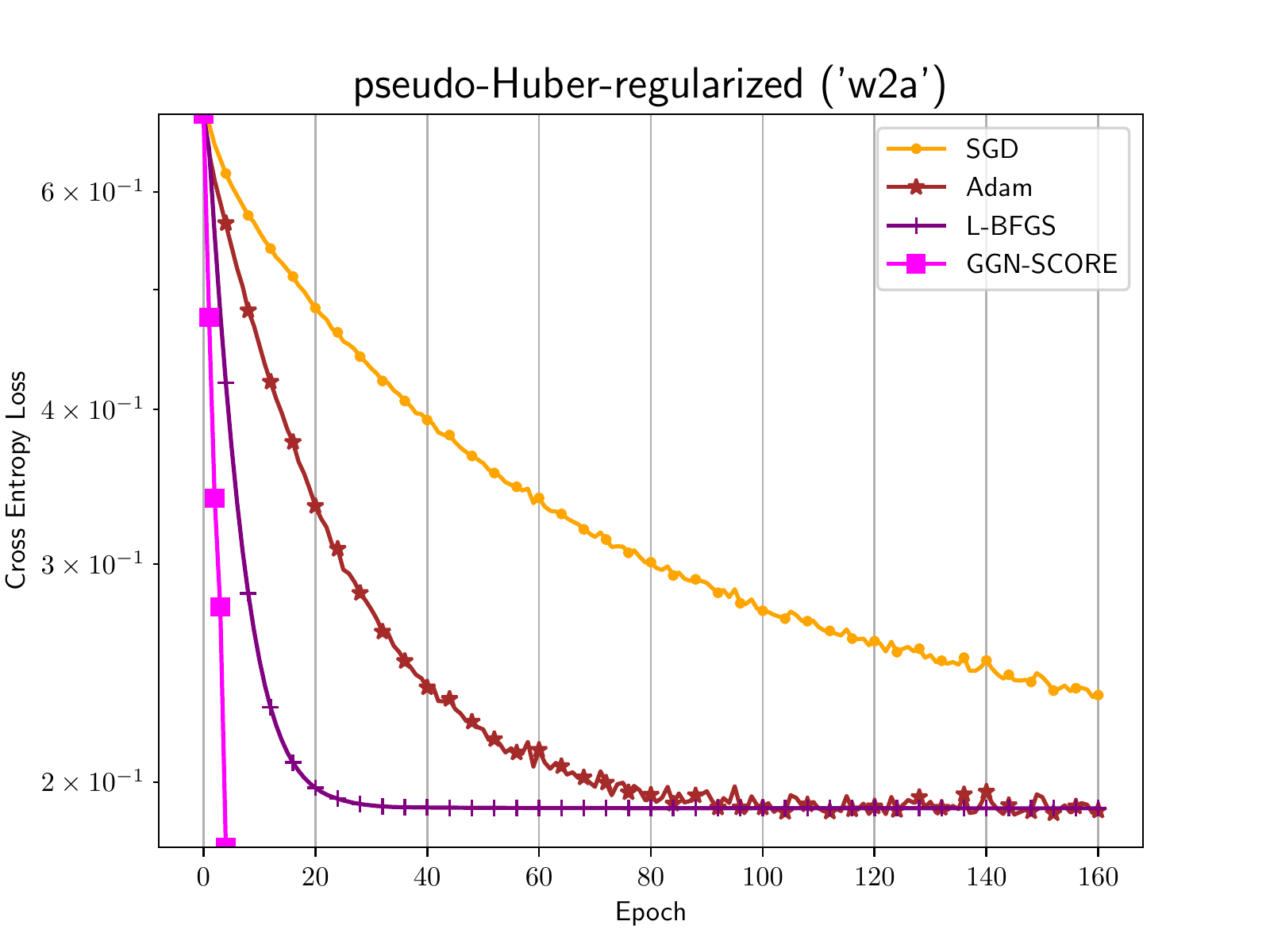}}

			\centerline{\includegraphics[width=1.2\linewidth]{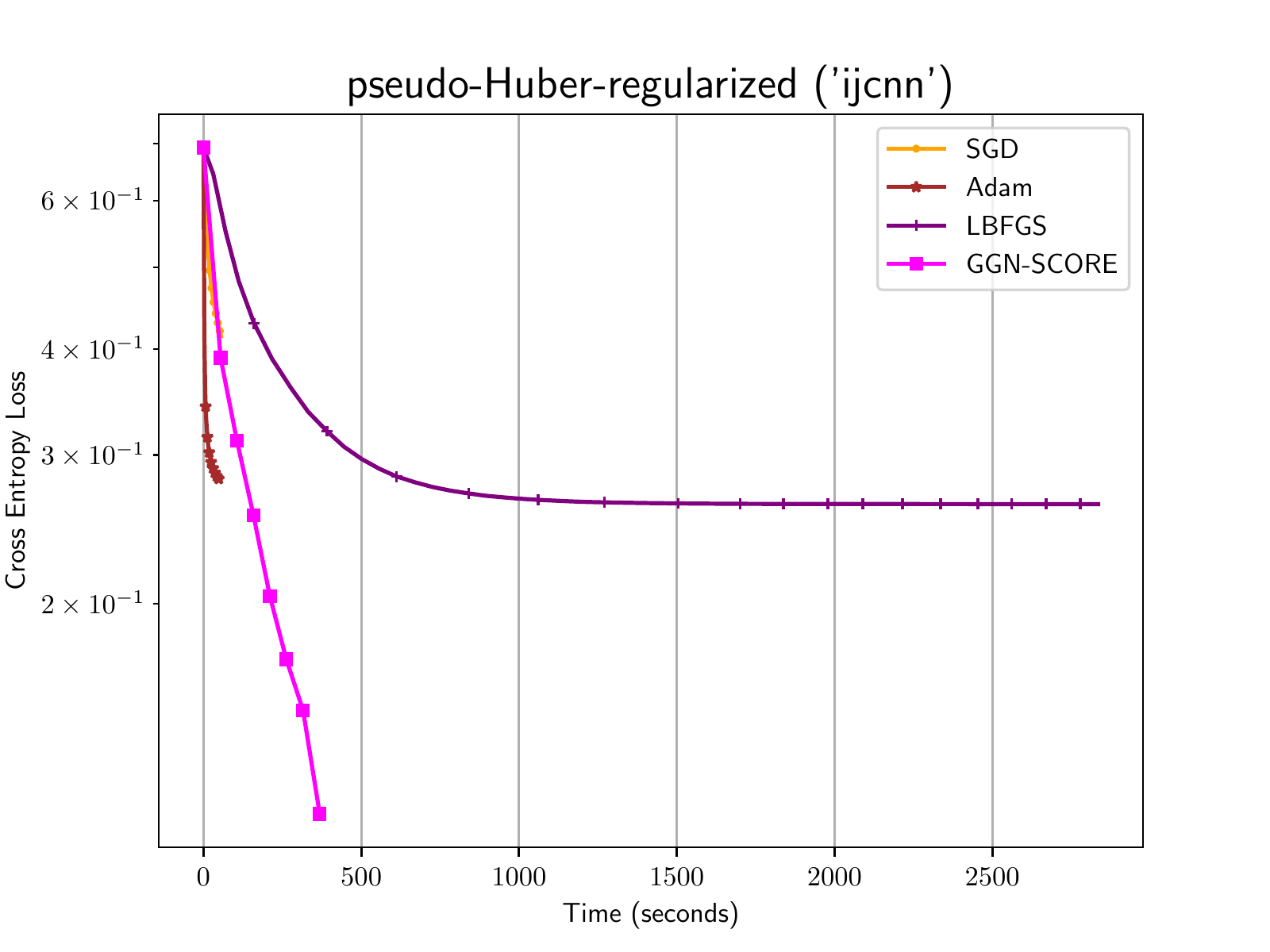}}
			\centerline{\includegraphics[width=1.2\linewidth]{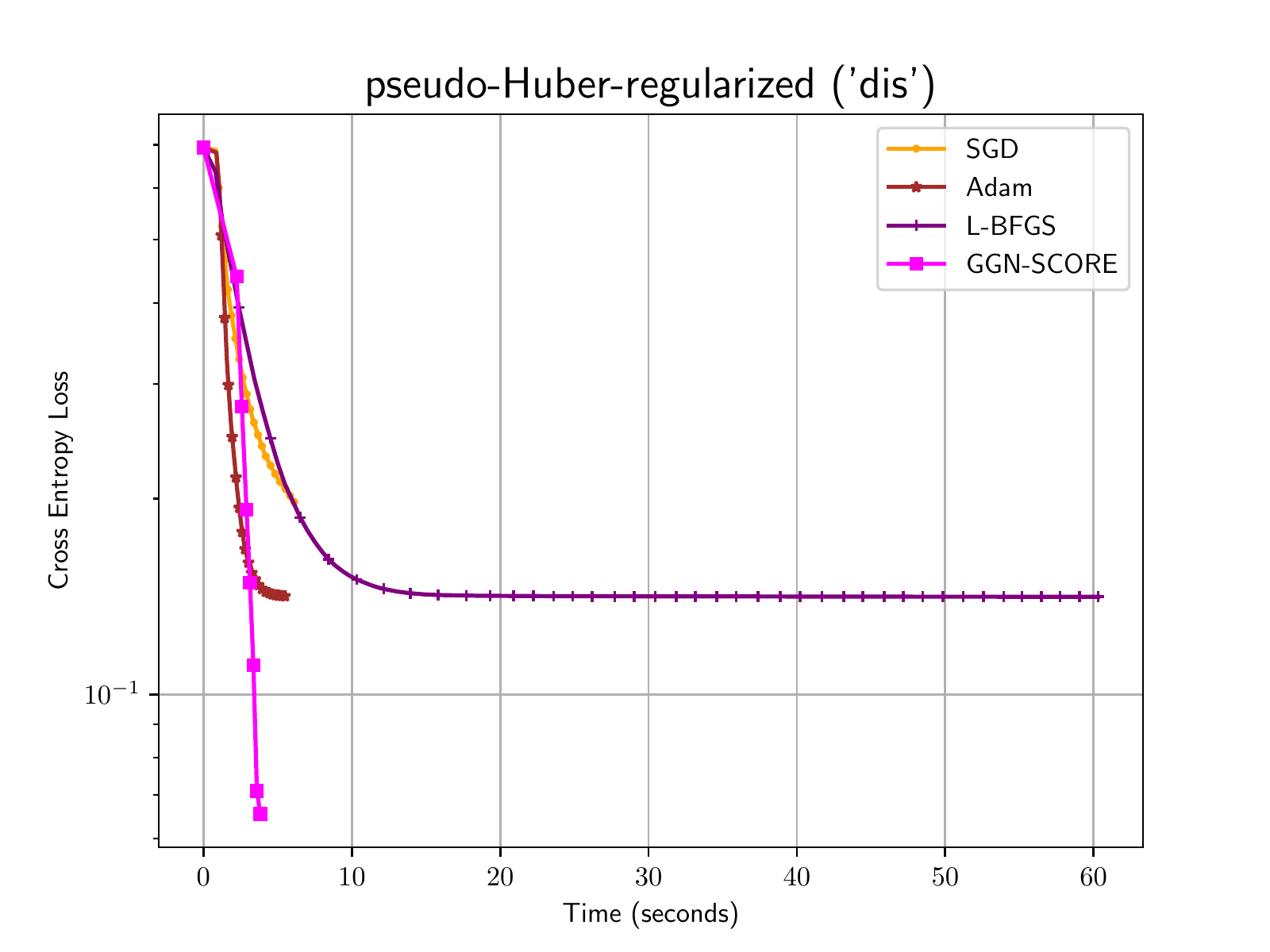}}
			\centerline{\includegraphics[width=1.2\linewidth]{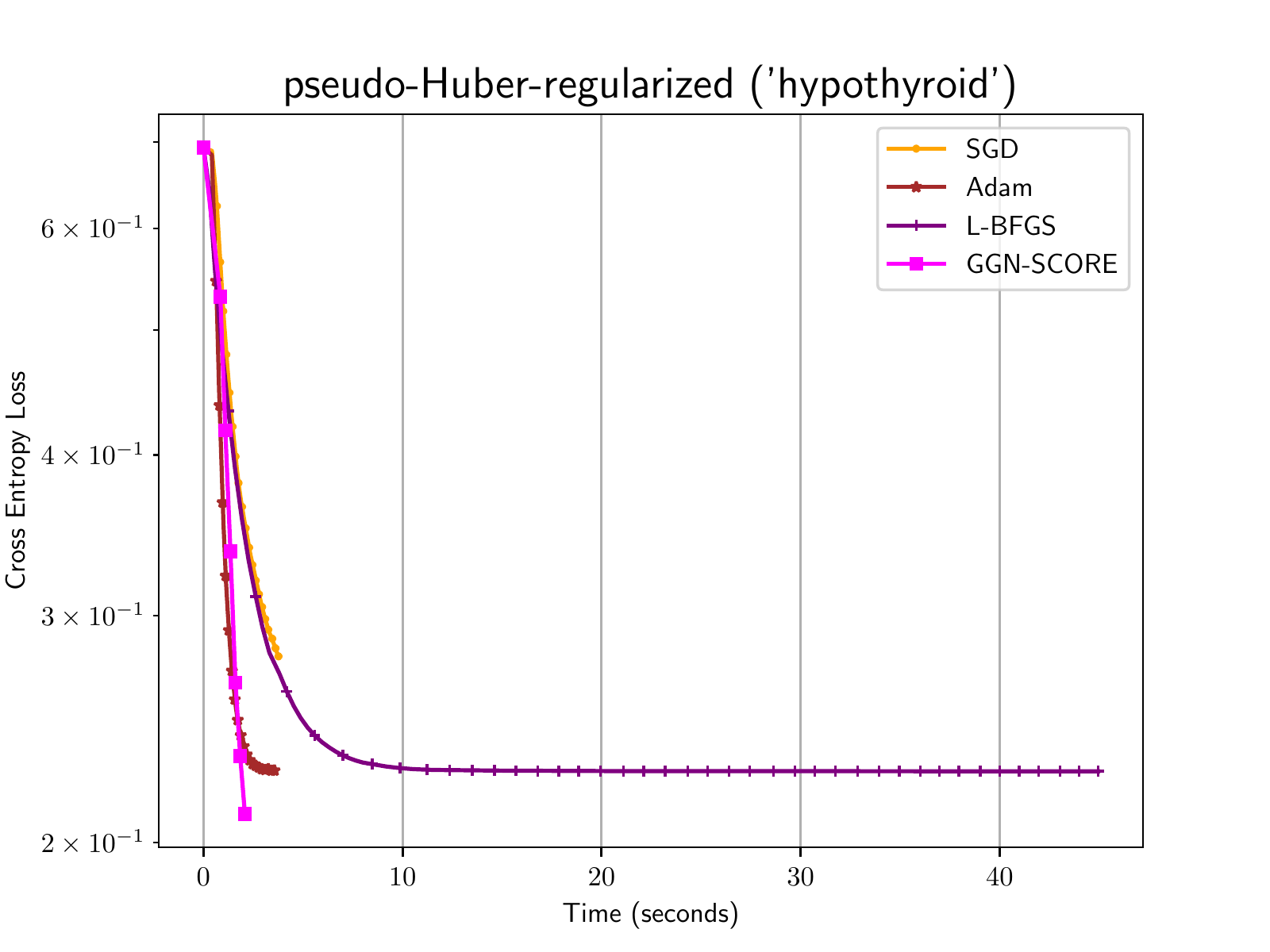}}
			\centerline{\includegraphics[width=1.2\linewidth]{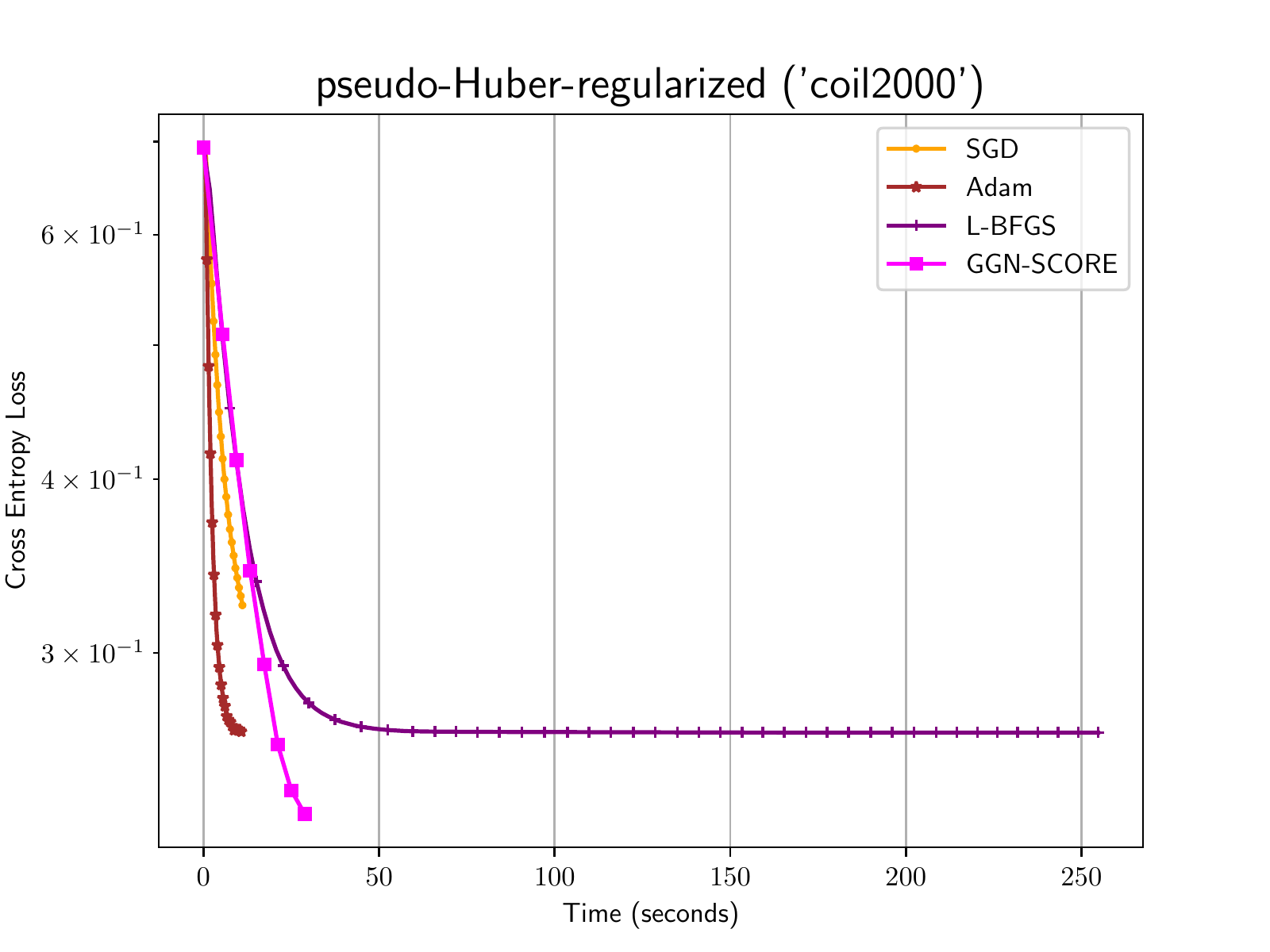}}
			\centerline{\includegraphics[width=1.2\linewidth]{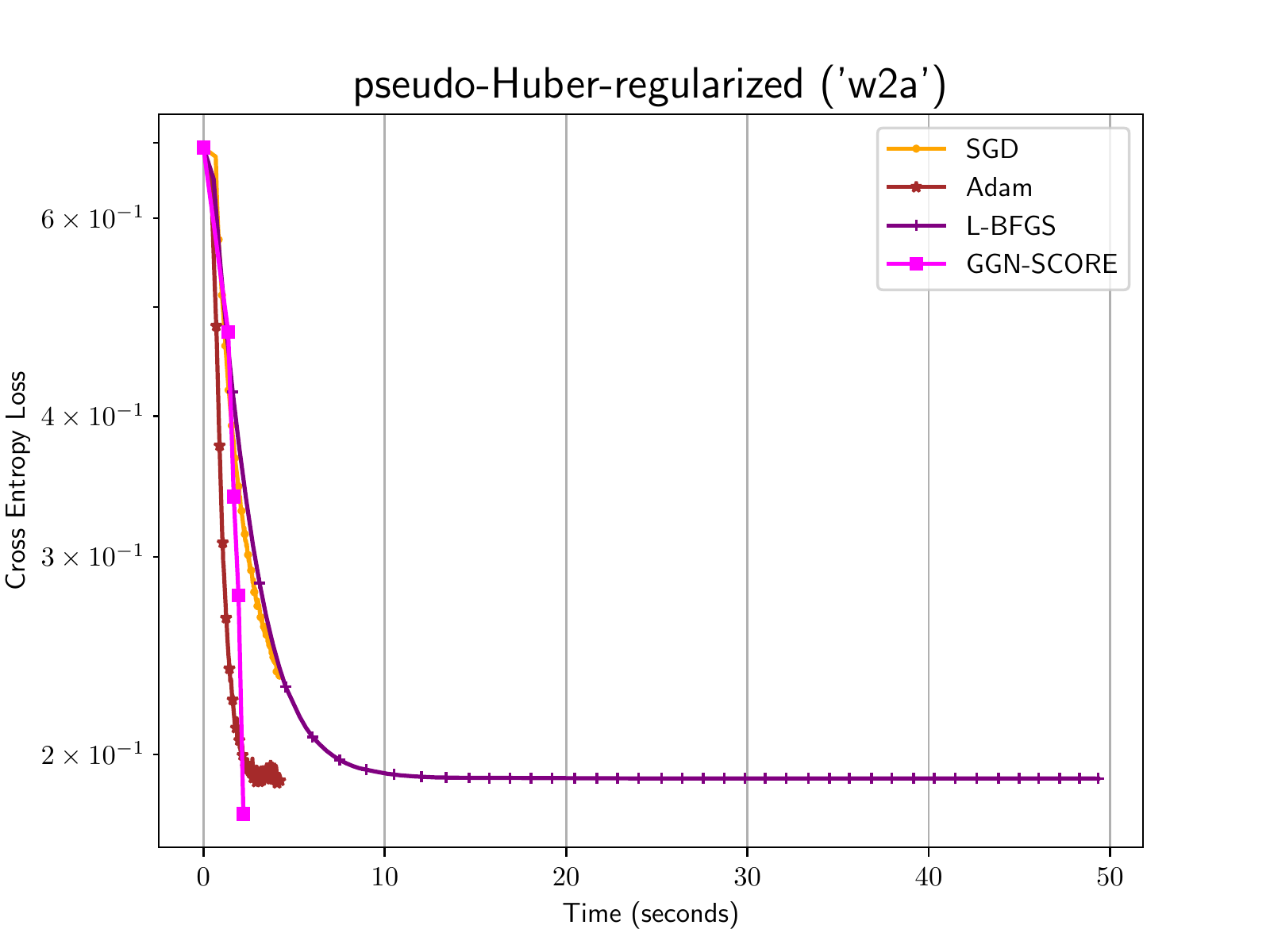}}

			\centerline{\includegraphics[width=1.2\linewidth]{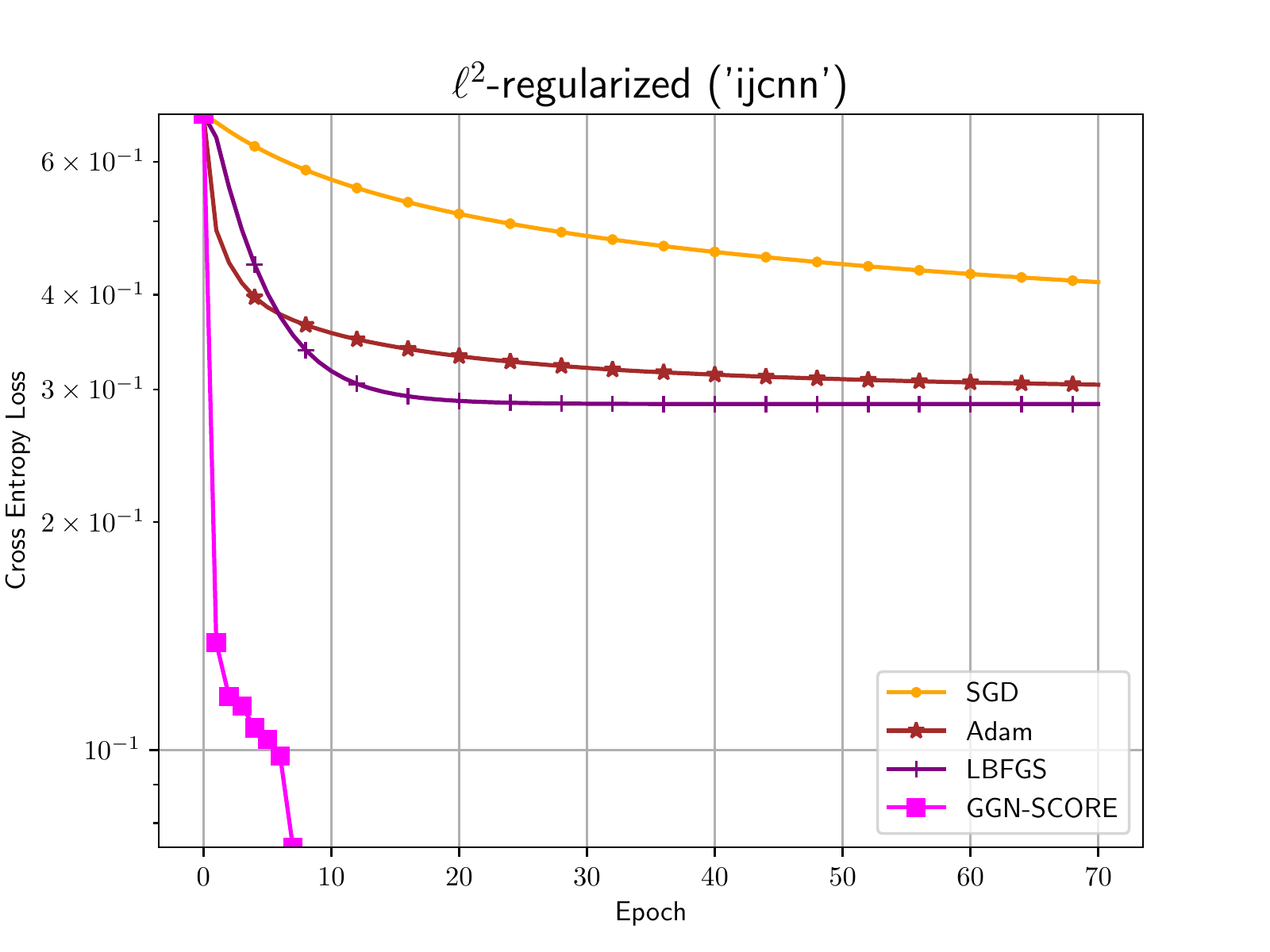}}
			\centerline{\includegraphics[width=1.2\linewidth]{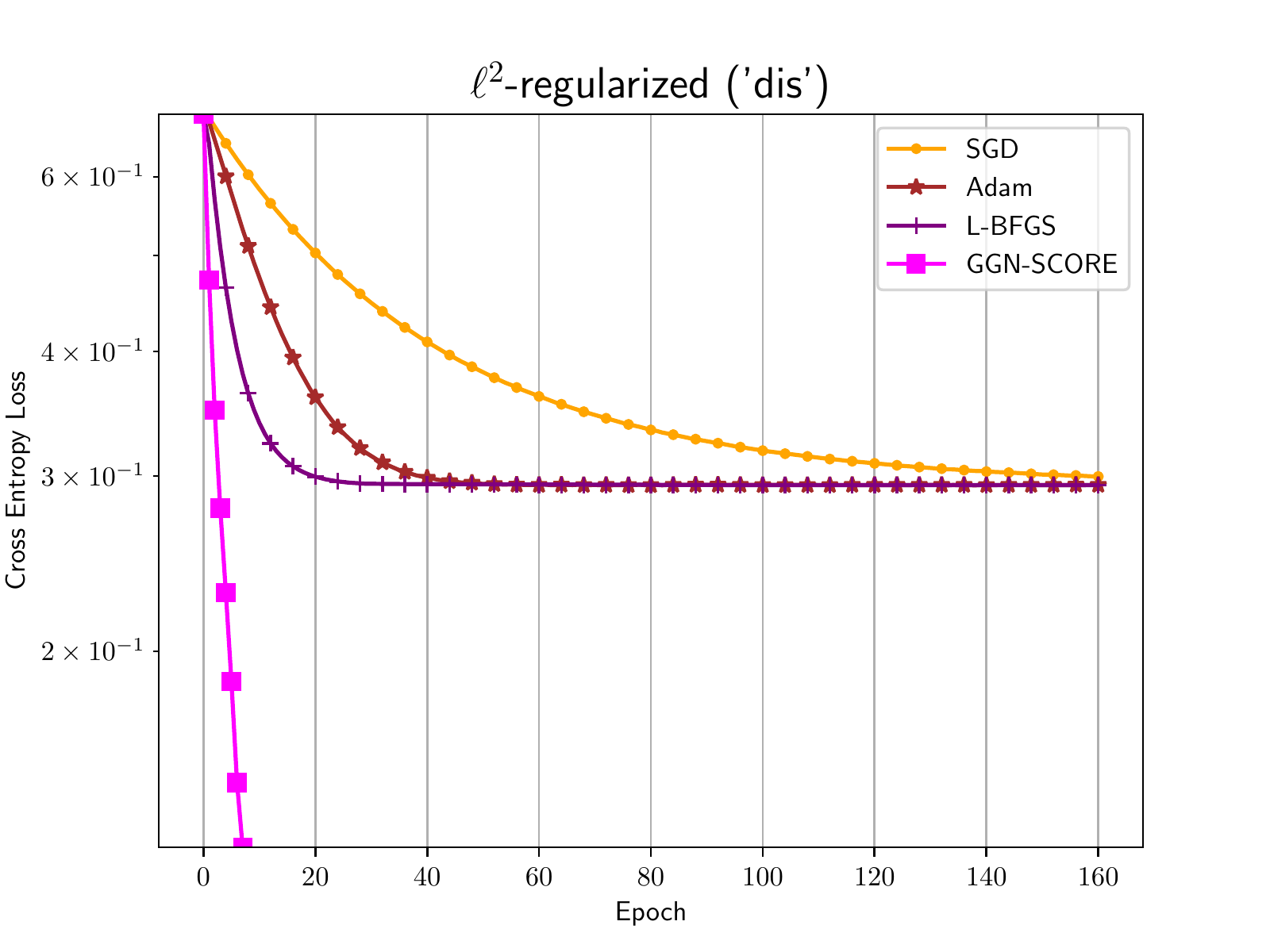}}
			\centerline{\includegraphics[width=1.2\linewidth]{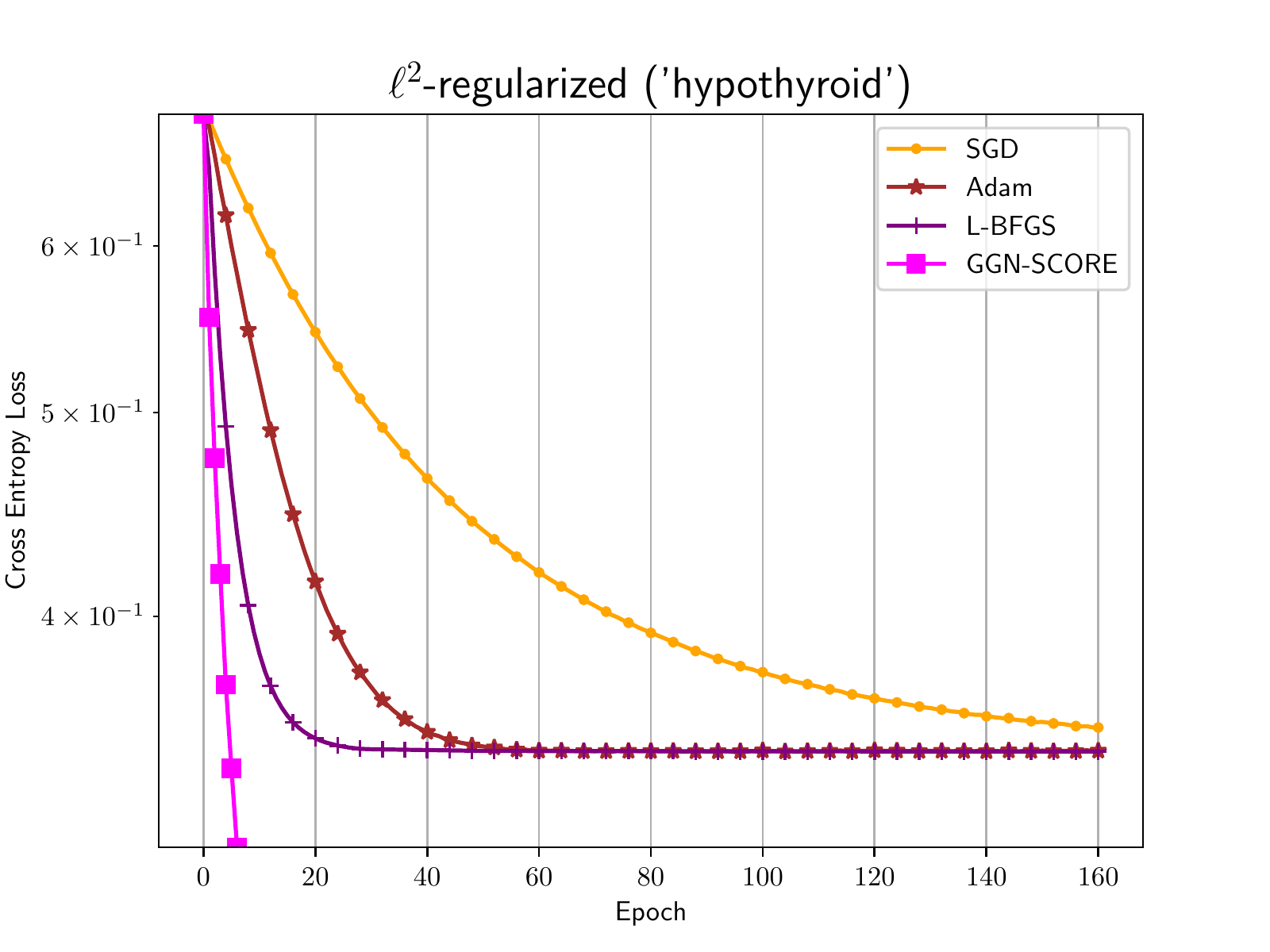}}
			\centerline{\includegraphics[width=1.2\linewidth]{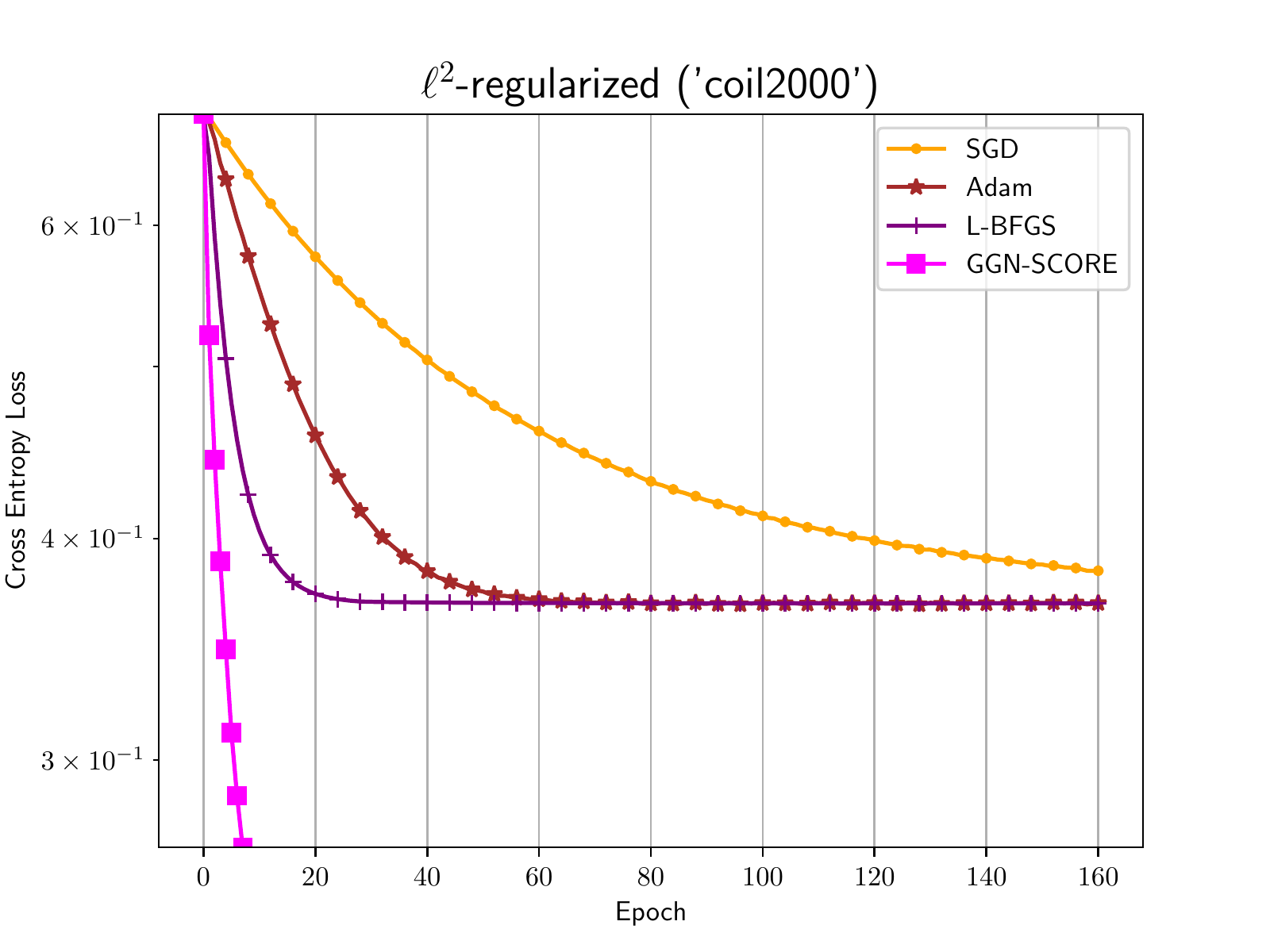}}
			\centerline{\includegraphics[width=1.2\linewidth]{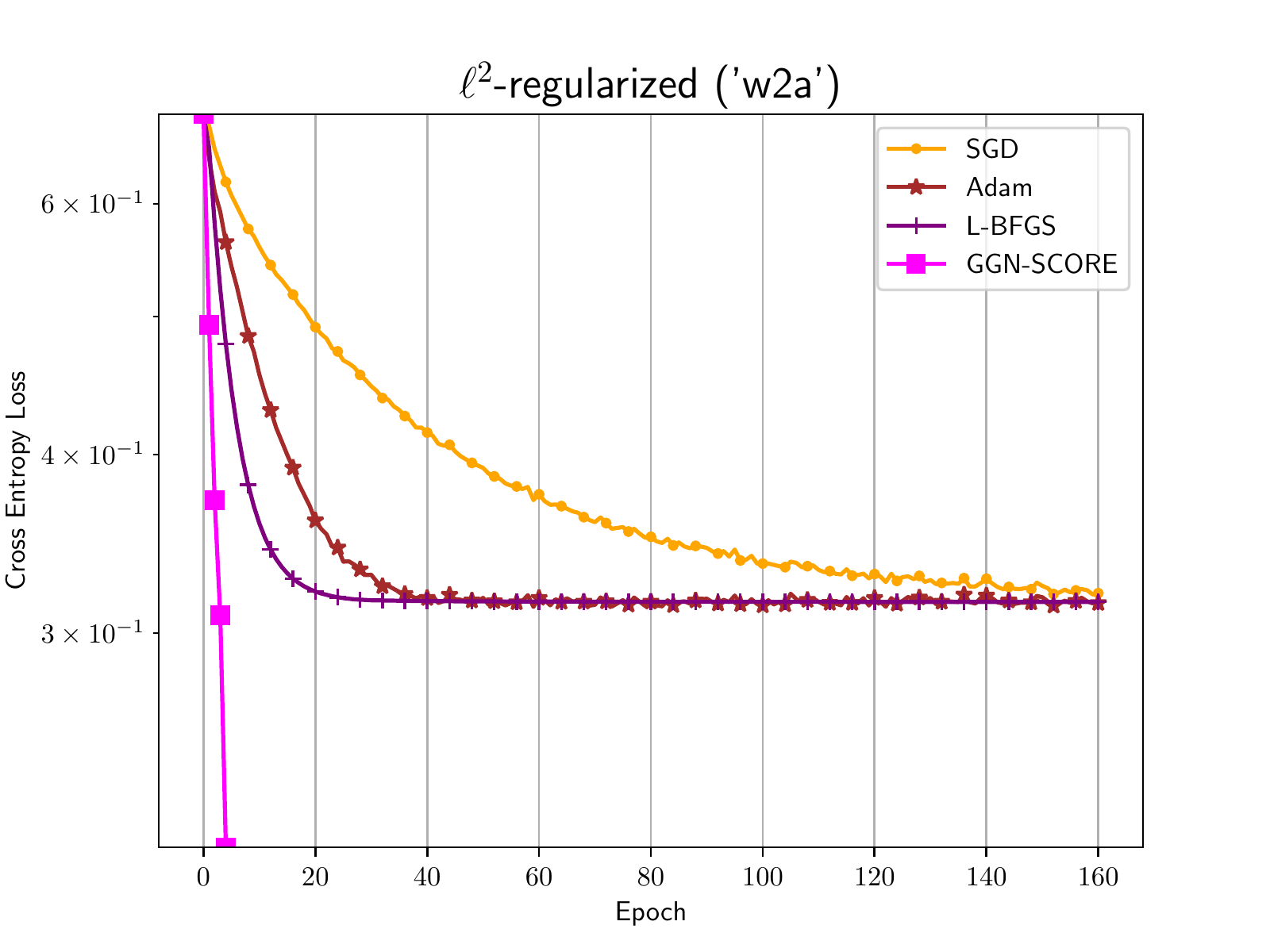}}

			\centerline{\includegraphics[width=1.2\linewidth]{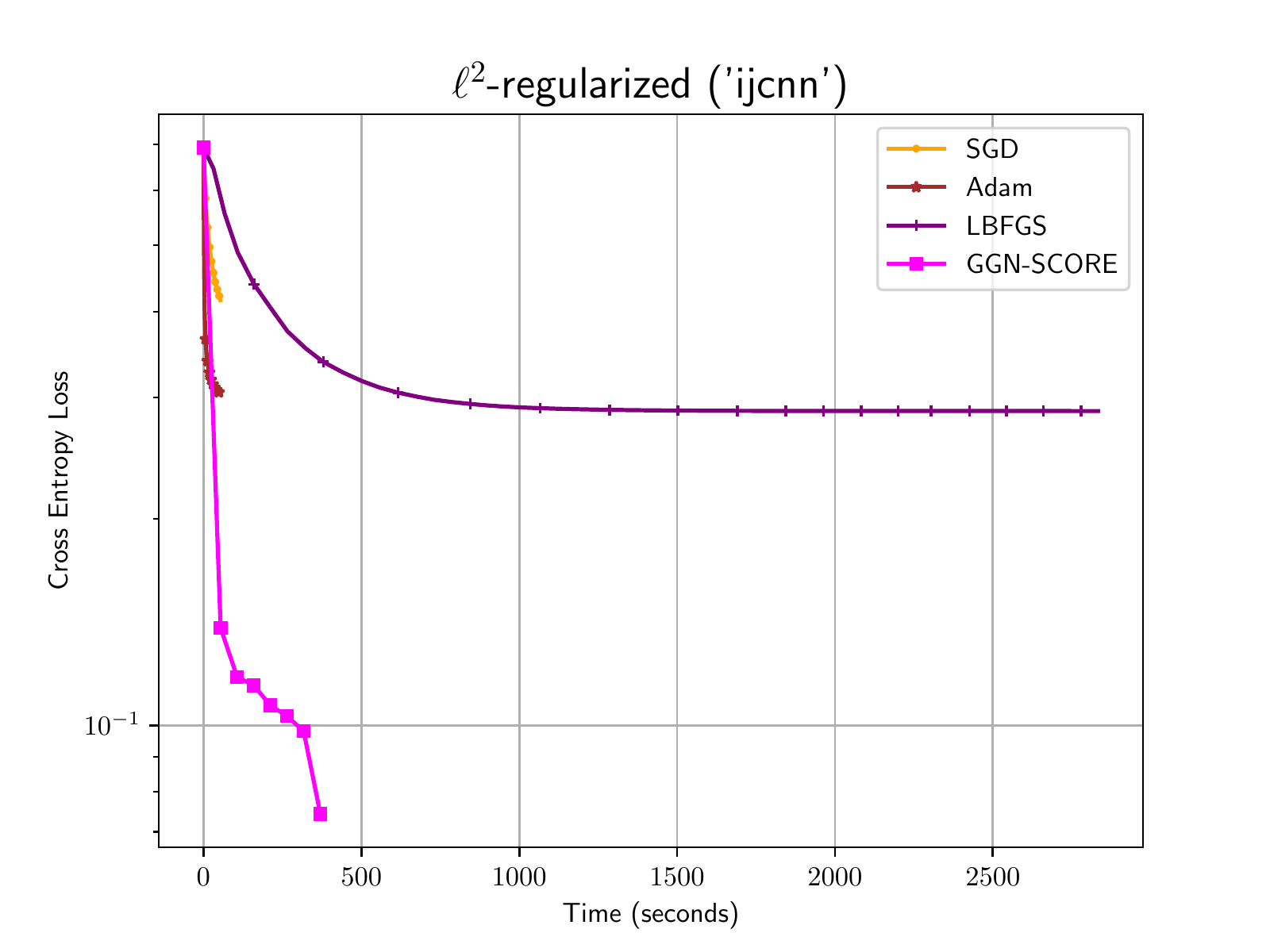}}
			\centerline{\includegraphics[width=1.2\linewidth]{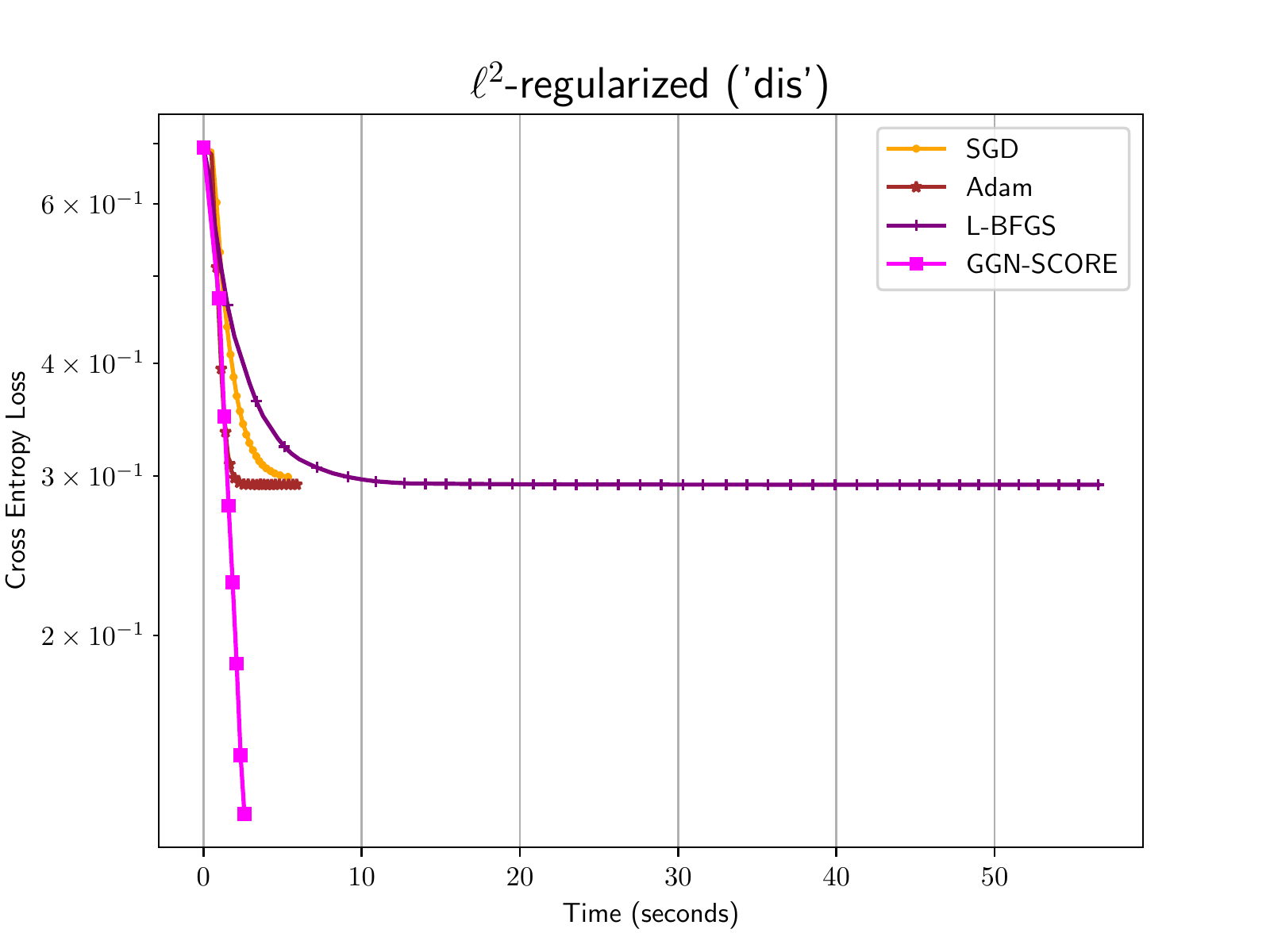}}
			\centerline{\includegraphics[width=1.2\linewidth]{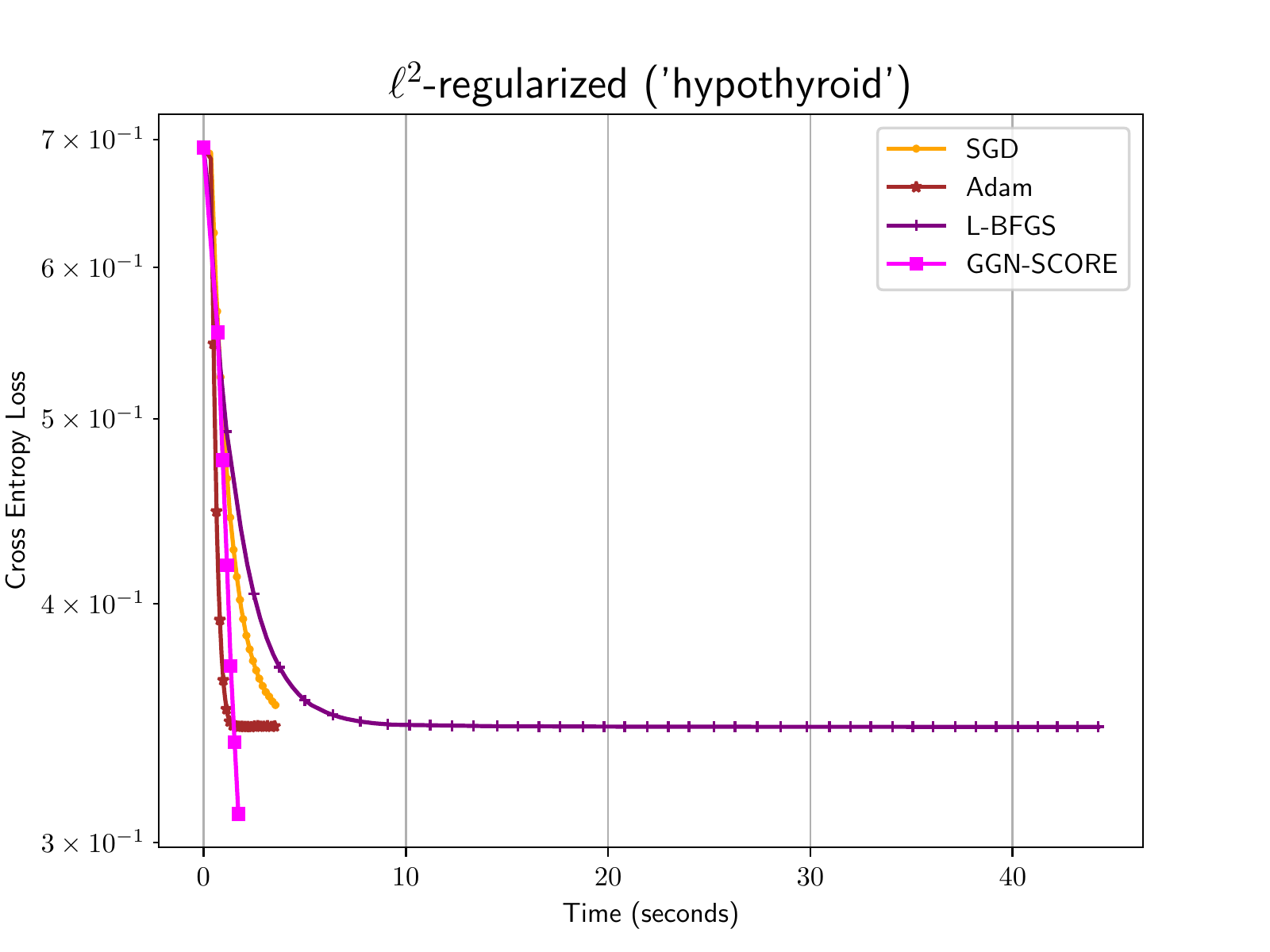}}
			\centerline{\includegraphics[width=1.2\linewidth]{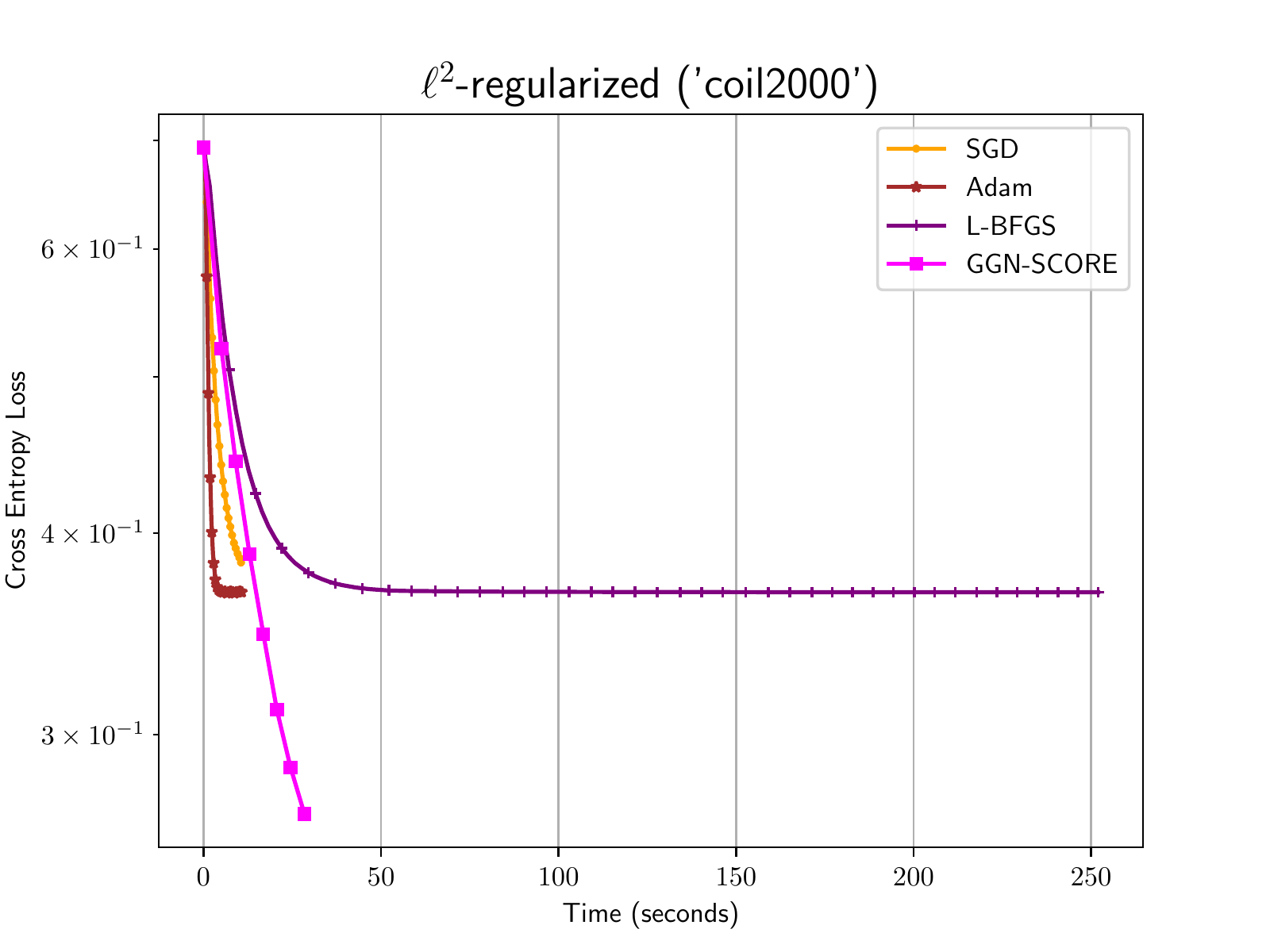}}
			\centerline{\includegraphics[width=1.2\linewidth]{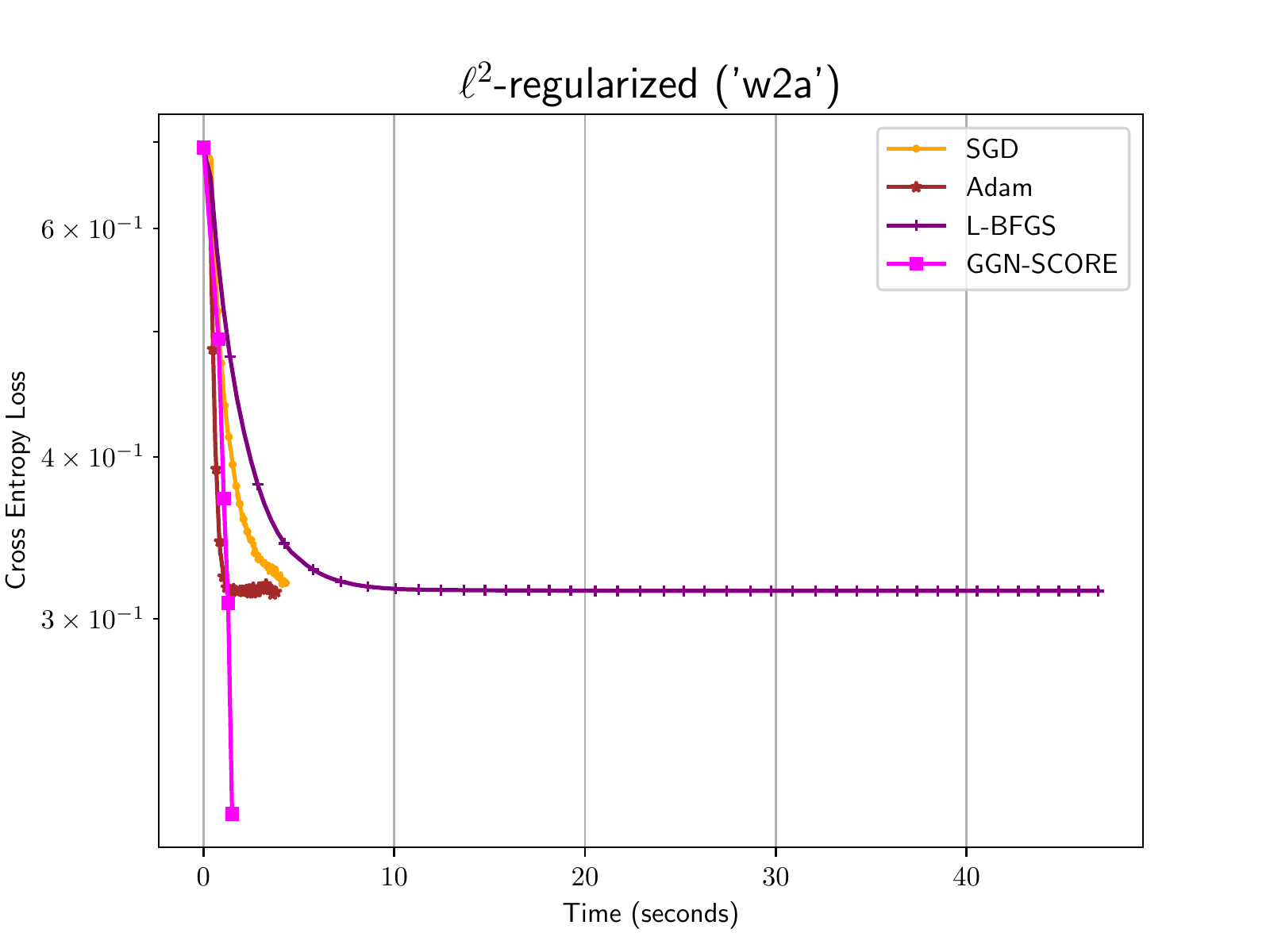}}
		\end{multicols}
	\end{subfigure}
	\caption{Convergence curves for GGN-SCORE, SGD, Adam and L-BFGS in the proposed convex problem. $m = 512$ for \texttt{w2a}, \texttt{dis} and \texttt{hypothyroid}, $m=2048$ for \texttt{coil2000}, and $m=4096$ for \texttt{ijcnn1}.}
	\label{fig:lossplots}
\end{figure*}
\begin{figure*}
	\begin{subfigure}{1.0\textwidth}
		\vspace{-2.0\baselineskip}
		\begin{multicols}{4}
			\centerline{\includegraphics[width=1.2\linewidth]{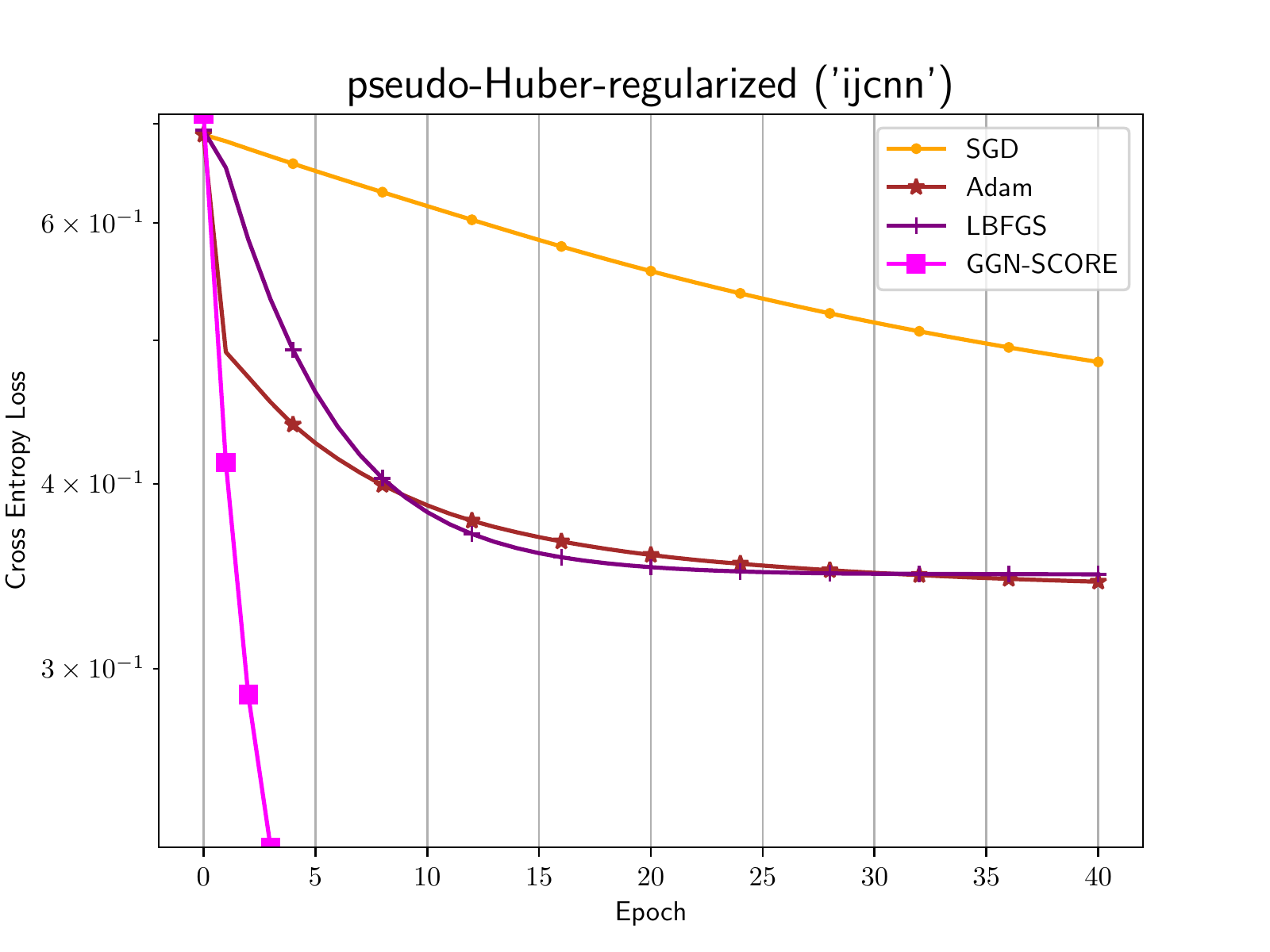}}
			\centerline{\includegraphics[width=1.2\linewidth]{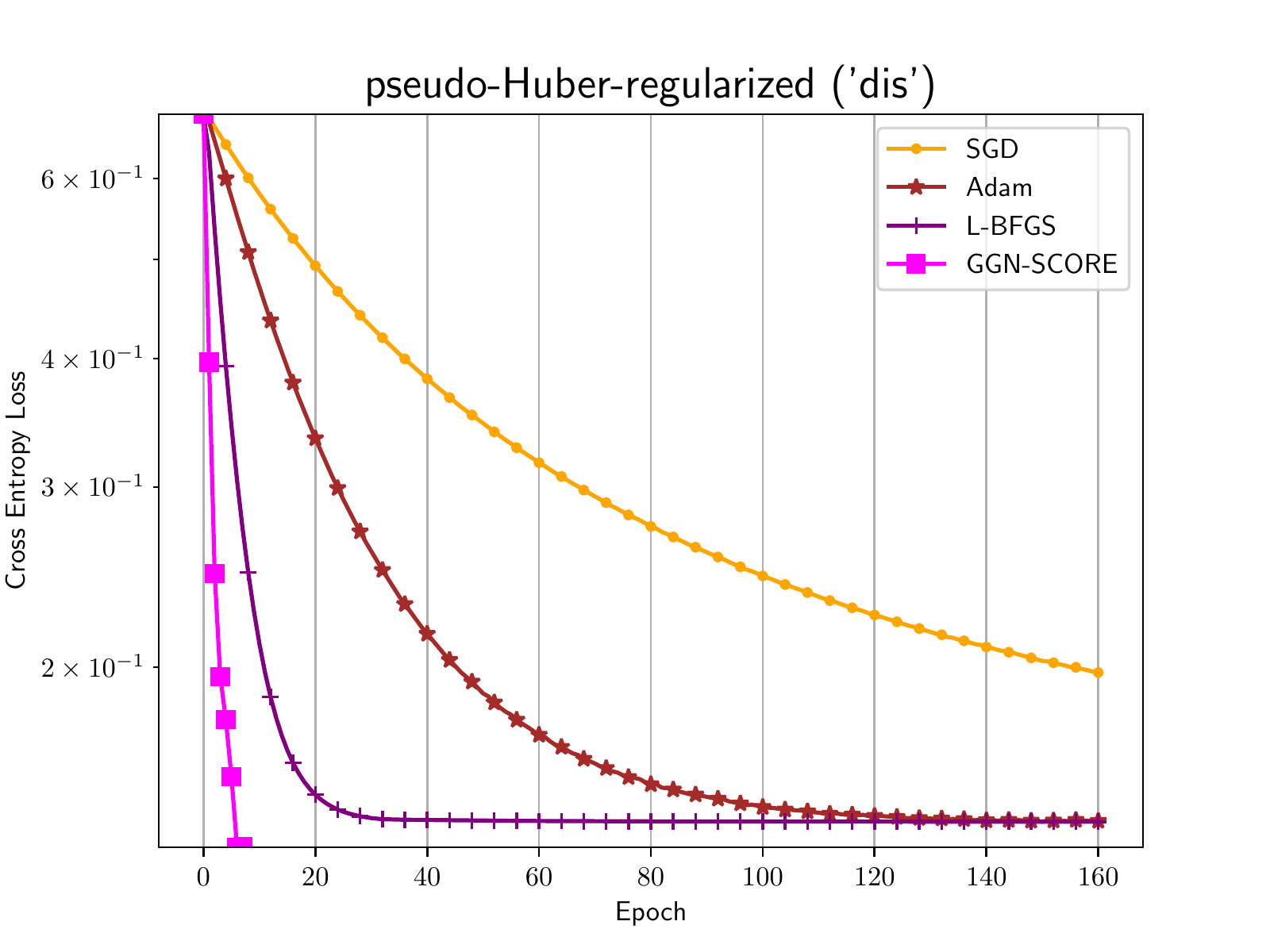}}
			\centerline{\includegraphics[width=1.2\linewidth]{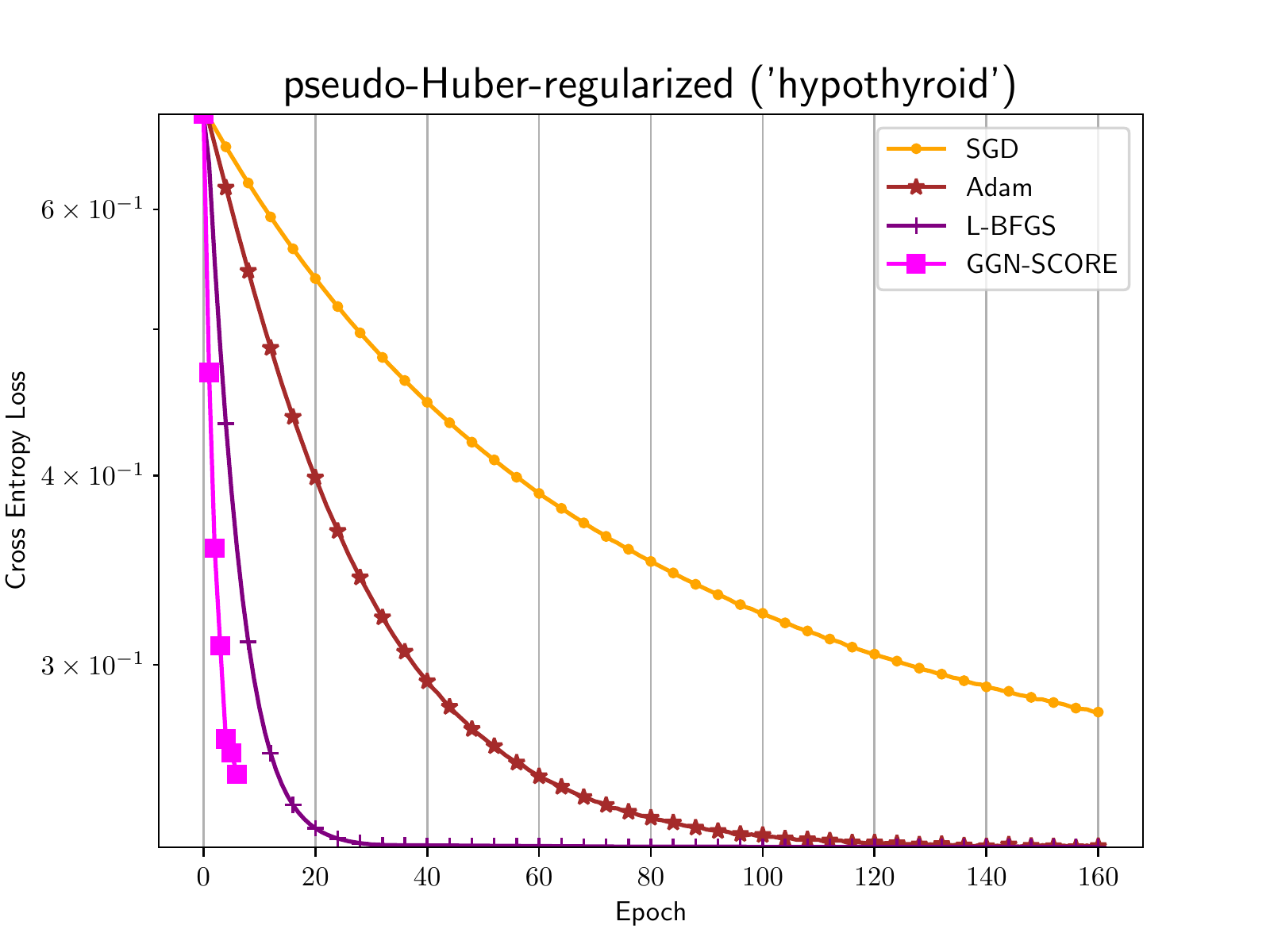}}
			\centerline{\includegraphics[width=1.2\linewidth]{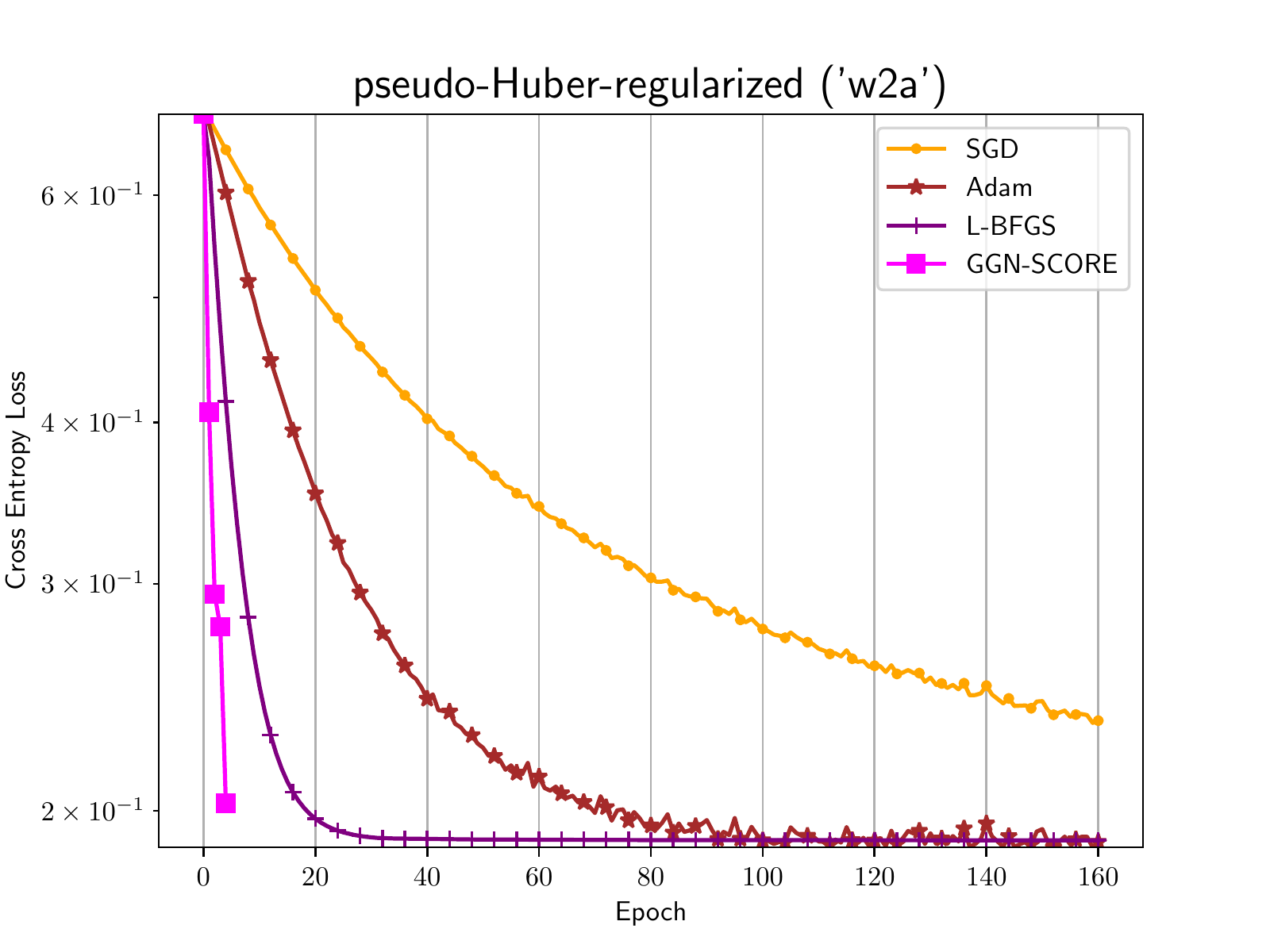}}

			\centerline{\includegraphics[width=1.2\linewidth]{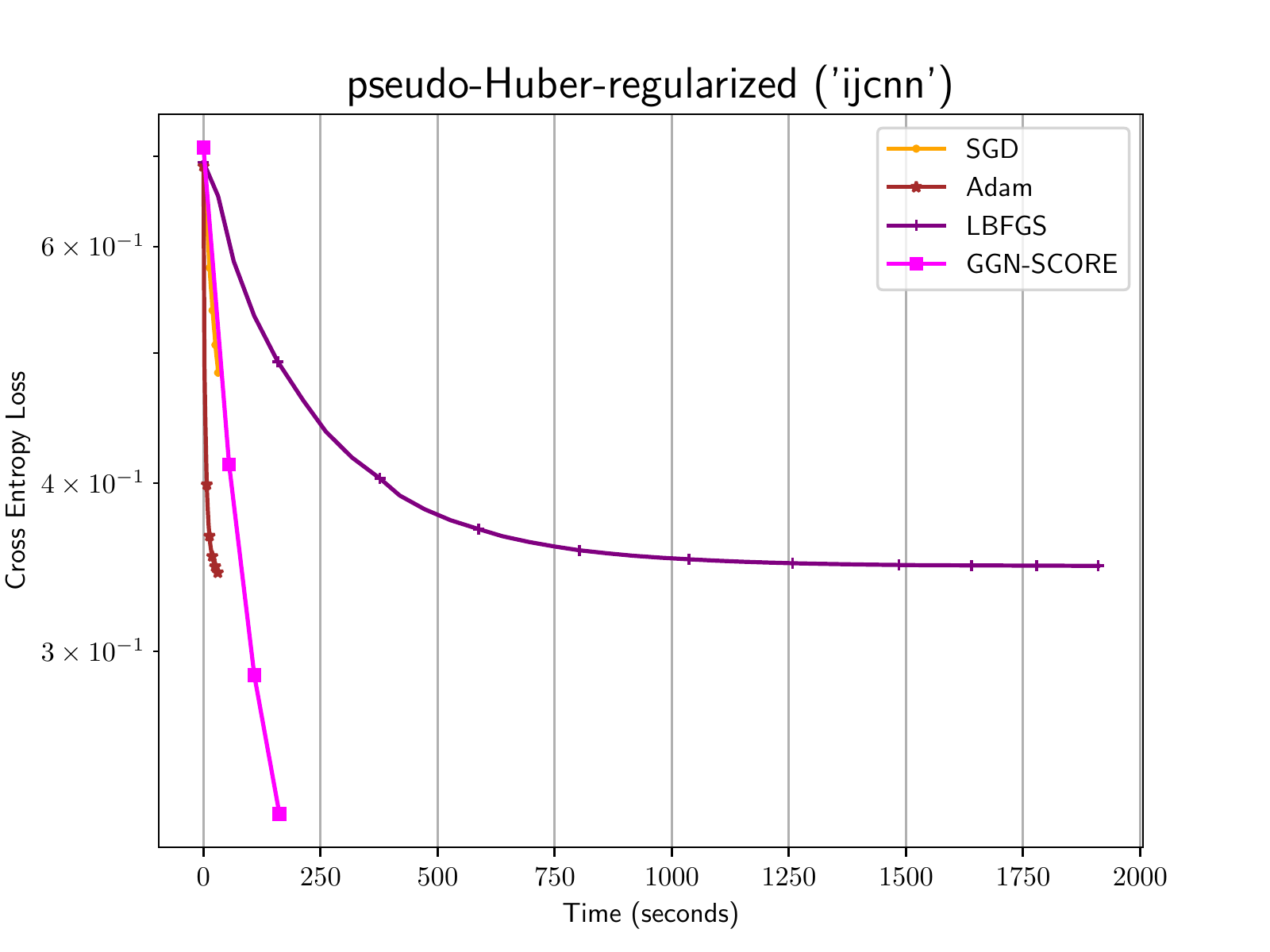}}
			\centerline{\includegraphics[width=1.2\linewidth]{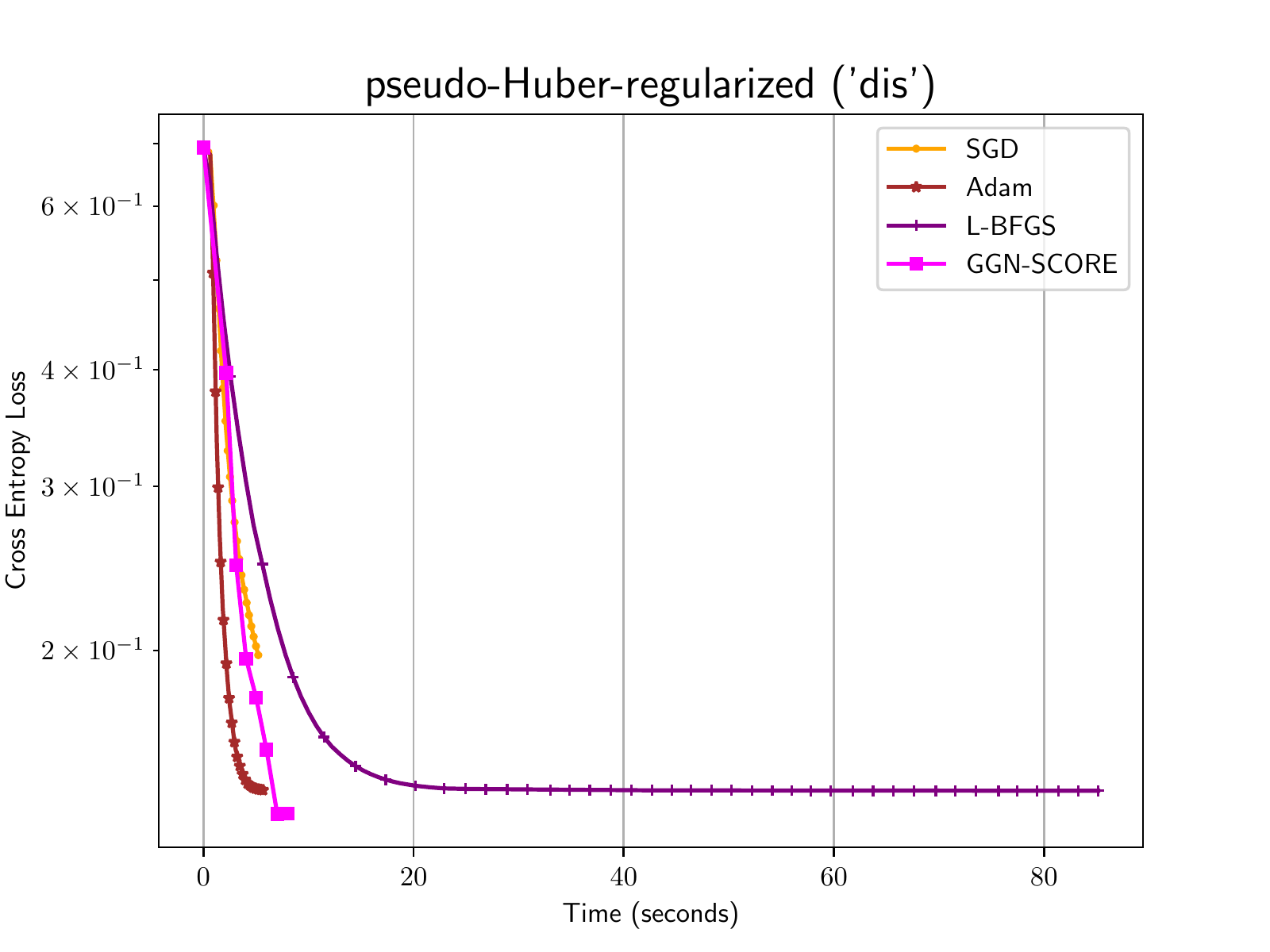}}
			\centerline{\includegraphics[width=1.2\linewidth]{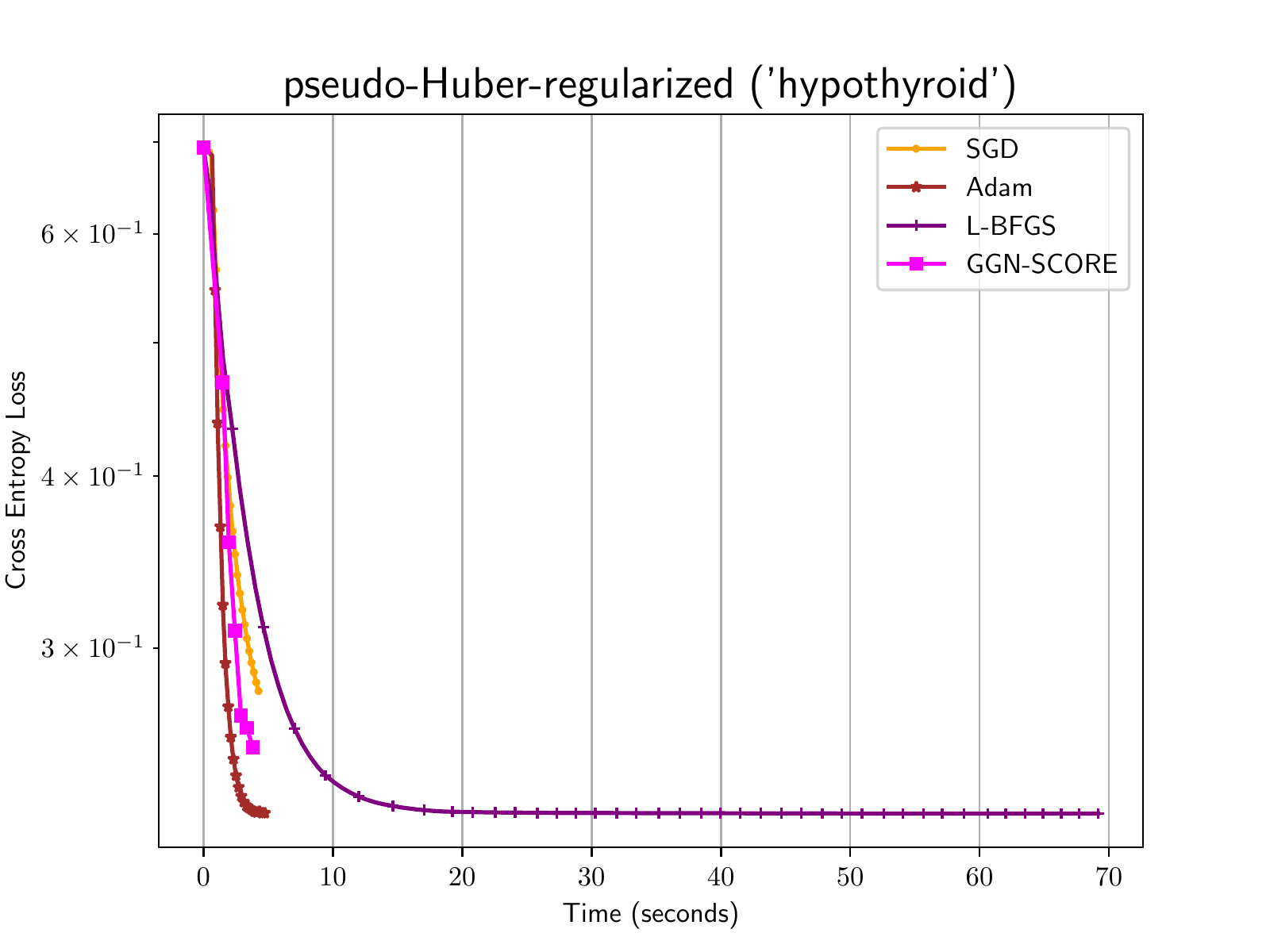}}
			\centerline{\includegraphics[width=1.2\linewidth]{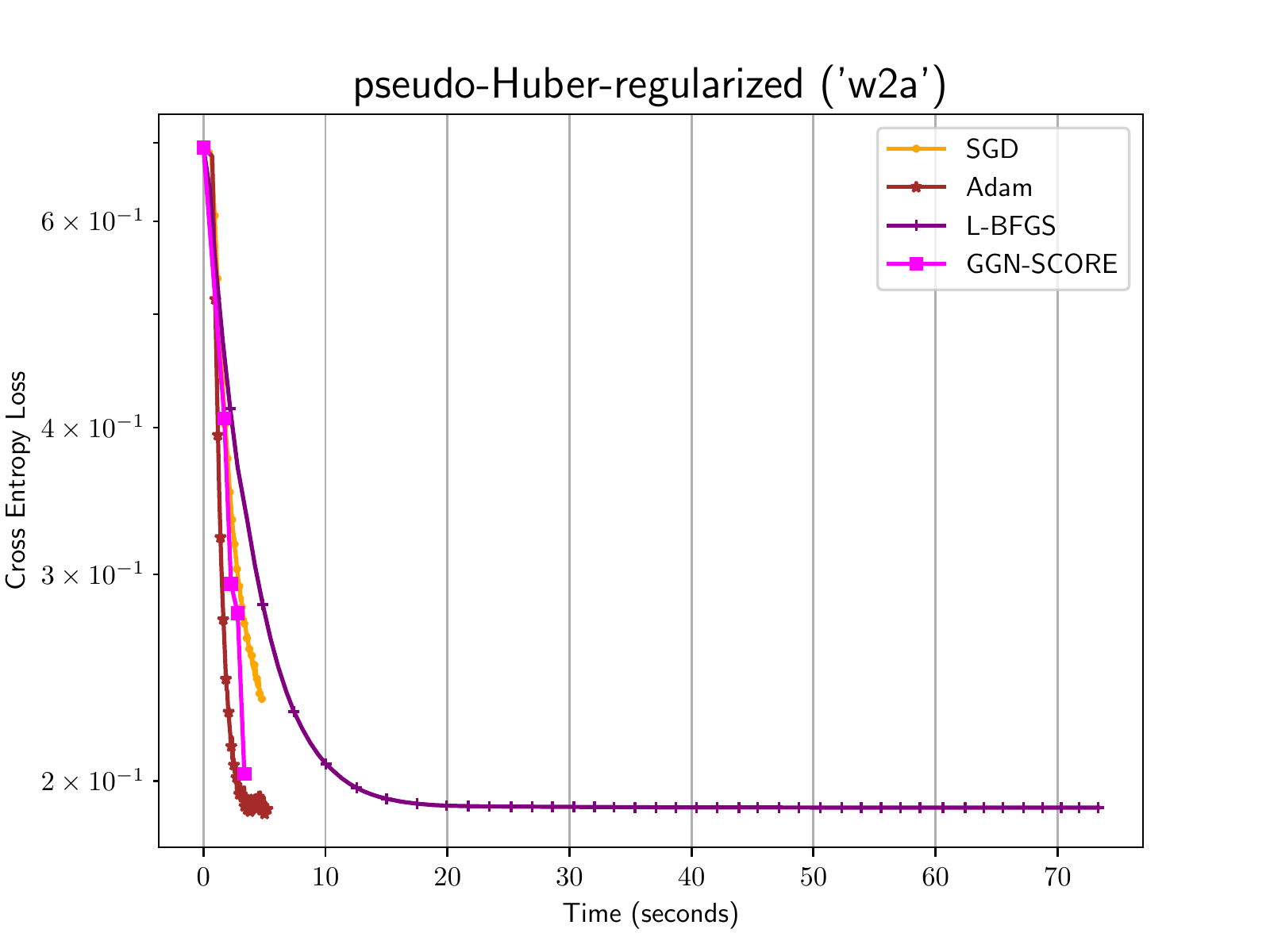}}

			\centerline{\includegraphics[width=1.2\linewidth]{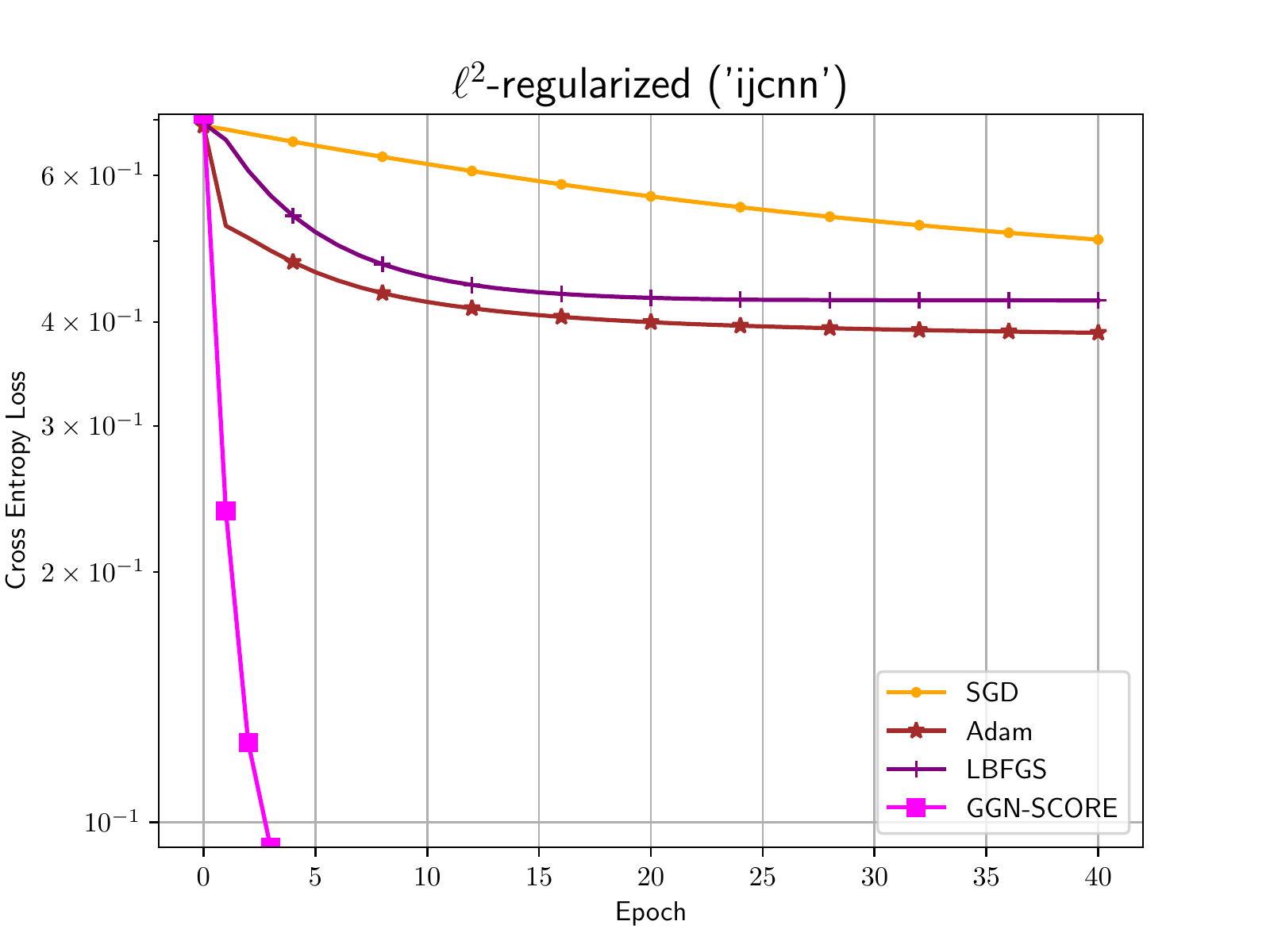}}
			\centerline{\includegraphics[width=1.2\linewidth]{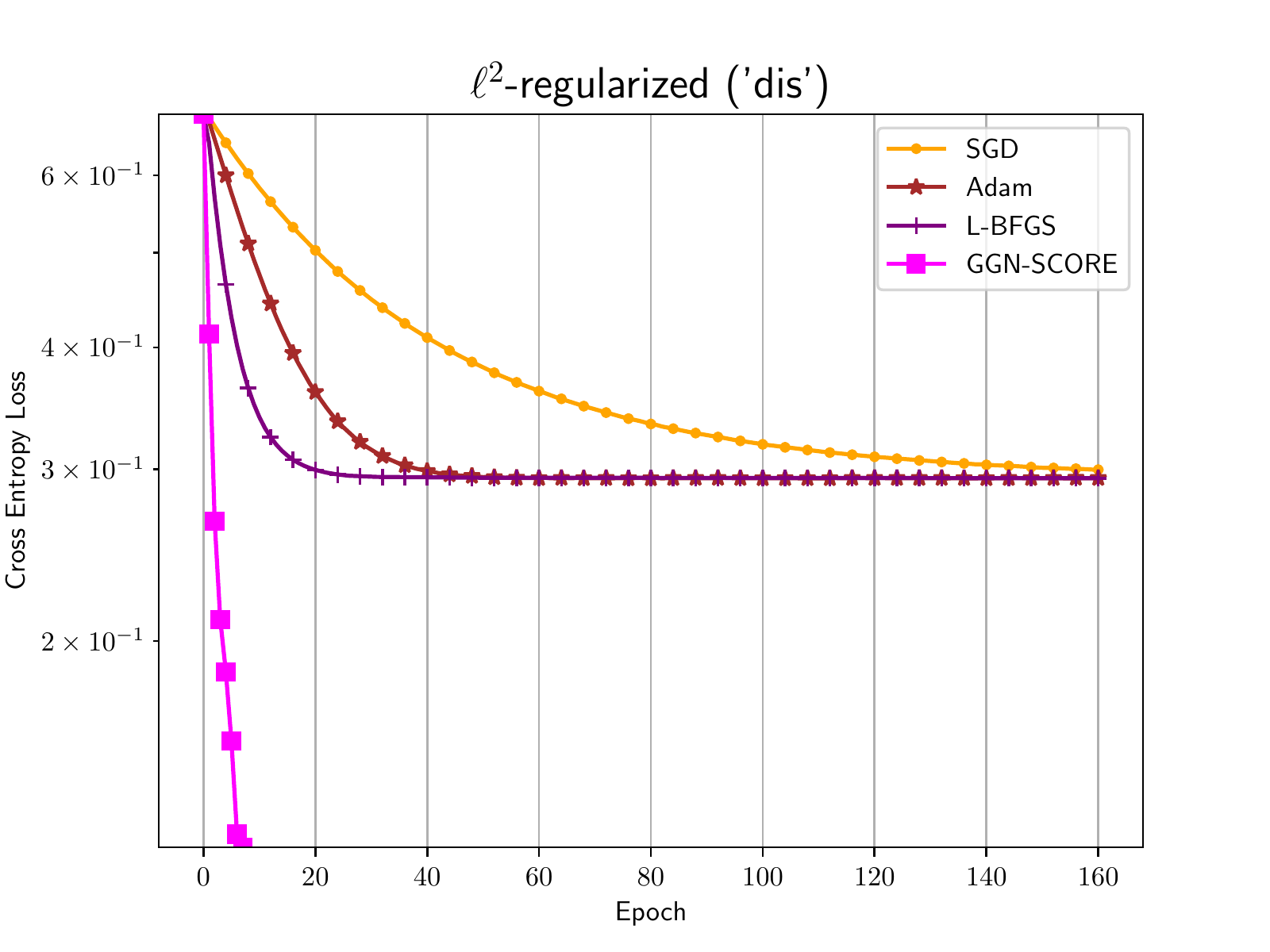}}
			\centerline{\includegraphics[width=1.2\linewidth]{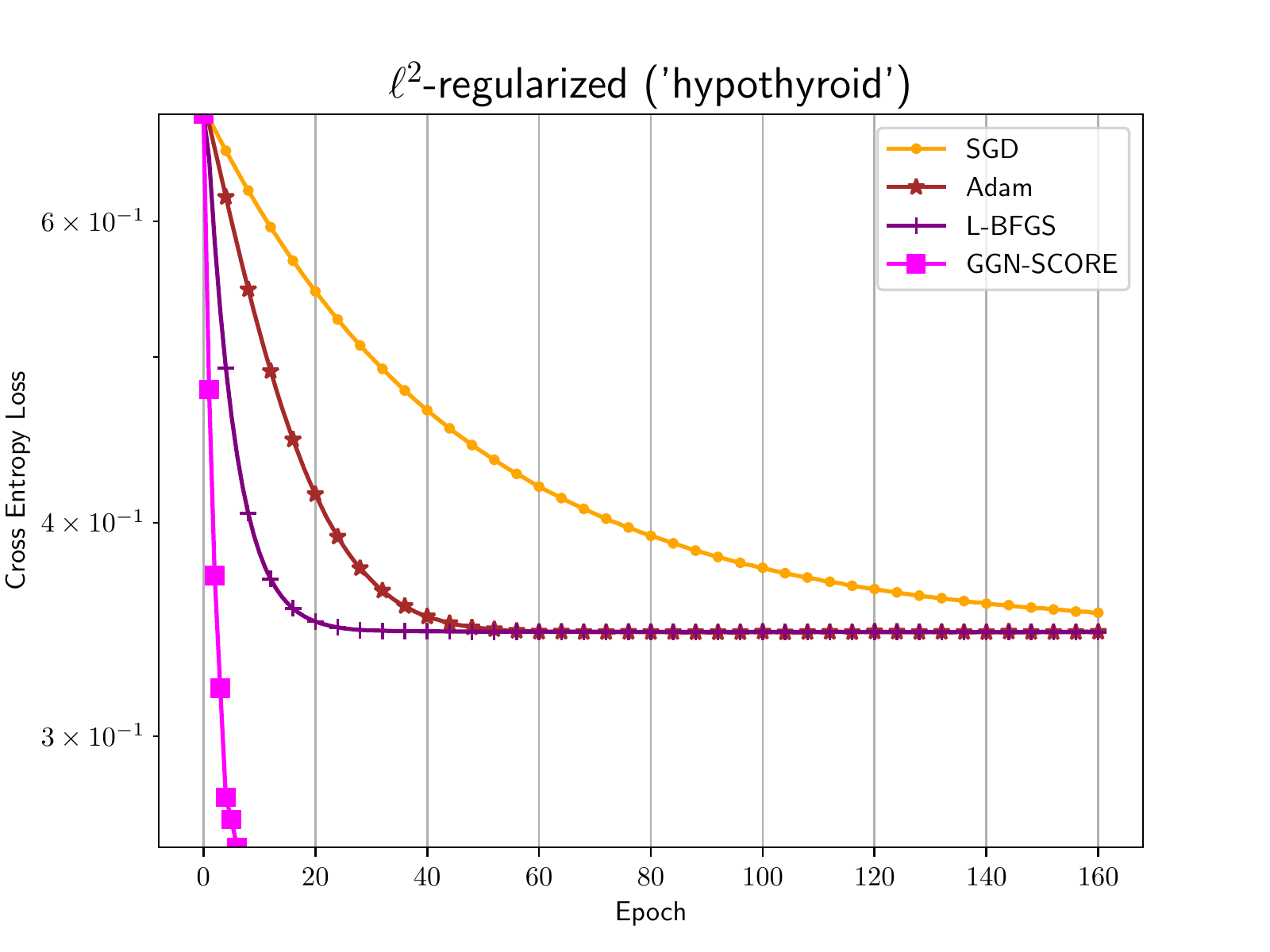}}
			\centerline{\includegraphics[width=1.2\linewidth]{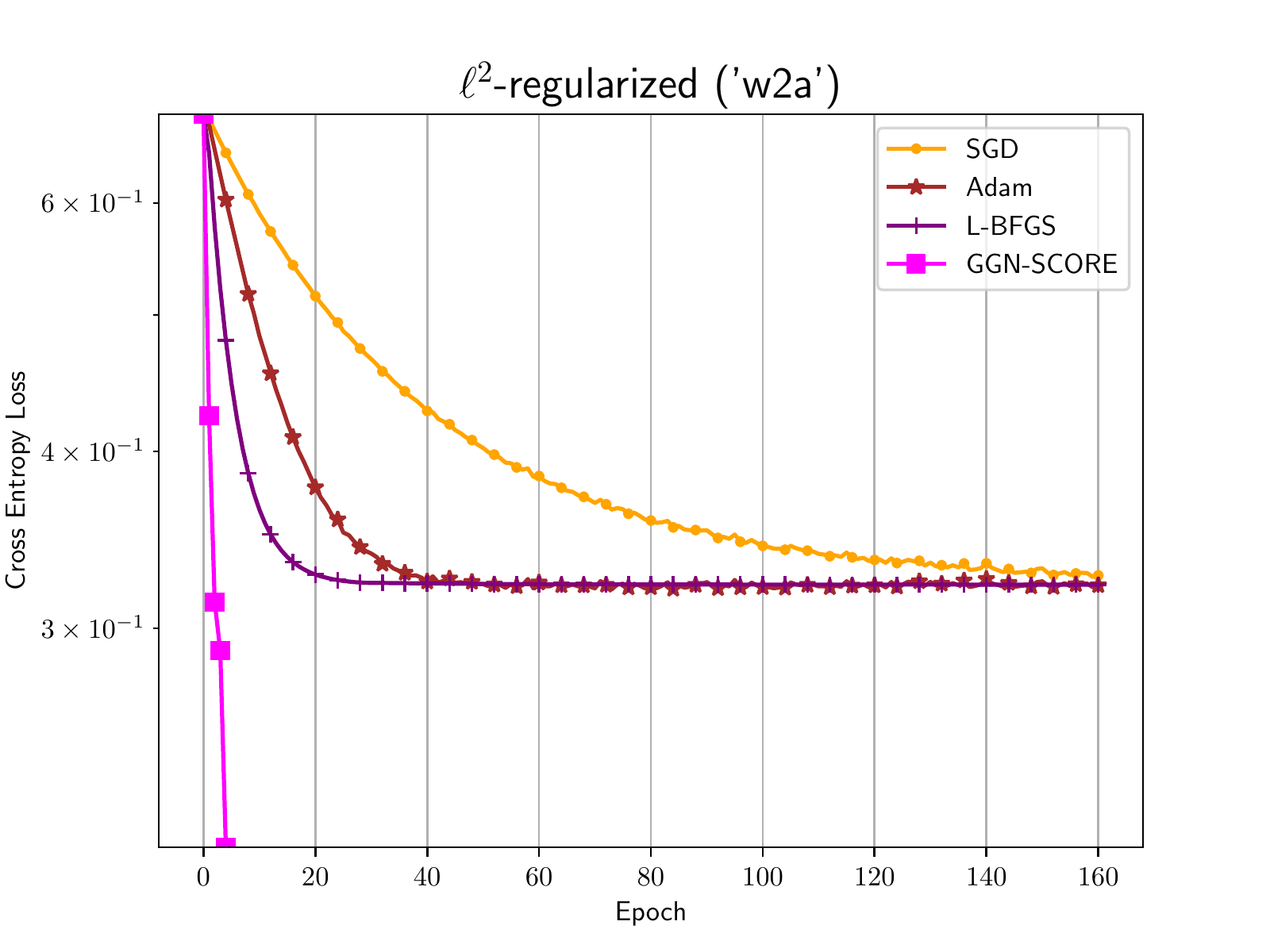}}

			\centerline{\includegraphics[width=1.2\linewidth]{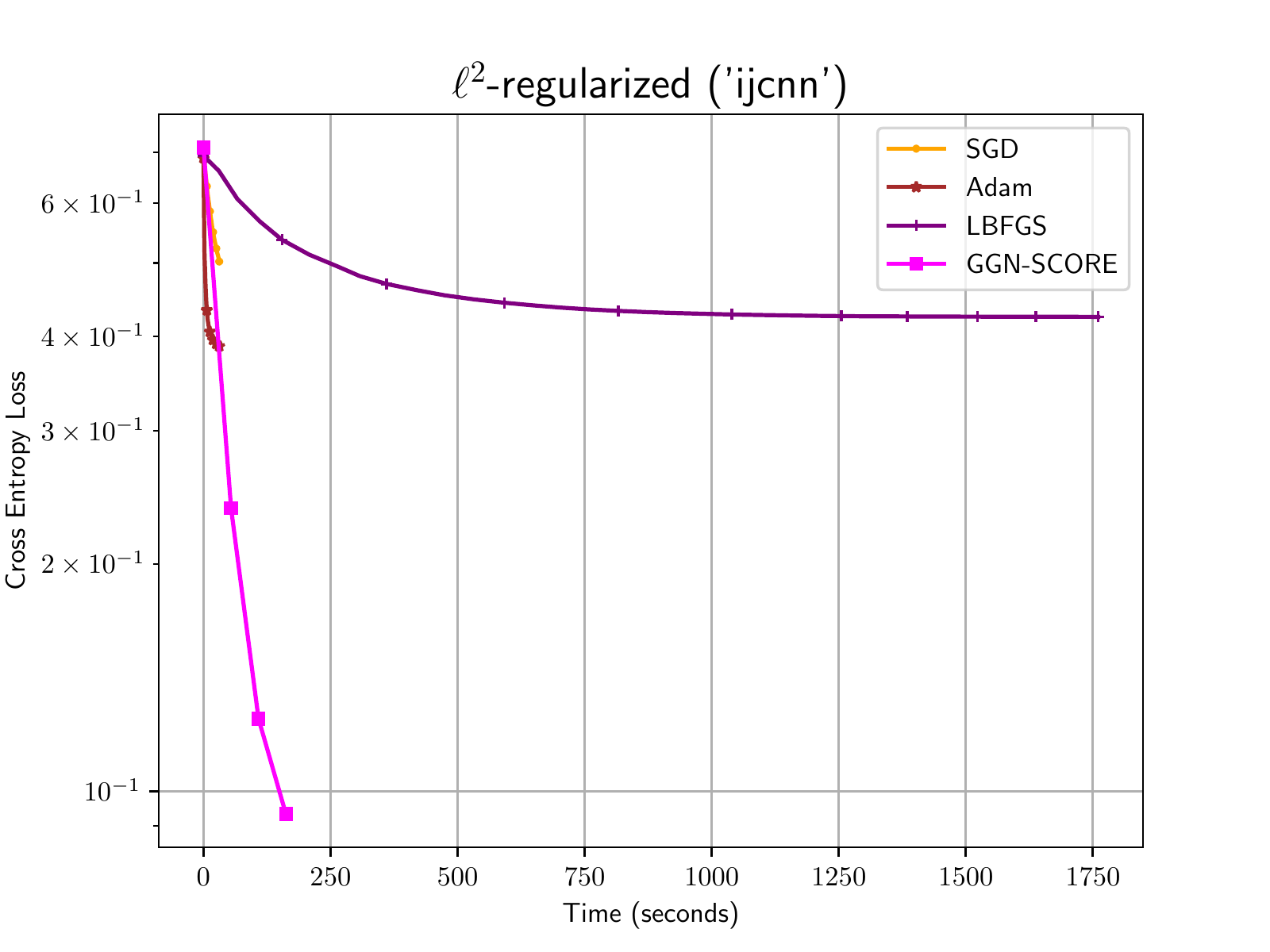}}
			\centerline{\includegraphics[width=1.2\linewidth]{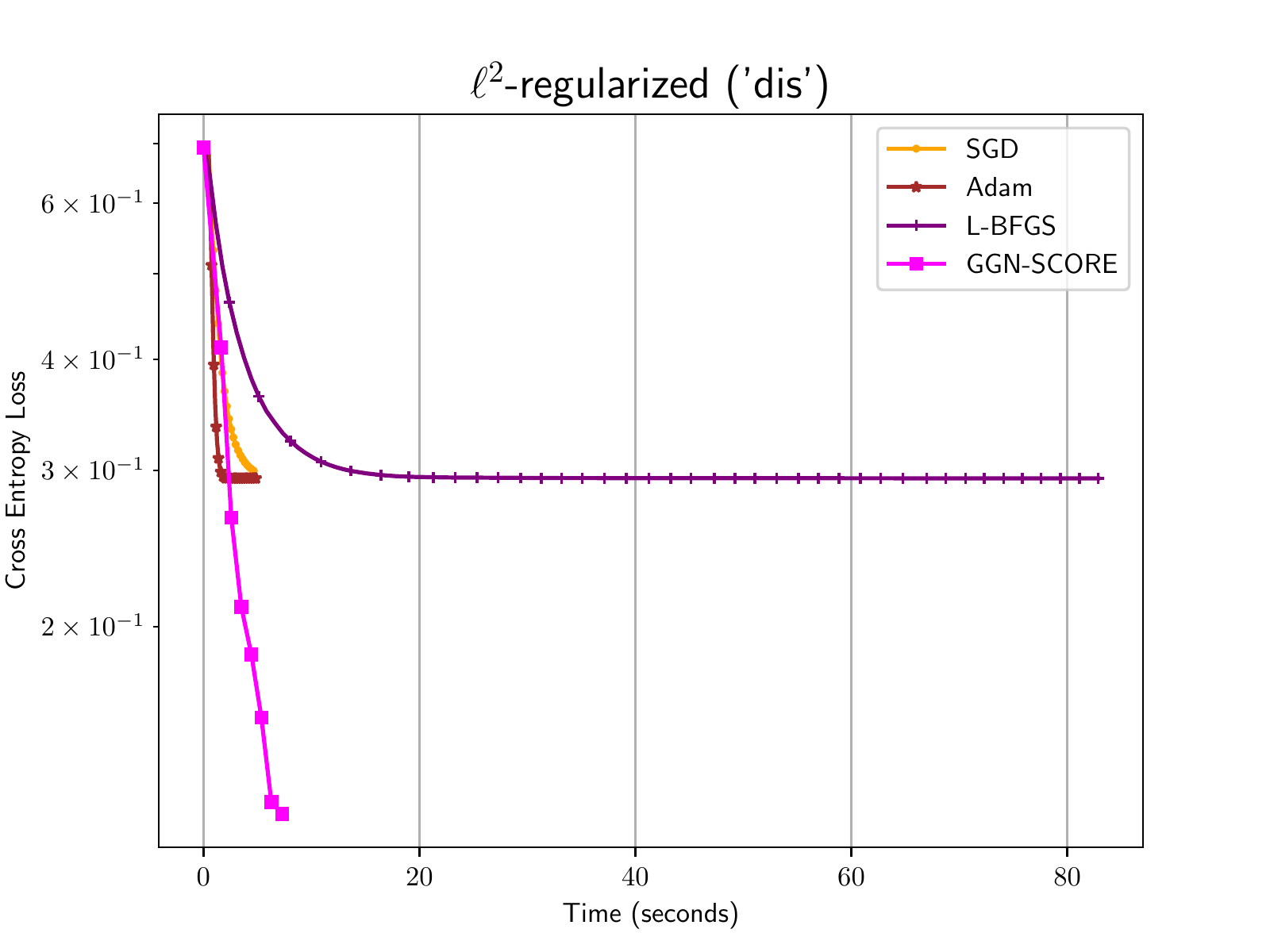}}
			\centerline{\includegraphics[width=1.2\linewidth]{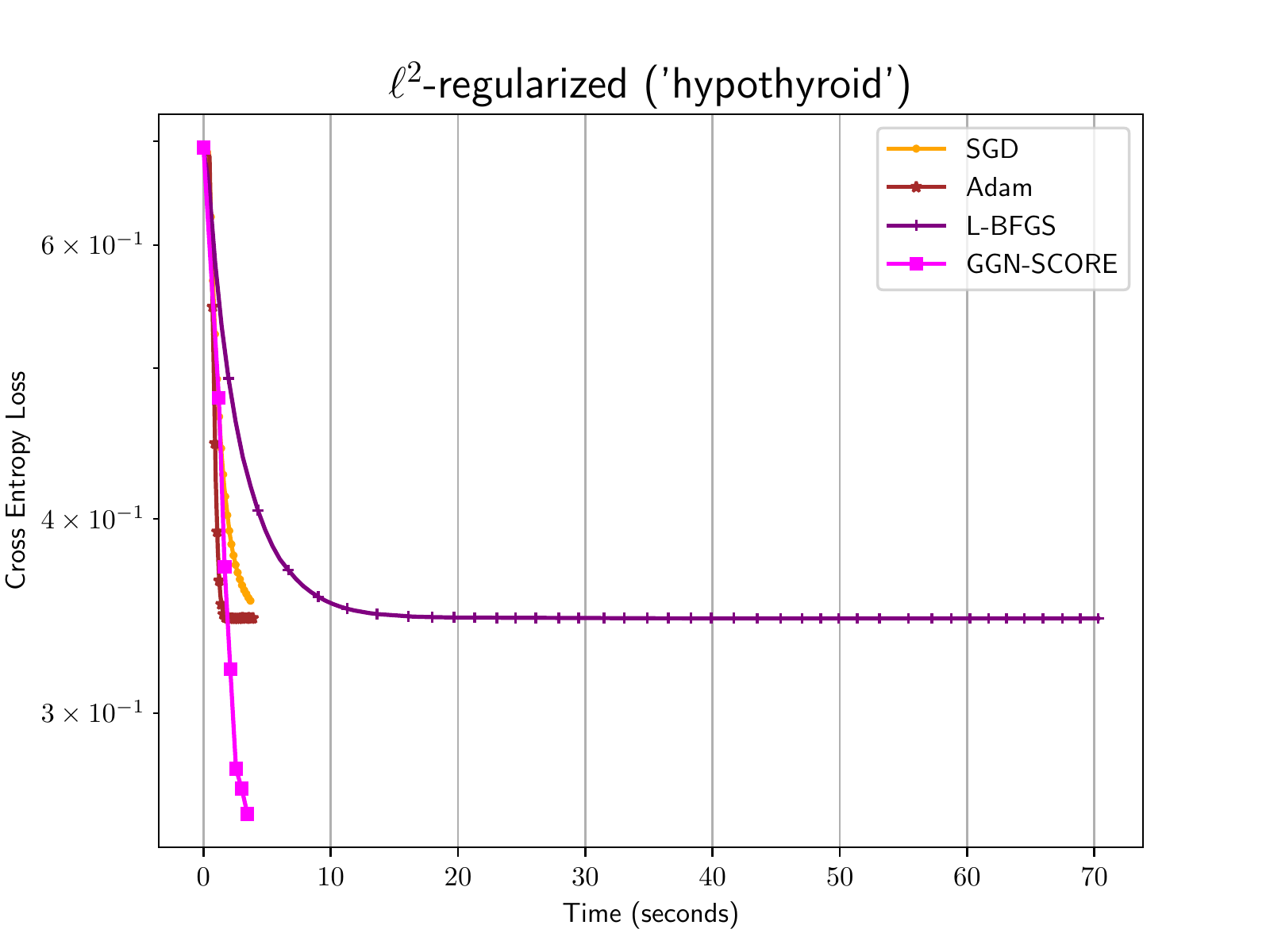}}
			\centerline{\includegraphics[width=1.2\linewidth]{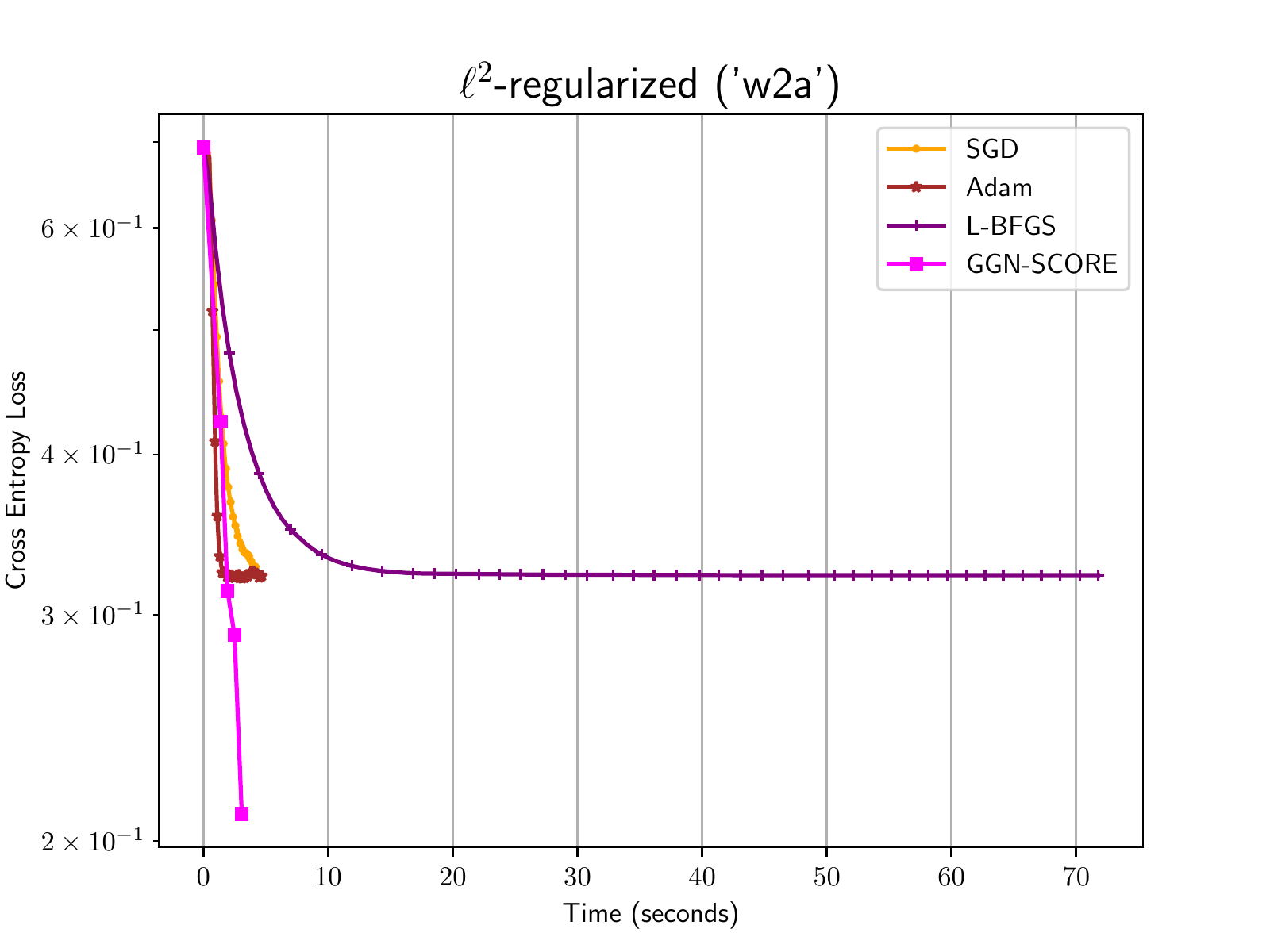}}
		\end{multicols}
	\end{subfigure}
	\caption{Convergence curves for GGN-SCORE, SGD, Adam and L-BFGS in the non-convex neural network training problem: overparameterized for \texttt{dis}, \texttt{hypothyroid}, \texttt{w2a} and \texttt{coil2000}. $m = 512$ for \texttt{w2a}, \texttt{dis} and \texttt{hypothyroid}, $m=2048$ for \texttt{coil2000}, and $m=4096$ for \texttt{ijcnn1}.}
	\label{fig:nonconvex}
\end{figure*}
\begin{figure*}
	\begin{subfigure}{1.0\textwidth}
		\vspace{-2.0\baselineskip}
		\centerline{\includegraphics[width=0.9\textwidth]{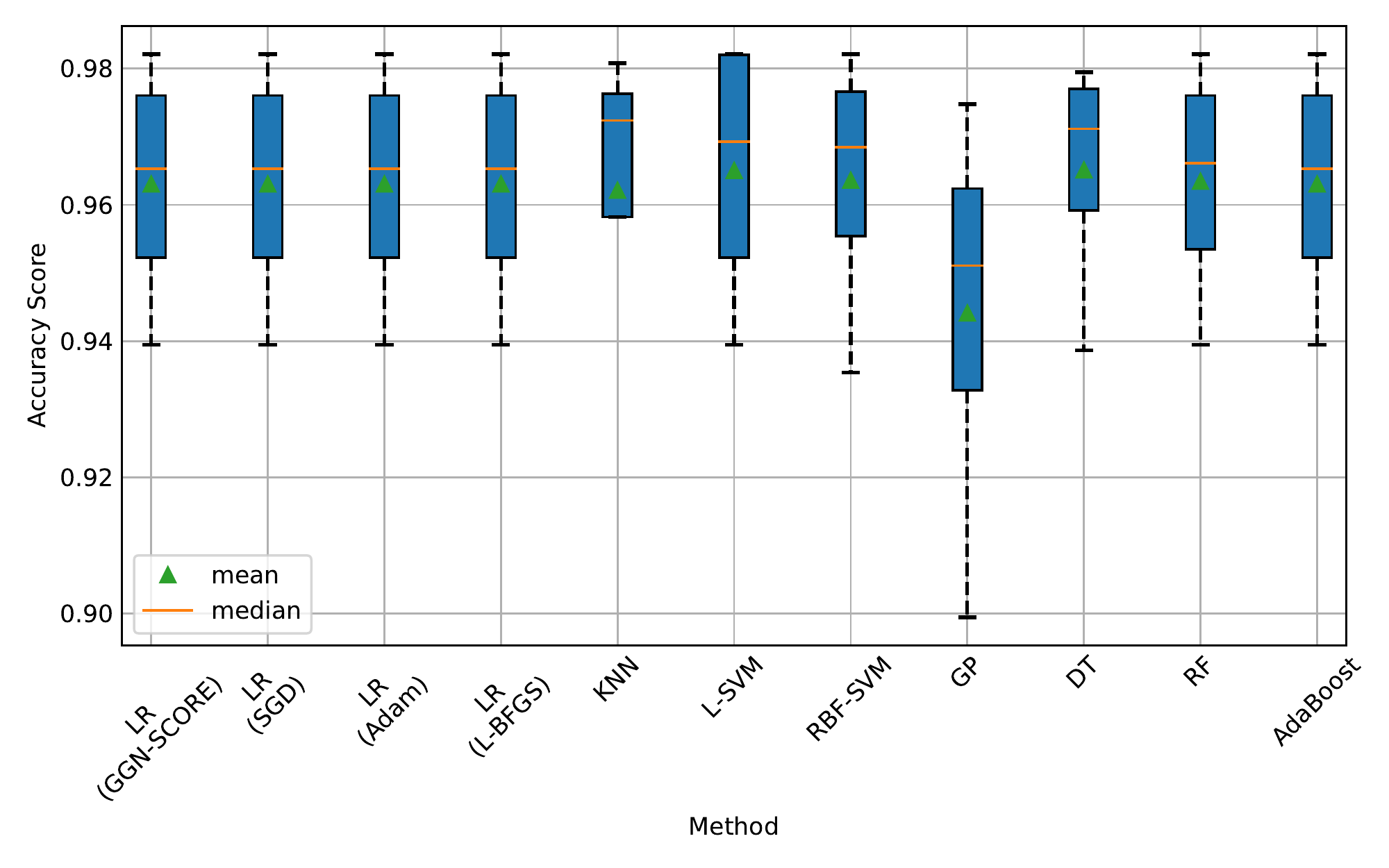}}\caption{Mean accuracy score.}
		\vfil
		\centerline{\includegraphics[width=0.9\textwidth]{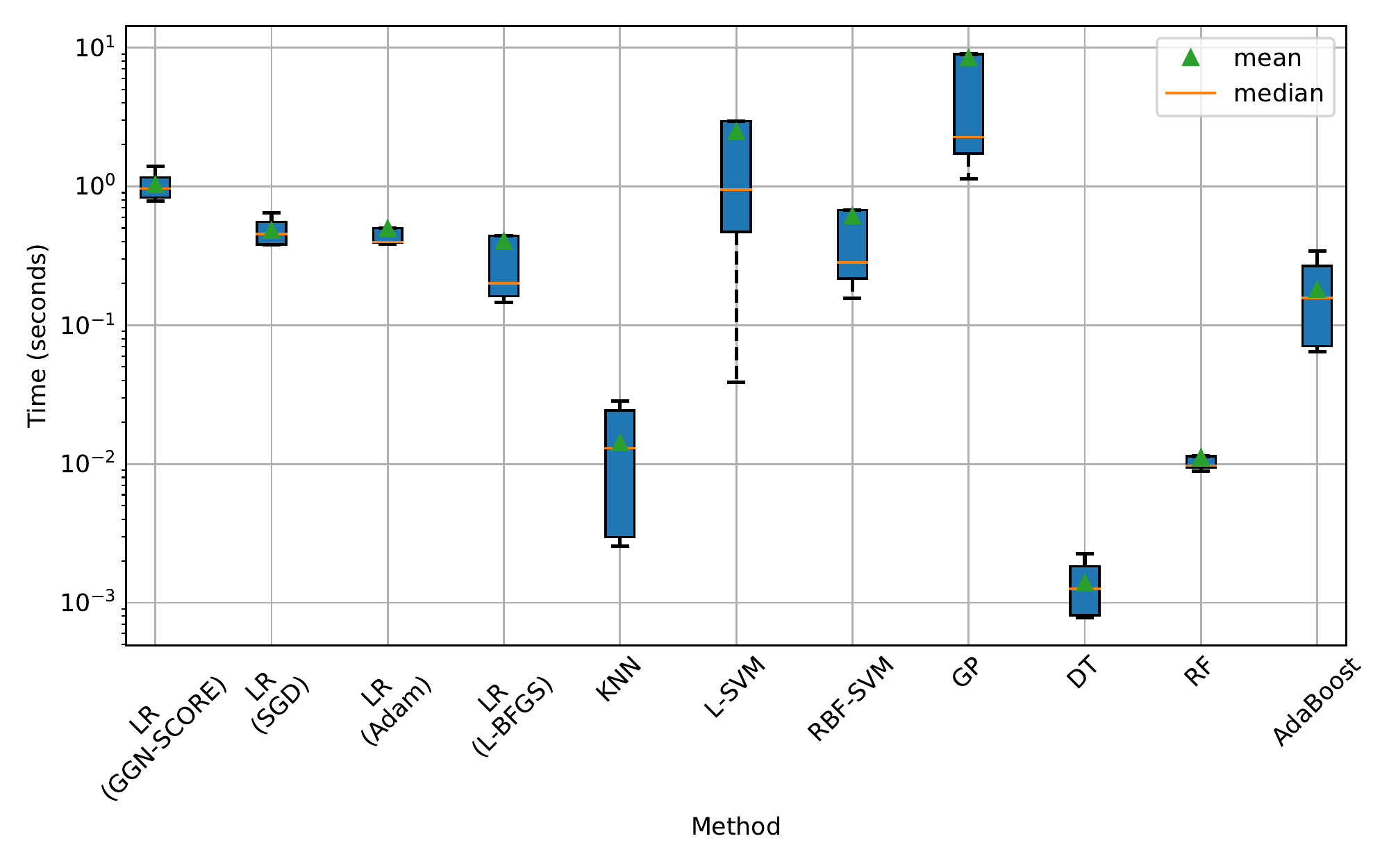}}\caption{Mean CPU time (in seconds).}
	\end{subfigure}
	\caption{Classification task results using the method of this paper -- logistic regression (LR), k-nearest neighbours (KNN), linear support vector machine (L-SVM), RBF-SVM, Gaussian process (GP), decision tree (DT), random forest (RF) and adaptive boosting (AdaBoost) techniques for \texttt{dis}, \texttt{hypothyroid}, \texttt{coil2000} and \texttt{w2a} datasets. LR and SVM methods use the $\ell_2$ penalty function with parameter $\lambda = 1.0$. Solutions with GGN-SCORE, SGD, Adam, and L-BFGS are obtained over one data pass with $m=128$.}
	\label{fig:benchmarkscores}
\end{figure*}

To investigate how well the learned model generalizes, we use the binary accuracy metric which measures how often the model predictions match the true labels when presented with new, previously unseen data: $Accuracy = \frac{1}{N}\sum_{n=1}^{N}\left(2\bm{y}_n\hat{\bm{y}}_n-\bm{y}_n-\hat{\bm{y}}_n+1\right)$. While GGN-SCORE converges faster than SGD, Adam and L-BFGS methods, it generalizes comparatively well. The results are further compared with other known binary classification techniques to measure the quality of our solutions. The accuracy scores for \texttt{dis}, \texttt{hypothyroid}, \texttt{coil2000} and \texttt{w2a} datasets, with a $60$:$40$ train:test split each, are computed on the test set. The mean scores are compared with those from the different classification techniques, and are shown in \figref{fig:benchmarkscores}. The CPU runtimes are also compared where it is indicated that on average GGN-SCORE solves each of the problems within one second. This scales well with the other techniques, as we note that while GNN-SCORE solves each of the problems in high dimensions, the success of most of the other techniques are limited to relatively smaller dimensions of the problems. The obtained results from the classification techniques used for comparison are computed directly from the respective scikit-learn \cite{scikit-learn} functions.

The L-BFGS experiments are implemented with PyTorch \cite{paszke2017automatic} (v.~1.10.1+cu102) in full-batch mode. The GGN-SCORE, Adam and SGD methods are implemented using the open-source Keras API with TensorFlow \cite{tensorflow2015-whitepaper} backend (v.~2.7.0). All experiments are performed on a laptop with dual (2.30GHz + 2.30GHz) Intel Core i7-11800H CPU and 16GB RAM.

In summary, GGN-SCORE converges way faster (in terms of number of epochs) than SGD, Adam, and L-BFGS, and generalizes comparatively well. Experimental results show the computational convenience and elegance achieved by our ``augmented" approach for including regularization functions in the GGN approximation. Although GGN-SCORE comes with a higher computational cost (in terms of wall-clock time per iteration) than first-order methods on average, if per-iteration learning time is not provided as a bottleneck, this may not become an obvious issue as we need to pass the proposed optimizer on the dataset only a few times (epochs) to obtain superior function approximation and relatively high-quality solutions in our experiments.
\section{Conclusion}\label{ss:conclusion}
In this paper, we have proposed GGN-SCORE, a generalized Gauss-Newton-type algorithm for solving unconstrained regularized minimization problems, where the regularization function is considered to be self-concordant. In this generalized setting, we employed a matrix approximation scheme that significantly reduces the computational overhead associated with the method. Unlike existing techniques that impose self-concordance on the problem's objective function, our analysis involves a less restrictive condition from a practical point of view but similarly benefits from the idea of self-concordance by considering scaled optimization step-lengths that depend on the self-concordant parameter of the regularization function. We proved a quadratic local convergence rate for our method under certain conditions, and validate its efficiency in numerical experiments that involve both strongly convex problems and the general non-convex problems that arise when training neural networks. In both cases, our method compare favourably against Adam, SGD, and L-BFGS methods in terms of per-iteration convergence speed, as well as some machine learning techniques used in binary classification, in terms of solution quality.

In future research, it would be interesting to relax some conditions on the problem and analyze a global convergence rate for the proposed method. We would also consider an analysis for our method in a general non-convex setting even though numerically we have observed a similar convergence speed as the strongly convex case.

\bibliographystyle{unsrt}
\bibliography{references}

\appendix

\section{Useful Results}\label{sec:appA}
Following Assumptions \ref{thm:ass1} -- \ref{thm:ass2}, we get that \cite[Theorem 2.1.6]{nesterov2018lectures}
\begin{align}
	\gamma_l\bm{I}_{n_w}\preceq\partial^2g(\bm{\theta})\preceq \gamma_u \bm{I}_{n_w},\quad
	\gamma_a\bm{I}_{n_w}\preceq\partial^2h(\bm{\theta})\preceq \gamma_b \bm{I}_{n_w},
\end{align}
where $\bm{I}_{n_w}\in\mathbb{R}^{n_w \times n_w}$ is an identity matrix. Consequently,
\begin{align}
	(\gamma_l+\lambda\gamma_a)\bm{I}_{n_w}\preceq\partial^2\mathcal{L}(\bm{\theta})\preceq (\gamma_u+\lambda\gamma_b) \bm{I}_{n_w}.
\end{align}
In addition, the second derivative of $\mathcal{L}(\bm{\theta})$ is $(\gamma_g + \lambda\gamma_h)$-Lipschitz continuous $\forall \bm{x} \in \rr^{n_p}, \bm{y} \in \rr^d$, that is,
\begin{align}
	\left\|\partial^2 \mathcal{L}(\bm{y}, f(\bm{\theta_1};\bm{x})) - \partial^2 \mathcal{L}(\bm{y}, f(\bm{\theta_2};\bm{x}))\right\| \le (\gamma_g + \lambda\gamma_h)\left\|\bm{\theta_1} - \bm{\theta_2}\right\|. \label{eq:lipell2}
\end{align}

\begin{lemma}\label{thm:hessppty}
	Let Assumptions \ref{thm:ass1}, \ref{thm:ass1-1} and \ref{thm:ass2} hold. Let $\bm{\theta_1}, \bm{\theta_2}$ be any two points in $\rr^{n_w}$. Then we have
	\begin{align}
		\left\|\bm{g}(\bm{\theta_1}) - \bm{g}(\bm{\theta_2}) - \left<\bm{H}(\bm{\theta_2}),\bm{\theta_1}-\bm{\theta_2}\right>\right\| \le \frac{\gamma_g + \lambda\gamma_h}{2}\left\|\bm{\theta_1-\bm{\theta_2}}\right\|^2, \label{eq:hessppty1}
		\\
		\left\vert\mathcal{L}(\bm{\theta_1})-\mathcal{L}(\bm{\theta_2})-\left<\bm{g}(\bm{\theta_2}),\bm{\theta_1-\bm{\theta_2}}\right>-\frac{1}{2}\left<\bm{H}(\bm{\theta_2})(\bm{\theta_1}-\bm{\theta_2}),\bm{\theta_1}-\bm{\theta_2}\right>\right\vert\nonumber\\\le\frac{\gamma_g + \lambda\gamma_h}{6}\left\|\bm{\theta_1-\bm{\theta_2}}\right\|^3. \label{eq:hessppty2}
	\end{align}
\end{lemma}
\begin{proof}
	Fix $\bm{\theta_1}, \bm{\theta_2} \in\rr^{n_w}$. Then
	\begin{align*}
		\bm{g}(\bm{\theta_1}) - \bm{g}(\bm{\theta_2}) = \int_0^1\partial^2\mathcal{L}(\bm{\theta_2}+\tau(\bm{\theta_1}-\bm{\theta_2}))(\bm{\theta_1}-\bm{\theta_2})d\tau.
	\end{align*}
	As $\tau\in[0,1]$, we have $\bm{\theta_2}+\tau(\bm{\theta_1}-\bm{\theta_2}) \in\rr^{n_w}$. Hence writing $\partial^2\mathcal{L}=\bm{H}$, we have
	\begin{align*}
		\bm{g}(\bm{\theta_1}) - \bm{g}(\bm{\theta_2}) &= \int_0^1\bm{H}(\bm{\theta_2}+\tau(\bm{\theta_1}-\bm{\theta_2}))(\bm{\theta_1}-\bm{\theta_2})d\tau\\
		\bm{g}(\bm{\theta_1}) - \bm{g}(\bm{\theta_2}) - \bm{H}(\bm{\theta_2})(\bm{\theta_1}-\bm{\theta_2}) &= \int_0^1\left(\bm{H}(\bm{\theta_2}+\tau(\bm{\theta_1}-\bm{\theta_2}))-\bm{H}(\bm{\theta_2})\right)(\bm{\theta_1}-\bm{\theta_2})d\tau
	\end{align*}
	\begin{align*}
		&\left\|\bm{g}(\bm{\theta_1}) - \bm{g}(\bm{\theta_2}) - \bm{H}(\bm{\theta_2})(\bm{\theta_1}-\bm{\theta_2})\right\| =\\
		&\qquad \qquad  \left\|\int_0^1\left(\bm{H}(\bm{\theta_2}+\tau(\bm{\theta_1}-\bm{\theta_2}))-\bm{H}(\bm{\theta_2})\right)(\bm{\theta_1}-\bm{\theta_2})d\tau\right\|\\
		&\qquad \le \int_0^1\left\|\left(\bm{H}(\bm{\theta_2}+\tau(\bm{\theta_1}-\bm{\theta_2}))-\bm{H}(\bm{\theta_2})\right)(\bm{\theta_1}-\bm{\theta_2})\right\|d\tau\\
		&\qquad \le \int_0^1\left\|\bm{H}(\bm{\theta_2}+\tau(\bm{\theta_1}-\bm{\theta_2}))-\bm{H}(\bm{\theta_2})\right\|\cdot\left\|\bm{\theta_1}-\bm{\theta_2}\right\|d\tau\\
		&\qquad \le \int_0^1\tau(\gamma_g + \lambda\gamma_h)\left\|\bm{\theta_1}-\bm{\theta_2}\right\|^2d\tau\\
		&\qquad = \frac{\gamma_g + \lambda\gamma_h}{2}\left\|\bm{\theta_1}-\bm{\theta_2}\right\|^2.
	\end{align*}
	The proof of \eqref{eq:hessppty2} follows immediately using a similar procedure (see e.g., \cite{fowkes2013branch}).
\end{proof}
\begin{corollary}\label{thm:hessppty2}
	Let Assumptions \ref{thm:ass1}, \ref{thm:ass1-1} and \ref{thm:ass2} hold. Let $\bm{\theta_1}, \bm{\theta_2}$ be any two points in $\rr^{n_w}$. Then by writing $\partial^2\mathcal{L}=\bm{H}$,
	\begin{align}
		\bm{H}(\bm{\theta_2})-(\gamma_g + \lambda\gamma_h)\left\|\bm{\theta_1}-\bm{\theta_2}\right\| \le \bm{H}(\bm{\theta_1})\le \bm{H}(\bm{\theta_2})+(\gamma_g + \lambda\gamma_h)\left\|\bm{\theta_1}-\bm{\theta_2}\right\|.
	\end{align}
\end{corollary}
\begin{proof}
	The proof follows immediately by recalling for any $\bm{\theta}\in\rr^{n_w}$, $\bm{H}(\bm{\theta})$ is positive definite, and hence the eigenvalues $s_j$ of the difference $\bm{H}(\bm{\theta_1})-\bm{H}(\bm{\theta_2})$ satisfy
	\begin{align*}
		\left\vert s_j \right\vert \le (\gamma_g + \lambda\gamma_h)\left\|\bm{\theta_1}-\bm{\theta_2}\right\|, \quad j=1,2,\ldots,n_w,
	\end{align*}
	so that we have
	\begin{align*}
		-(\gamma_g + \lambda\gamma_h)\left\|\bm{\theta_1}-\bm{\theta_2}\right\| \le \bm{H}(\bm{\theta_1})-\bm{H}(\bm{\theta_2})\le (\gamma_g + \lambda\gamma_h)\left\|\bm{\theta_1}-\bm{\theta_2}\right\|.
	\end{align*}
\end{proof}
\begin{lemma}[{\cite[Theorem 2.1.5]{nesterov2018lectures}}]\label{thm:Llip}
	Let the first derivative of a function $\Phi(\cdot)$ be $L$-Lipschitz on $\mathrm{dom}(\Phi)$. Then for any $\bm{\theta_1},\bm{\theta_2}\in \mathrm{dom}(\Phi)$, we have
	\begin{align*}
		0\le \Phi(\bm{\theta_1})-\Phi(\bm{\theta_2})-\left<\partial\Phi(\bm{\theta_2}),\bm{\theta_1}-\bm{\theta_2}\right>\le\frac{L}{2}\|\bm{\theta_2}-\bm{\theta_1}\|^2.
	\end{align*}
\end{lemma}
\begin{definition}[{\cite[Definition 2.1.3]{nesterov2018lectures}}]\label{def:gdef}
	A continuously differentiable function $\psi(\cdot)$ is $\gamma_{sc}$-strongly convex on $\mathrm{dom}(\psi)$ if for any $\bm{\theta_1},\bm{\theta_2}\in \mathrm{dom}(\psi)$, we have
	\begin{align*}
		\psi(\bm{\theta_1})\ge \psi(\bm{\theta_2})+\left<\partial \psi(\bm{\theta_2}),\bm{\theta_1}-\bm{\theta_2}\right> + \frac{\gamma_{sc}}{2}\|\bm{\theta_1-\bm{\theta_2}}\|^2, \quad \gamma_{sc}>0.
	\end{align*}
\end{definition}
We remark that if the first derivative of the function $\psi(\cdot)$ in the above definition is $L$-Lipschitz continuous, then by construction, for any $\bm{\theta_1},\bm{\theta_2}\in \mathrm{dom}(\psi)$, $\psi$ satisfies the Lipschitz constraint
\begin{align}
	\vert\psi(\bm{\theta_1})-\psi(\bm{\theta_2})\vert\le L\|\bm{\theta_1}-\bm{\theta_2}\|. \label{eq:glip}
\end{align}
\begin{lemma}[{\cite[Theorem 5.1.8]{nesterov2018lectures}}]\label{thm:hlem}
	Let the function $\phi(\cdot)$ be $M_\phi$-self-concordant. Then, for any $\bm{\theta_1},\bm{\theta_2}\in \mathrm{dom}(\phi)$, we have
	\begin{align*}
		\phi(\bm{\theta_1})\ge \phi(\bm{\theta_2})+\left<\partial \phi(\bm{\theta_2}),\bm{\theta_1}-\bm{\theta_2}\right> + \frac{1}{M_\phi^2}\omega(M_\phi\|\bm{\theta_1}-\bm{\theta_2}\|_{\bm{\theta_2}}),
	\end{align*}
	where $\omega(\cdot)$ is an auxiliary univariate function defined by $\omega(t)\coloneqq t-\ln(1+t)$.
\end{lemma}
\begin{lemma}[{\cite[Corollary 5.1.5]{nesterov2018lectures}}]\label{thm:operator}
	Let the function $\phi(\cdot)$ be $M_\phi$-self-concordant. Let $\bm{\theta_1},\bm{\theta_2}\in \mathrm{dom}(\phi)$ and $r=\left\|\bm{\theta_1}-\bm{\theta_2}\right\|_{\bm{\theta_1}}<\frac{1}{M_\phi}$. Then
	\begin{align*}
		\left(1-M_\phi r+\frac{1}{3}M_\phi^2r^2\right)\partial^2\phi(\bm{\theta_1})\preceq\int_0^1\partial^2\phi(\bm{\theta_2}+\tau(\bm{\theta_1}-\bm{\theta_2}))d\tau\preceq\frac{1}{1-M_\phi r}\partial^2\phi(\bm{\theta_1}).
	\end{align*}
\end{lemma}
\section{Missing Proofs}\label{sec:appB}
\begin{proof}[\textbf{Proof of \lemref{thm:descent}}]
	By the arguments of \remref{thm:rem1}, we have that the matrix $(\bm{J}(\bm{\theta}_k)^T\bm{Q}\bm{J}(\bm{\theta}_k) + \lambda \bm{H}_h(\bm{\theta}_k))^{-1}$ is positive definite. Hence, with $\bm{g}(\bm{\theta}_k)\ne \bm{0}$, we have $\bm{g}\bm{H}(\bm{\theta}_k)^{-1}	\bm{g}(\bm{\theta}_k) = -\delta\bm{\theta}^T\bm{g}(\bm{\theta}_k)> 0$ and $\delta\bm{\theta}^T\bm{g}(\bm{\theta}_k) < \bm{0}$.
\end{proof}
\begin{proof}[\textbf{Proof of \thmref{thm:basicconv}}]
	The process formulated in \eqref{eq:delta_theta_N+1} performs the update
	\begin{align*}
		\bm{\theta}_{k+1} = \bm{\theta}_k - \bm{H}(\bm{\theta}_k)^{-1}\bm{g}(\bm{\theta}_k).
	\end{align*}
	As $\bm{g}(\bm{\theta}^*)=0$ by mean value theorem and the first part of \eqref{eq:SOSC}, we have
	\begin{align*}
		\bm{\theta}_{k+1}-\bm{\theta}^* &= \bm{\theta}_k - \bm{\theta}^* - \bm{H}(\bm{\theta}_k)^{-1}\int_0^1\partial^2\mathcal{L}(\bm{\theta}^*+\tau(\bm{\theta}_k-\bm{\theta}^*))(\bm{\theta}_k-\bm{\theta}^*)d\tau\\
		&= \left[\bm{I} - \bm{H}(\bm{\theta}_k)^{-1}\int_0^1\partial^2\mathcal{L}(\bm{\theta}^*+\tau(\bm{\theta}_k-\bm{\theta}^*))d\tau\right](\bm{\theta}_k-\bm{\theta}^*)\\
		&= \bm{H}(\bm{\theta}_k)^{-1}\int_0^1\left(\bm{H}(\bm{\theta}_k) - \partial^2\mathcal{L}(\bm{\theta}^*+\tau(\bm{\theta}_k-\bm{\theta}^*))\right)d\tau(\bm{\theta}_k-\bm{\theta}^*)\\
		\left\|\bm{\theta}_{k+1}-\bm{\theta}^*\right\| &= \left\|\bm{H}(\bm{\theta}_k)^{-1}\int_0^1\left(\bm{H}(\bm{\theta}_k) - \partial^2\mathcal{L}(\bm{\theta}^*+\tau(\bm{\theta}_k-\bm{\theta}^*))\right)d\tau(\bm{\theta}_k-\bm{\theta}^*)\right\|\\
		&\le \left\|\bm{H}(\bm{\theta}_k)^{-1}\right\|\int_0^1\left\|\bm{H}(\bm{\theta}_k) - \partial^2\mathcal{L}(\bm{\theta}^*+\tau(\bm{\theta}_k-\bm{\theta}^*))\right\|d\tau\left\|\bm{\theta}_k-\bm{\theta}^*\right\|\\
	\end{align*}
	Also, as $\tau\in[0,1]$ we have that $\bm{\theta}^*+\tau(\bm{\theta}_k-\bm{\theta}^*)\in\mathcal{B}_\epsilon(\bm{\theta}^*)\subseteq\rr^{n_w}$. By taking the limit $\lim\limits_{k\to\infty}\left\|\bm{\theta}_{k+1}-\bm{\theta}^*\right\|$, we have by the first part of \eqref{eq:SOSC} and \assref{thm:ass3}(ii)
	\begin{align*}
		\left\|\bm{\theta}_{k+1}-\bm{\theta}^*\right\| &\le \left\|\bm{H}(\bm{\theta}_k)^{-1}\right\|\int_0^1\left\|\bm{H}(\bm{\theta}_k) - \bm{H}(\bm{\theta}^*+\tau(\bm{\theta}_k-\bm{\theta}^*))\right\|d\tau\left\|\bm{\theta}_k-\bm{\theta}^*\right\|\\
		&\le \left\|\bm{H}(\bm{\theta}_k)^{-1}\right\|\int_0^1(1-\tau)(\gamma_g + \lambda\gamma_h)\left\|\bm{\theta}_k-\bm{\theta}^*\right\|d\tau\left\|\bm{\theta}_k-\bm{\theta}^*\right\|\\
		&= \frac{\gamma_g + \lambda\gamma_h}{2}\left\|\bm{H}(\bm{\theta}_k)^{-1}\right\|\left\|\bm{\theta}_k-\bm{\theta}^*\right\|^2.
	\end{align*}
	By combining the claims in \colref{thm:hessppty2} and the bounds of $\bm{H}(\bm{\theta}_k)$ in \eqref{eq:hessapproxreg}, we obtain the relation
	\begin{align*}
		\gamma_l+a\gamma_a - (\gamma_g+\lambda\gamma_h)\left\|\bm{\theta}_k-\bm{\theta}^*\right\| \preceq \bm{H}(\bm{\theta}^*)-(\gamma_g+\lambda\gamma_h)\left\|\bm{\theta}_k-\bm{\theta}^*\right\|\preceq \bm{H}(\bm{\theta}_k).
	\end{align*}
	Recall that $\bm{H}(\bm{\theta}_k)$ is positive definite, and hence invertible. We deduce that, indeed for all $\bm{\theta}_k$ satisfying $\left\|\bm{\theta}_k-\bm{\theta}^*\right\|\le\epsilon$, $\epsilon$ small enough, we have
	\begin{align*}
		\left\|\bm{H}(\bm{\theta}_k)^{-1}\right\| \le \left(\gamma_l-\gamma_g - \lambda(\gamma_h-\gamma_a)\left\|\bm{\theta}_k-\bm{\theta}^*\right\|\right)^{-1}.
	\end{align*}
	Therefore,
	\begin{align*}
		\left\|\bm{\theta}_{k+1}-\bm{\theta}^*\right\| \le \xi_k\left\|\bm{\theta}_k-\bm{\theta}^*\right\|^2,
	\end{align*}
	where
	\begin{align*}
		\xi_k = \frac{1}{2}\frac{\gamma_g + \lambda\gamma_h}{\left(\gamma_l-\gamma_g - \lambda(\gamma_h-\gamma_a)\left\|\bm{\theta}_k-\bm{\theta}^*\right\|\right)}.
	\end{align*}
\end{proof}
\begin{proof}[\textbf{Proof of \thmref{thm:main}}]
	First, we upper bound the norm $\|\delta\bm{\theta}\|_{\bm{\theta}_k}\coloneqq\|\bm{\theta}_{k+1}-\bm{\theta}_k\|_{\bm{\theta}_k}=\|\bm{H}_h^{1/2}(\bm{\theta}_{k+_1}-\bm{\theta}_k)\|$. From \remref{thm:remJ}, we have
	\begin{align*}
		\|\bm{J}(\bm{\theta}_k)^T(\lambda\bm{I}+\bm{Q}\bm{J}\bm{H}_h^{-1}\bm{J}^T(\bm{\theta}_k))^{-1}\| &\le \frac{\gamma_a\|\bm{J}(\bm{\theta}_k)^T\|}{K+\lambda\gamma_a}\le \frac{\tilde{\beta}\gamma_a}{K+\lambda\gamma_a}.
	\end{align*}
	Hence,
	\begin{align}
		&\|\bm{\theta}_{k+1}-\bm{\theta}_k\|_{\bm{\theta}_k} \le \frac{\alpha}{1+M_h\eta_k}\left\|\bm{H}_h^{1/2}\bm{H}_h^{-1}\bm{J}^T(\bm{\theta}_k)(\lambda\bm{I}+\bm{Q}\bm{J}\bm{H}_h^{-1}\bm{J}^T(\bm{\theta}_k))^{-1}\bm{e}(\bm{\theta}_k)\right\|\notag\\
		&\le \frac{\alpha}{1+M_h\eta_k} \left\|\bm{H}_h(\bm{\theta}_k)^{1/2}\right\|\left\|\bm{H}_h(\bm{\theta}_k)^{-1}\right\|\left\|\bm{J}(\bm{\theta}_k)^T(\lambda\bm{I}+\bm{Q}\bm{J}\bm{H}_h^{-1}\bm{J}^T(\bm{\theta}_k))^{-1}\right\|\left\|\bm{e}(\bm{\theta}_k)\right\|\notag\\
		&\le \frac{\alpha\beta\tilde{\beta}\gamma_a^{-1/2}}{(K+\lambda\gamma_a)(1+M_h\eta_k)}. \label{eq:local-norm-d}
	\end{align}
	Similarly, we have
	\begin{align}
		\|\bm{\theta}_{k+1}-\bm{\theta}_k\|\le \frac{\alpha\beta\tilde{\beta}}{(K+\lambda\gamma_a)(1+M_h\eta_k)}. \label{eq:2-norm-d}
	\end{align}
	By \lemref{thm:Llip}, the function $\mathcal{L}(\bm{\theta})$ satisfies
	\begin{align*}
		&\mathcal{L}(\bm{\theta}_{k+1})\le\mathcal{L}(\bm{\theta}_k)+\left<\partial\mathcal{L}(\bm{\theta}_k),\bm{\theta}_{k+1}-\bm{\theta}_k\right>+\frac{\gamma_u+\lambda\gamma_b}{2}\|\bm{\theta}_{k+1}-\bm{\theta}_k\|^2\\
		&= \mathcal{L}(\bm{\theta}_k)+\left<\bm{g}_g(\bm{\theta}_k),\bm{\theta}_{k+1}-\bm{\theta}_k\right>+\lambda\left<\bm{g}_h(\bm{\theta}_k),\bm{\theta}_{k+1}-\bm{\theta}_k\right>\frac{\gamma_u+\lambda\gamma_b}{2}\|\bm{\theta}_{k+1}-\bm{\theta}_k\|^2.
	\end{align*}
	By convexity of $g$ and $h$, and self-concordance of $h$, we have (using \defnref{def:gdef}, \lemref{thm:hlem} and \eqref{eq:glip})
	\begin{align*}
		\mathcal{L}(\bm{\theta}_{k+1})&\le\mathcal{L}(\bm{\theta}_k)+(\gamma_u+\lambda\gamma_b)\|\bm{\theta}_{k+1}-\bm{\theta}_k\|+\frac{\gamma_u+\lambda\gamma_b-\gamma_l}{2}\|\bm{\theta}_{k+1}-\bm{\theta}_k\|^2\\
		&\qquad -\frac{\lambda}{M_h^2}\omega(M_h\|\bm{\theta}_{k+1}-\bm{\theta}_k\|_{\bm{\theta}_k})\\
		&\stackrel{\eqref{eq:local-norm-d}}{\le}\mathcal{L}(\bm{\theta}_k)+(\gamma_u+\lambda\gamma_b)\|\bm{\theta}_{k+1}-\bm{\theta}_k\|+\frac{\gamma_u+\lambda\gamma_b-\gamma_l}{2}\|\bm{\theta}_{k+1}-\bm{\theta}_k\|^2\\
		&\qquad-\frac{\lambda}{M_h^2}\omega\left(\frac{\alpha\beta\tilde{\beta}\gamma_a^{-1/2}M_h}{(K+\lambda\gamma_a)(1+M_h\eta_k)}\right).
	\end{align*}
	Substituting the choice $\alpha= (\beta\tilde{\beta})^{-1}(\gamma_a)^{1/2}(K+\lambda\gamma_a)$, we have
	\begin{align*}
		\mathcal{L}(\bm{\theta}_{k+1})\le\,\,&\mathcal{L}(\bm{\theta}_k)+(\gamma_u+\lambda\gamma_b)\|\bm{\theta}_{k+1}-\bm{\theta}_k\|+\frac{\gamma_u+\lambda\gamma_b-\gamma_l}{2}\|\bm{\theta}_{k+1}-\bm{\theta}_k\|^2\\&-\frac{\lambda}{M_h^2}\omega\left(\frac{M_h}{1+M_h\eta_k}\right).
	\end{align*}
	Taking expectation on both sides with respect to $m$ conditioned on $\bm{\theta}_k$, we get
	\begin{align*}
		\mathbb{E}[\mathcal{L}(\bm{\theta}_{k+1})]\le\,\,&\mathcal{L}(\bm{\theta}_k)+(\gamma_u+\lambda\gamma_b)\|\bm{\theta}_{k+1}-\bm{\theta}_k\|+\frac{\gamma_u+\lambda\gamma_b-\gamma_l}{2}\|\bm{\theta}_{k+1}-\bm{\theta}_k\|^2\\&-\mathbb{E}\left[\frac{\lambda}{M_h^2}\omega\left(\frac{M_h}{1+M_h\eta_k}\right)\right],
	\end{align*}
	Note the second derivative $\omega''$ of $\omega$: $\omega''(t)=1/(1+t)^2$. By the convexity of $\omega$ and using Jensen's inequality, also recalling unbiasedness of the derivatives,
	\begin{align*}
		\mathbb{E}[\mathcal{L}(\bm{\theta}_{k+1})]&\le\mathcal{L}(\bm{\theta}_k)+\frac{\gamma_u+\lambda\gamma_b}{\sqrt{\gamma_a}(1+M_h\eta_k)}+\frac{\gamma_u+\lambda\gamma_b}{2\gamma_a(1+M_h\eta_k)^2}-\frac{\gamma_l}{2\gamma_a}\omega''(M_h\eta_k)\\
		&\qquad-\frac{\lambda}{M_h^2}\omega\left(\frac{M_h}{1+M_h\eta_k}\right)\\
		&\le\mathcal{L}(\bm{\theta}_k)-\left[\frac{\lambda}{M_h^2}\omega\left(\frac{M_h}{1+M_h\eta_k}\right)+\frac{\gamma_l}{2\gamma_a}\omega''(M_h\eta_k)-\frac{2(\gamma_u+\lambda\gamma_b)}{\sqrt{\gamma_a}}\right].
	\end{align*}
	In the above, we used the bounds of the norm in \eqref{eq:2-norm-d}, and again used the choice $\alpha= (\beta\tilde{\beta})^{-1}(\gamma_a)^{1/2}(K+\lambda\gamma_a)$.

	To proceed, let us make a simple remark that is not explicitly stated in \remref{thm:rem2}: For any $\bm{\theta}_k,\bm{\theta}_{k+1}\in\mathcal{N}_\epsilon(\bm{\theta}^*)$, we have
	\begin{subequations}
		\begin{eqnarray*}
			\left\|\bm{H}_g(\bm{\theta}_{k+1}) - \bm{H}_g(\bm{\theta}_k)\right\| \le \gamma_g\left\|\bm{\theta_1} - \bm{\theta_2}\right\| \to 0 \,\,\text{as}\,\, k\to\infty,\\
			\left\|\bm{H}_h(\bm{\theta}_{k+1}) - \bm{H}_h(\bm{\theta}_k)\right\| \le \gamma_h\left\|\bm{\theta_1} - \bm{\theta_2}\right\| \to 0 \,\,\text{as}\,\, k\to\infty.
		\end{eqnarray*}
	\end{subequations}
	Now, recall the proposed update step \eqref{eq:proposed}:
	\begin{align*}
		\bm{\theta}_{k+1}=\bm{\theta}_k-\frac{\alpha}{1+M_h\eta_k}\bm{H}_h^{-1}\bm{J}^T\left(\lambda\bm{I}+\bm{Q}\bm{J}
		\bm{H}_h^{-1}
		\bm{J}^T\right)^{-1}\bm{e}.
	\end{align*}
	Then the above remark allows us to perform the following operation: Subtract $\bm{\theta}^*$ from both sides and pre-multiply by $\bm{H}_h^{1/2}(\hm{\theta}_{k+1})\approx\bm{H}_h^{1/2}(\hm{\theta}_k)=\bm{H}_h^{1/2}$, we get the recursion
	\begin{align*}
		&\bm{H}_h^{1/2}(\bm{\theta}_{k+1}-\bm{\theta}^*)=\\
		&\bm{H}_h^{1/2}(\bm{\theta}_k-\bm{\theta}^*)-\frac{\alpha}{1+M_h\eta_k}\bm{H}_h^{1/2}\bm{H}_h^{-1}\bm{J}^T(\bm{\theta}_k)\left(\lambda\bm{I}+\bm{Q}\bm{J}
		\bm{H}_h^{-1}\bm{J}^T(\bm{\theta}_k)\right)^{-1}\bm{e}(\bm{\theta}_k),\\
		&\left\|\bm{H}_h^{1/2}(\bm{\theta}_{k+1}-\bm{\theta}^*)\right\|=\\
		&\left\|\bm{H}_h^{1/2}(\bm{\theta}_k-\bm{\theta}^*)-\frac{\alpha}{1+M_h\eta_k}\bm{H}_h^{1/2}\bm{H}_h^{-1}\bm{J}^T(\bm{\theta}_k)\left(\lambda\bm{I}+\bm{Q}\bm{J}
		\bm{H}_h^{-1}\bm{J}^T(\bm{\theta}_k)\right)^{-1}\bm{e}(\bm{\theta}_k)\right\|\\
		&\le \left\|\bm{H}_h^{1/2}(\bm{\theta}_k-\bm{\theta}^*)\right\|+\frac{\alpha}{1+M_h\eta_k}\left\|\bm{H}_h^{1/2}\right\|\left\|\bm{H}_h^{-1}\bm{J}^T\left(\lambda\bm{I}+\bm{Q}\bm{J}
		\bm{H}_h^{-1}\bm{J}^T\right)^{-1}\bm{e}(\bm{\theta}_k)\right\|.
	\end{align*}
	Take expectation with respect to $m$ on both sides conditioned on $\bm{\theta}_k$ and again consider unbiasedness of the derivatives. Further, recall the definition of the \emph{local norm} $\|\cdot\|_{\bm{\theta}}$, and the bounds of $\bm{H}_h$, then
	\begin{align*}
		&\mathbb{E}\left\|\bm{\theta}_{k+1}-\bm{\theta}^*\right\|_{\bm{\theta}_{k+1}}\le\\
		&\qquad \left\|\bm{\theta}_k-\bm{\theta}^*\right\|_{\bm{\theta}_k}+\frac{\alpha}{\sqrt{\gamma_a}(1+M_h\eta_k)}\left\|\left(\lambda\bm{I}+\bm{Q}\bm{J}\bm{H}_h^{-1}\bm{J}^T(\bm{\theta}_k)\right)^{-1}\right\|\left\|\bm{H}_h^{-1}\bm{J}^T\bm{e}(\bm{\theta}_k)\right\|\\
		&\qquad \le \left\|\bm{\theta}_k-\bm{\theta}^*\right\|_{\bm{\theta}_k}+\frac{\alpha\gamma_a}{\sqrt{\gamma_a}(1+M_h\eta_k)(K+\lambda\gamma_a)}\left\|\bm{H}_h^{-1}\bm{g}(\bm{\theta}_k)\right\|\\
		&\qquad = \left\|\bm{\theta}_k-\bm{\theta}^*\right\|_{\bm{\theta}_k}+\frac{\gamma_a}{\beta\tilde{\beta}(1+M_h\eta_k)}\left\|\bm{H}_h^{-1}\bm{g}_g(\bm{\theta}_k)+\lambda\bm{H}_h^{-1}\bm{g}_h(\bm{\theta}_k)\right\|\\
		&\qquad \le \left\|\bm{\theta}_k-\bm{\theta}^*\right\|_{\bm{\theta}_k}+\frac{\gamma_a}{\beta\tilde{\beta}}\left\|\bm{H}_h^{-1}\bm{g}_g(\bm{\theta}_k)+\lambda\bm{H}_h^{-1}\bm{g}_h(\bm{\theta}_k)\right\|\\
		&\qquad \le \left\|\bm{\theta}_k-\bm{\theta}^*\right\|_{\bm{\theta}_k}+\frac{\gamma_a}{\beta\tilde{\beta}}\left\|\bm{H}_h^{-1}\bm{g}_g(\bm{\theta}_k)\right\|+\frac{\lambda\gamma_a}{\beta\tilde{\beta}}\left\|\bm{H}_h^{-1}\bm{g}_h(\bm{\theta}_k)\right\|.
	\end{align*}
	By the mean value theorem and the first part of \eqref{eq:SOSC},
	\begin{align*}
		\bm{g}_g(\bm{\theta}_k) &= \int_0^1\bm{H}_g(\bm{\theta}^*+\tau(\bm{\theta}_k-\bm{\theta}^*))(\bm{\theta}_k-\bm{\theta}^*)d\tau\\
		&= \bm{H}_g(\bm{\theta}^*)(\bm{\theta}_k-\bm{\theta}^*) + \int_0^1\left(\bm{H}_g(\bm{\theta}^*+\tau(\bm{\theta}_k-\bm{\theta}^*))-\bm{H}_g(\bm{\theta}^*)\right)(\bm{\theta}_k-\bm{\theta}^*)d\tau,
	\end{align*}
	where $\bm{H}_g$ is the second derivative of $g$.
	\begin{align*}
		&\left\|\bm{g}_g(\bm{\theta}_k)\right\|= \left\|\bm{H}_g(\bm{\theta}^*)(\bm{\theta}_k-\bm{\theta}^*) + \int_0^1\left(\bm{H}_g(\bm{\theta}^*+\tau(\bm{\theta}_k-\bm{\theta}^*))-\bm{H}_g(\bm{\theta}^*)\right)(\bm{\theta}_k-\bm{\theta}^*)d\tau\right\|\\
		&\qquad \le \left\|\bm{H}_g(\bm{\theta}^*)\right\|\left\|\bm{\theta}_k-\bm{\theta}^*\right\| + \int_0^1\left\|\bm{H}_g(\bm{\theta}^*+\tau(\bm{\theta}_k-\bm{\theta}^*))-\bm{H}_g(\bm{\theta}^*)\right\|\left\|\bm{\theta}_k-\bm{\theta}^*\right\|d\tau\\
		&\qquad = \left\|\bm{H}_g(\bm{\theta}^*)\right\|\left\|\bm{\theta}_k-\bm{\theta}^*\right\|+\int_0^1\tau\gamma_g\left\|\bm{\theta}_k-\bm{\theta}^*\right\|^2d\tau\\
		&\qquad = \left\|\bm{H}_g(\bm{\theta}^*)\right\|\left\|\bm{\theta}_k-\bm{\theta}^*\right\|+\frac{\gamma_g}{2}\left\|\bm{\theta}_k-\bm{\theta}^*\right\|^2d\tau\\
		&\qquad \le \gamma_u\left\|\bm{\theta}_k-\bm{\theta}^*\right\|+\frac{\gamma_g}{2}\left\|\bm{\theta}_k-\bm{\theta}^*\right\|^2.
	\end{align*}
	In the above steps, we have used \assref{thm:ass3} and the remarks that follow it. Further,
	\begin{align*}
		\left\|\bm{H}_h^{-1}\bm{g}_g(\bm{\theta}_k)\right\|&=\left\|\bm{H}_h^{-1}\bm{g}_g(\bm{\theta}_k)\right\|\\
		&\le \left\|\bm{H}_h^{-1}\right\|\left\|\bm{g}_g(\bm{\theta}_k)\right\|\\
		&\le \frac{1}{\gamma_a}\left(\gamma_u\left\|\bm{\theta}_k-\bm{\theta}^*\right\|+\frac{\gamma_g}{2}\left\|\bm{\theta}_k-\bm{\theta}^*\right\|^2\right)\\
		&= \frac{\gamma_u}{\gamma_a}\left\|\bm{\theta}_k-\bm{\theta}^*\right\|+\frac{\gamma_g}{2\gamma_a}\left\|\bm{\theta}_k-\bm{\theta}^*\right\|^2.
	\end{align*}
	Next, we analyze $\left\|\bm{H}_h^{-1}\bm{g}_h(\bm{\theta}_k)\right\|$. We have
	\begin{align*}
		\left\|\bm{H}_h^{-1}\bm{g}_h(\bm{\theta}_k)\right\| &= \left\|\bm{H}_h^{-3/2}\bm{H}_h^{1/2}\bm{g}_h(\bm{\theta}_k)\right\|\\
		&\le\frac{1}{\gamma_a^{3/2}}\left\|\bm{H}_h^{1/2}\bm{g}_h(\bm{\theta}_k)\right\|\\
		&\stackrel{\eqref{eq:SOSC}}{=} \frac{1}{\gamma_a^{3/2}}\left\|\bm{H}_h^{1/2}(\bm{\theta}_k)\left(\bm{g}_h(\bm{\theta}_k)-\bm{g}_h(\bm{\theta}^*)\right)\right\|,
	\end{align*}
	and by the mean value theorem,
	\begin{align*}
		&\left\|\bm{H}_h^{-1}\bm{g}_h(\bm{\theta}_k)\right\| = \frac{1}{\gamma_a^{3/2}}\left\|\int_0^1\bm{H}_h^{1/2}(\bm{\theta}_k)\bm{H}_h(\bm{\theta}^*+\tau(\bm{\theta}_k-\bm{\theta}^*))(\bm{\theta}_k-\bm{\theta}^*)d\tau\right\|\\
		&\qquad \le \frac{1}{\gamma_a^{3/2}}\left\|\bm{H}_h^{1/2}(\bm{\theta}_k)(\bm{\theta}_k-\bm{\theta}^*)\right\|\left\|\int_0^1\bm{H}_h(\bm{\theta}^*+\tau(\bm{\theta}_k-\bm{\theta}^*))d\tau\right\|\\
		&\qquad = \frac{1}{\gamma_a^{3/2}}\left\|\bm{\theta}_k-\bm{\theta}^*\right\|_{\bm{\theta}_k}\left\|\int_0^1\bm{H}_h(\bm{\theta}^*+\tau(\bm{\theta}_k-\bm{\theta}^*))d\tau\right\|\\
		&\qquad \le \frac{1}{\gamma_a^{3/2}}\left\|\bm{\theta}_k-\bm{\theta}^*\right\|_{\bm{\theta}_k}\left\|\int_0^1\bm{H}_h^{-1/2}(\bm{\theta}_k)\bm{H}_h(\bm{\theta}^*+\tau(\bm{\theta}_k-\bm{\theta}^*))\bm{H}_h^{-1/2}(\bm{\theta}_k)d\tau\right\|\\
		&\,\,\, \stackrel{\lemref{thm:operator}}{\le} \frac{\left\|\bm{\theta}_k-\bm{\theta}^*\right\|_{\bm{\theta}_k}}{\gamma_a^{3/2}\left(1-M_h\left\|\bm{\theta}_k-\bm{\theta}^*\right\|_{\bm{\theta}_k}\right)}.
	\end{align*}
	In the above, we have used the fact $\bm{\theta}_0\in\mathcal{N}_{M_h^{-1}}(\bm{\theta}^*)\implies\bm{\theta}_k\in\mathcal{N}_{M_h^{-1}}(\bm{\theta}^*)$ for all $\bm{\theta}_k$ generated by the process \eqref{eq:proposed}.

	Combining the above results, we have
	\begin{align*}
		\mathbb{E}\left\|\bm{\theta}_{k+1}-\bm{\theta}^*\right\|_{\bm{\theta}_{k+1}}\le&\,\, \left[1+\frac{\lambda\gamma_a^{-1/2}}{\beta\tilde{\beta}\left(1-M_h\left\|\bm{\theta}_k-\bm{\theta}^*\right\|_{\bm{\theta}_k}\right)}\right]\left\|\bm{\theta}_k-\bm{\theta}^*\right\|_{\bm{\theta}_k}\\
		&+\frac{\gamma_u}{\beta\tilde{\beta}}\left\|\bm{\theta}_k-\bm{\theta}^*\right\|+\frac{\gamma_g}{2}\left\|\bm{\theta}_k-\bm{\theta}^*\right\|^2.
	\end{align*}
\end{proof}

\end{document}